\documentclass[twoside,11pt, reqno]{article}

\makeatletter
\def\subsubsection{\@startsection{subsubsection}{3}%
\z@{.5\baselineskip\@plus.7\baselineskip}{-.5em}%
{\normalfont\bfseries}}
\makeatother

\usepackage[numbers]{natbib}
\usepackage{amsfonts}
\usepackage{amsmath}
\usepackage{amsthm}
\usepackage{amssymb}
\usepackage{color}
\usepackage{graphicx}
\usepackage{hyperref}
\hypersetup{
    colorlinks=true, 
    linktoc=all,     
    linkcolor=blue,  
}

\usepackage{comment}
\usepackage[top=1in, bottom=1in, left=1in, right=1in]{geometry}

\usepackage[utf8]{inputenc} 
\usepackage[T1]{fontenc}    
\usepackage{url}            
\usepackage{booktabs}       
\usepackage{amsfonts}       
\usepackage{nicefrac}       
\usepackage{microtype}      
\usepackage{xcolor}         
\usepackage{setspace}

\usepackage[normalem]{ulem}

\usepackage{amsmath}
\usepackage{amsthm}
\usepackage{amssymb}
\usepackage{graphicx}
\graphicspath{{../}{./}}
\usepackage{comment}
\usepackage{algorithm, algpseudocode}
\usepackage{mathrsfs}  
\usepackage{mathtools}
\usepackage{enumitem}
\usepackage{multirow}
\usepackage{dsfont}
\usepackage{subcaption}

\newcommand{\red}[1]{\textcolor{red}{#1}}

\def\a{\alpha}
\def\cA{\mathcal{A}}

\def\cF{\mathcal{F}}
\def\cG{\mathcal{G}}

\def\cK{\mathcal{K}}

\def\cN{\mathcal{N}}
\def\cO{\mathcal{O}}

\def\cT{\mathcal{T}}

\def\cV{\mathcal{V}}

\def\sN{{\mathbb{N}}}
\def\sP{\mathbb{P}}

\def\sR{{\mathbb R}}

\def\d{{\mathrm{d}}}

\numberwithin{equation}{section}

\theoremstyle{plain}

\newtheorem{theorem}{Theorem}[section]
\newtheorem{lemma}{Lemma}[section]
\newtheorem{prop}{Proposition}[section]
\newtheorem{corollary}{Corollary}[section]
\newtheorem{assumption}{Assumption}[section]
\newtheorem{outstanding_assumption}{Assumption}[section]

\theoremstyle{definition}
\newtheorem{definition}[theorem]{Definition}

\theoremstyle{remark}
\newtheorem{remark}[theorem]{Remark}

\newtheorem{example}{Example}[section]

\def\pathspace{{\cV^p([0,T];\mathbb{R}^d)}}
\def\pathspaceL{{\cV^p([0,L];\mathbb{R}^d)}}
\newcommand{\Sig}{\text{Sig} }
\newcommand{\vep}{\varepsilon}
\DeclareMathOperator*{\argmin}{argmin}

\makeatletter
\providecommand*{\shuffle}{%
  \mathbin{\mathpalette\shuffle@{}}%
}
\newcommand*{\shuffle@}[2]{%
  \sbox0{$#1\vcenter{}$}%
  \kern .15\ht0 
  \rlap{\vrule height .25\ht0 depth 0pt width 2.5\ht0}%
  \raise.1\ht0\hbox to 2.5\ht0{%
    \vrule height 1.75\ht0 depth -.1\ht0 width .17\ht0 %
    \hfill
    \vrule height 1.75\ht0 depth -.1\ht0 width .17\ht0 %
    \hfill
    \vrule height 1.75\ht0 depth -.1\ht0 width .17\ht0 %
  }%
  \kern .15\ht0 
}

\makeatother

\title{Signature Approach for Contextual Bandits with Nonlinear and Path-dependent Rewards: Regret Bound and Statistical Inference}

\author{%
  Xin Guo \thanks{
 University of California, Berkeley,
 Berkeley, CA, 
  \texttt{xinguo@berkeley.edu}} \\
  \and
 Grace He \thanks{
 University of California, Berkeley,
 Berkeley, CA, 
 \texttt{grace\_he@berkeley.edu} }
  \and
  Xinyu Li \thanks{
  University of Oxford, 
  Oxford, UK, 
  \texttt{xinyu.li@maths.ox.ac.uk}} 
}

\begin{document}

\maketitle

\begin{abstract}
We study contextual bandits with nonlinear and path-dependent rewards through a novel signature-transform-based approach. Leveraging the universal nonlinearity property of signatures, we approximate continuous path-dependent reward functionals by linear functionals in the signature space. 
This representation enables the use of efficient linear contextual bandit methods while preserving expressive sequential structure. 
Building on this framework, we propose \texttt{DisSigUCB}, a signature-based disjoint upper confidence bound (UCB) algorithm. 
Under boundedness and non-degeneracy assumptions, we establish a high-probability regret bound of order \(\tilde\cO\bigl(d_\Sig \sqrt{KT} + \vep_N T \sqrt{d_\Sig}\bigr)\) with \(K\) being the number of actions, \(d_\Sig\)  the signature feature dimension, and \(\vep_N\)  the finite-depth approximation error.
We further introduce $\zeta$-\texttt{DisSigUCB}, a randomized variant that ensures persistent exploration for statistical inference. 
Under additional margin and long-run design conditions, we prove consistency and asymptotic normality of the action-specific ridge estimators.
Both the synthetic experiments and the numerical studies on temperature sensor monitoring, sleep-stage classification, and hospital nurse staffing demonstrate the effectiveness of \texttt{DisSigUCB} over linear and kernelized contextual bandit baselines in nonlinear and path-dependent settings.

\end{abstract}

\section{Introduction}

\paragraph{Contextual bandit.} This sequential decision-making problem extends the classical multi-armed bandit framework by allowing the learner to observe additional information before choosing an action. It enables dynamic and personalized decision-making in applications such as personalized dynamic pricing \cite{LYWS:25, LS:26, ZJY:26}, 
online recommendations \cite{CLS:21,ZHWZS:25,LSL:26},
and individualized medical treatments 
\cite{CWFWL:24}.
More specifically, in a contextual bandit,  the learner at each time $t$ observes a context $x_t$, selects an action $a_t$, and receives a reward $r_t(x_t,a_t)$  for the chosen action.
Those rewards associated with  unselected actions are not observed. The objective of the learner is to maximize the cumulative reward over a horizon of $T$ rounds.
The learner observes rewards only for selected actions and must adapt decisions sequentially, giving rise to the exploration--exploitation tradeoff: exploiting actions that appear optimal while exploring uncertain actions to improve future decisions.
Performance is commonly evaluated through cumulative regret, which measures the loss relative to an oracle that selects the optimal action at each round.

The adaptive nature of contextual bandits also creates challenges for statistical estimation and inference. 
Because action selection depends on previously observed contexts and rewards, the resulting observations are generally dependent and non-identically distributed, and estimators constructed from adaptively collected data may no longer follow the classical asymptotic behavior expected under i.i.d.\ sampling.

\paragraph{Existing work.} A large proportion of work in contextual bandits studies linear contextual bandits where rewards are linear in the context in the sense that $\mathbb{E}[r_t(a)|x_{a,t}] = x_{a,t}^\top \beta_a$, where $x_{a,t}$ denotes the observed feature vector associated with action $a$ at round $t$, and $\beta_a$ is the unknown action-specific reward parameter.
A classical approach is the upper-confidence-bound framework (\texttt{LinUCB} \citep{LCLLCB:11, CLRS:11}), where the action with the highest confidence bound on the reward estimate is chosen. 
That is,
$
A_t \in \arg\max_{a \in [K]} x_{a,t}^\top \hat\beta_a + \gamma \|x_{a,t}\|_{V_{a,t}^{-1}},
$
where $\hat\beta_a$ is the estimate of $\beta_a$, $\|x_{a,t}\|_{V_{a,t}^{-1}}$ measures the uncertainty of the reward estimate \cite{APSLSB:11}, and $\gamma$ is the exploration parameter. In each round $t$, the learner achieves exploitation by choosing an action that maximizes the reward estimate $x_{a,t}^\top \hat\beta_a$ and exploration by choosing an action with large uncertainty $\|x_{a,t}\|_{V_{a,t}^{-1}}$, promoting actions that have been selected less often. 
Feature-diversity or covariance non-degeneracy conditions are also commonly imposed to ensure sufficient variation in the observed contexts and growth of the empirical design matrices \cite{CMBOSOM:20, DSLCB:24}.
However, the linearity assumption is an oversimplification and is often violated in practice. 
In healthcare, treatment rewards may depend on recent patient symptoms and prior responses, instead of only the current clinical state. 
In finance, many option payoffs are nonlinear functions  of the price trajectory. 
In such settings, linear reward models that depend only on the current context are generally misspecified.

When the expected reward is only approximately linear in a prescribed feature representation,  Ghosh et al.~\cite{GCGz:17} show that algorithms designed for exactly linear rewards can suffer linear regret under model misspecification. 
Subsequent work develops methods that explicitly account for the additive approximation error $\varepsilon$ and studies robustness to misspecification under a fixed linear feature representation \cite{FGMZ:20, TIHSFKK:21}.


Relaxing the linear assumption, generalized linear bandits assume the reward is generated by a generalized linear model (GLM) where $\mathbb{E}[r_{t}(a) \mid x_{a,t}] = \mu(x_{a,t}^\top \beta_a^*)$ for a known link function $\mu$ and unknown parameter $\beta_a^*$ \cite{FCGS:10, LLZGLCB:17, ZXBGLCB:19}. 
Their methods exploit the known monotone link function and GLM regression to construct optimistic confidence bounds for the reward and design corresponding UCB-style algorithms such as \texttt{GLM-UCB} \cite{FCGS:10}. 
Kernelized contextual bandits instead assume the mean reward lies in a reproducing kernel Hilbert space (RKHS) \cite{SZBPWKCB:11, VKMFCKCB:13, CGKCB:17} and use kernelized regression to design corresponding UCB-style algorithms such as \texttt{KernelUCB} \cite{VKMFCKCB:13}. 
Due to the kernel-induced similarity, these methods generalize information across nearby contexts and/or actions, which is particularly useful when the context or action space is large or continuous. 
Both of these approaches reduce nonlinear rewards to linear rewards in a transformed feature space and improve expressivity but still rely on a fixed functional or kernel structure. 
In contrast, our setting considers rewards that may depend nonlinearly on an entire recent contextual path.
We summarize the classical algorithms, their assumptions, and regret scaling in Table \ref{tab:related_work}.

\begin{table}
\centering
\footnotesize
\caption{Regret bounds of related contextual bandit methods, up to logarithmic factors. Here, $\tilde d$ denotes the effective kernel dimension, $d_{\Sig}$ denotes the signature feature dimension, and $\vep_N$ denotes the finite-depth signature approximation error.}
\label{tab:related_work}
\renewcommand{\arraystretch}{1.25}
\setlength{\tabcolsep}{3pt}
\begin{tabular}{p{0.25\linewidth} p{0.45\linewidth} p{0.22\linewidth}}
\toprule
\textbf{Method} 
& \textbf{Assumption on reward} 
& \textbf{Regret scaling} \\
\midrule

\texttt{LinUCB} \cite{LCLLCB:11,CLRS:11}
& Linear, shared parameter
& \(\tilde{\mathcal O}(d\sqrt{T})\) \cite{APSLSB:11} \\

Disjoint \texttt{LinUCB} \cite{LCLLCB:11}
& Linear, action-specific parameters
& \(\tilde\cO(d \sqrt{KT})\) \cite{DSLCB:24}\\

Misspecified \texttt{LinCB} \cite{FGMZ:20, TIHSFKK:21}
& $\vep$-misspecified linear, shared parameter
& \(\tilde{\mathcal O}(d\sqrt{T} + \vep \sqrt{d} T)\) \cite{FGMZ:20} \\


\texttt{GLM-UCB} \cite{FCGS:10}
& Generalized linear model, shared parameter
& \(\tilde{\mathcal O}(d\sqrt{T})\) \cite{FCGS:10} \\

\texttt{KernelUCB} \cite{VKMFCKCB:13}
& Lies in RKHS, shared reward function
& \(\tilde{\mathcal O}(\sqrt{\tilde{d}T})\) \cite{VKMFCKCB:13} \\

\texttt{DisSigUCB} (ours)
& Nonlinear, path-dependent, action-specific parameters
& \(\tilde{\mathcal O}\!\left(d_\Sig \sqrt{KT} + \vep_N\sqrt{d_\Sig} T
\right) \) \\

\bottomrule
\end{tabular}
\end{table}

Beyond regret minimization, a growing literature studies statistical inference for estimators in contextual bandits. 
Because actions are selected adaptively, the resulting observations are generally dependent and non-identically distributed, and conventional regression-based confidence intervals may be invalid. 
Several approaches recover asymptotic normality through randomized exploration or modifications of the estimator, including adaptively weighted estimating equations and stochastic-gradient methods \cite{ZJM:20, CLS:21, LKW:25, CCLLLZ:26}. 
However, a common feature of several of these analyses is an independence assumption on the contextual process, such as i.i.d.\ contexts across rounds or conditionally i.i.d.\ contexts within batches \cite{CLS:21, ZJM:20, CCLLLZ:26}.
Related work has extended online inference to richer forms of contextual information, such as matrix-valued covariates with low-rank structure \cite{HSZ:25}.
For UCB-based policies, existing inference results are more specialized, focusing on non-contextual action means \cite{KZ:24} or \texttt{LinUCB} under a unit-ball action-space geometry \cite{fan2025statistical}.
None of these results directly establish asymptotic normality for action-specific estimators under fully online UCB-based sampling with dependent and nonstationary contextual features.


\paragraph{Our work.} We study the contextual bandit setting where rewards are {\it nonlinear} and  {\it  path-dependent}. 
{Instead of representing each context by a single feature vector, we model the observed contextual history as an evolving path with temporal structure.}
To capture this structure while retaining the tractability of linear bandit methods, we use the \emph{signature transform}. The signature of a path encodes both higher-order temporal interactions and geometric properties of the trajectory. Signature-based approaches have been successfully applied to sequential data analysis, including time-series prediction, kernel methods, deep learning architectures, and diffusion models \cite{BKPSL:19, KO:19, GGJKL:24, CK:25, JGCYPLBMN:25}. A key property of signatures underlying their success is their universal nonlinearity: continuous path-dependent reward functionals can be approximated arbitrarily well by linear functionals of signature features. In this contextual bandit problem, this novel signature-based approach enables 
\begin{itemize}
    \item Representing nonlinear rewards in a linear feature space and adopting upper confidence bound (UCB) methods (Theorem \ref{thm:linear_regret}).
    \item Proposing \texttt{DisSigUCB}, a UCB-style algorithm for contextual bandits using signature features to handle nonlinear path-dependent rewards (Algorithm \ref{alg:dissigucb}). 
    \item {Establishing a high-probability, data-dependent regret bound for \texttt{DisSigUCB} of order $R(T) = \tilde{\mathcal{O}}(d_{\Sig}  \sqrt{KT} + \vep_N T\sqrt{d_\Sig})$, with $K$ the number of actions, $T$ the time horizon,  $d_{\Sig}$ the contextual signature feature dimension, which depends on the reward function and contextual features, and $\vep_N$ the finite-depth approximation error (Theorem \ref{thm:sigUCB_regret}). In the exact finite-depth realization case $\vep_N = 0$, this reduces to a sublinear regret bound of order $R(T) = \tilde{\mathcal{O}}(d_{\Sig}  \sqrt{KT})$.}
    \item {Proving consistency and asymptotic normality of the action-specific ridge estimators under $\zeta$-\texttt{DisSigUCB}, a randomized variant of \texttt{DisSigUCB} that ensures persistent exploration (Theorem~\ref{thm:asymptotic_normality_eps_dissigucb}).}
    \item Validating the method empirically with synthetic experiments and applications to temperature sensor monitoring, sleep-stage classification, and hospital nurse staffing.
\end{itemize}

Our signature approach (i) linearizes nonlinear reward functionals and (ii) captures history dependence in the contextual trajectory. The truncated signature yields a finite-dimensional feature representation that encodes higher-order temporal interactions between path coordinates. A broad class of continuous functionals on path space can be approximated arbitrarily well by linear functionals of the signature. This allows us to model rich nonlinear rewards without committing a priori to a specific link function (as in GLMs) or a particular RKHS/kernel class, while still leveraging the algorithmic and theoretical machinery of linear bandit methods. 


Our regret analysis and technique differ from those of \cite{CMBOSOM:20}, which assumes that contexts are independent of the history and have a fixed conditional covariance structure. 
In contrast,  our setting allows dependent and nonstationary contextual features. 
As a result, the main challenge is to control the action-specific Gram matrices under adaptive action selection. 
We address this by applying an adapted matrix Chernoff inequality to the sequence of features observed when each action is selected, which yields a high-probability lower bound on the minimum eigenvalue of the corresponding Gram matrix.

To ensure valid statistical inference, we introduce $\zeta$-\texttt{DisSigUCB}, an $\epsilon$-greedy randomized variant that ensures persistent exploration. 
Our analysis builds on the martingale-based framework of Chen et al.~\cite{CLS:21}, but differs in that their analysis assumes i.i.d.\ contexts and plug-in greedy exploitation, whereas our setting allows {\it dependent and nonstationary} signature features while retaining UCB-based exploitation. 
Consequently, our analysis additionally requires establishing convergence of the adaptive design matrix and showing that both the estimation error and UCB bonus vanish.

\paragraph{Organization of the remainder of the paper.} 
Section~\ref{sec:problem} introduces the general contextual nonlinear bandit problem.
Section~\ref{ssec:sig_setup_bandit}  develops the linear functional representation using signature features for the contextual bandit problem.
Section~\ref{sec:dissigucb} proposes a new UCB-style algorithm \texttt{DisSigUCB}.  Section~\ref{sec:regret} derives a finite-time sublinear 
regret bound for \texttt{DisSigUCB}, with proofs deferred to Section~\ref{sec:detailed_proofs}. Section~\ref{sec:asym_norm} establishes  statistical inference for the linear estimates under $\zeta$-\texttt{DisSigUCB}, with proofs deferred to Section~\ref{sec:detailed_proofs_asym_norm}. Finally, Section~\ref{sec:exp} reports synthetic experiments and numerical applications.

\paragraph{Notations.} 
$\mathbb{N}=\{1,2, \ldots\}$ is the set of positive integers. For $n \in \sN$,  $[n]=\{1,2, \ldots, n\}$. $|\cdot|:= \|\cdot \|_2$ denotes the Euclidean norm of a vector and $\|\cdot \| := \|\cdot \|_\text{op}$ denotes the operator norm of a matrix. For a symmetric matrix $M \in \mathbb{R}^{m \times m}$,  $\lambda_{\min}(M)$ and $\lambda_{\max}(M)$ denote the smallest and largest eigenvalues of $M$, respectively. For a positive definite matrix $M \in \mathbb{R}^{m \times m}$,  $\|v\|_M := \sqrt{v^\top M v}$ is the weighted norm of vector $v$ with respect to $M$.  $\mathbf{I}$ is the identity matrix. 

$C([0, T] ; \mathbb{R}^d)$ is the set of continuous paths on $[0,T]$ taking values in $\mathbb{R}^d$.  $\pathspace := \{x \in C\left([0, T] ; \mathbb{R}^d\right):\|x\|_p<\infty\}$ is the space of finite $p$-variation paths defined on $[0,T]$, with norm $\|\cdot\|_{\cV^p} := \|\cdot\|_p + \|\cdot\|_\infty$. Here, for $x \in \pathspace$, $\|x\|_p :=(\sup_\Pi \sum_{i=1}^{m-1}|x_{t_{i+1}}-x_{t_i}|^p)^{1 / p} $ denotes the $p$-variation semi-norm where $\Pi=\{0=t_1<\cdots<t_m=T\}$  is a partition of $[0,T]$ with $m$ partition points and  $\|x\|_\infty := \sup_{t \in [0,T]} |x_t|$ denotes the sup norm. Additionally, for $x \in \pathspace$ and $\alpha \in (0,1]$,  $\|x\|_{\alpha\text{-Höl}} := \sup_{s \neq t, s,t \in [0,T]} \frac{|x_t - x_s|}{|t - s|^\alpha}$ denotes the $\alpha$-Hölder semi-norm. For a path $x \in \pathspace$, $x_{[a,b]}$ denotes the segment $(x_t)_{t\in [a,b]} \in \cV^p([a,b], \sR^d)$. 

\section{Problem Formulation}
\label{sec:problem}
We study the following contextual bandit problem with nonlinear and path-dependent rewards.
Let $T \in \sN$ be the time horizon and $K \in \sN$ be the number of actions.
The context $X := (X_s)_{s \in[0, T]}$ evolves in continuous time as a stochastic process taking values in $\mathbb{R}^d$, where the learner interacts with the environment only at discrete rounds. 
Specifically, at each round $t \in[T]$, the learner observes a segment of the process $X_{[t-L,t]}$ of length $L$, and selects an action $A_t \in [K]$ to receive a reward $r_t \in \sR$ for this round. 
Consequently, this context sequence $\{X_{[t-L,t]}\}_{t\in[T]}$ is non-i.i.d., nonstationary, and inherently path-dependent.

Let $(\Omega, \mathcal{F},\mathbb{P})$ be a probability space that supports the stochastic process $X$. 
The information set is defined as a filtration
$\left\{\mathcal{F}_t\right\}_{t \in [T]}$, where
$
\mathcal{F}_t:=\sigma((X_s)_{s \leq t},(A_s, r_s)_{s=1}^{t-1}),
$
which represents all information available to the learner at round $t$ before choosing action $A_t$.  We additionally define the post-round filtration as 
$
\tilde\cF_t := \sigma((X_s)_{s \leq t},(A_s, r_s)_{s=1}^t).
$

We assume the sample paths of $X$ belong to the space of continuous paths of bounded $p$-variation for $p \geq 1$, $\pathspace$, equipped with the norm $\|\cdot\|_{\cV^p}$. This includes a broad class of processes: for $p=1$,  $\cV^p$ contains paths of bounded variation, such as Lipschitz continuous contexts. Higher values of $p$ correspond to increased ``roughness'' and lower Hölder regularity; for instance, sample paths of Brownian motion and general diffusion processes almost surely have finite $p$-variation for any $p > 2$.

At each round $t \in [T]$, the learner observes a length-$L$ segment of the process, denoted by $X_{[t-L,t]} := \{X_{t-L + s} : s \in [0,L]\}.$ 
Based on the current and past observations represented by $\cF_t$, the learner selects an $\cF_t$-measurable action $A_t$ and then receives a reward for round $t$,$$r_t = f_{A_t}(X_{[t-L,t]}) + \eta_t,$$
where, for each action $a \in [K]$, the unknown function $f_a : \pathspaceL \to \mathbb{R}$ is a continuous (not necessarily linear) path-dependent functional. 
We emphasize that each action is associated with its own independent functional $f_a$, and no structural relationship across actions is assumed.

The noise sequence $(\eta_t)_{t=1}^T$ is conditionally mean-zero and
1-sub-Gaussian. Specifically, for every action $a\in[K]$,
\(\mathbb E[\eta_t\mid\mathcal F_t,A_t=a]=0\), \(\mathbb E[\eta_t^2\mid\mathcal F_t,A_t=a] =\sigma_a^2\in(0,1],\) and 
\(
\mathbb E\left[\exp(\alpha\eta_t) \mid\mathcal F_t,A_t=a \right] \leq
\exp\left(\frac{\alpha^2}{2}\right),
\)
for any $\alpha \in \mathbb{R}$. 
The conditional sub-Gaussian assumption implies that $\eta_t$ has Gaussian-type tail decay. Examples include conditionally standard Gaussian noise $\eta_t | (\cF_t, A_t=a) \sim \cN(0,1)$, and conditionally mean-zero bounded noise with $|\eta_t| \leq 1$ almost surely.


The learner’s objective is to maximize the expected cumulative reward
\(
\mathbb{E} [\sum_{t=1}^T r_t].
\) 
It is evaluated against a dynamic oracle that selects an optimal action at each round. Accordingly, the dynamic regret with respect to the learner's actions \(\{A_t\}_{t \in [T]}\) is given by
\begin{equation}\label{eq:regret_1}
    R(T) = \sum_{t=L}^T \left( \max_{a \in [K]} f_a(X_{[t-L,t]}) - f_{A_t}(X_{[t-L,t]}) \right).
\end{equation}

Examples illustrating our framework are provided below, with further details in Section \ref{sec:exp}.

\begin{example}[Monitoring Temperature Sensor Networks] Consider a temperature sensor network with a small subset of always-observable anchor sensors. At each hour \(t\), the learner chooses one sensor \(A_t\) to track the network-wide maximum temperature. At the end of the hour, the network-wide maximum is revealed and the learner receives the negative error reward,
$
r_t = - \left|\max_{s \in [t, t+1)} X_{s, A_t}  - \max_{s \in [t, t+1), ~k \in [K]} X_{s, k} \right|,
$
where $X_{s, k}$ is the temperature of sensor \(k\) at time \(s\). The reward is nonlinear and path-dependent, as it involves the absolute deviation and the maximum of the temperature trajectories over the hour.
\end{example}

\begin{example}[Sleep Stage Classification] Consider an online sleep-stage classification task from a stream of physiological signals, where the learner periodically predicts the current sleep state. Since a window may contain multiple stages, let \(P_t(a)\) be the fraction of the window labeled as stage \(a\). After choosing \(A_t\), the learner receives a reward $r_t=P_t(A_t)$, so the optimal action is the dominant stage. The proportions \(P_t(a)\) are nonlinear and path-dependent because sleep stages are determined by temporal patterns in the signals (e.g., oscillations, frequency, and transitions) over the full window.
\end{example}

\begin{example}[Hospital Nurse Staffing] Consider a hospital staffing problem where, at the end of each week \(t\), the learner observes the emergency department visit counts and chooses a nurse staffing level \(A_t\) for the following week. Let \(D_{t+1}\) be the nurse demand in week \(t+1\). The learner receives a reward equal to the negative Newsvendor loss
\(
r_t = -\left(b(D_{t+1}-A_t)_+ + h(A_t-D_{t+1})_+\right),
\)
where \(b>0\), \(h>0\) are the underage and overage costs. This reward is nonlinear and path-dependent, as the nurse demand for the following week $D_{t+1}$ is driven by the recent trajectory of emergency department visits, including trends and spikes.
\end{example}

\section{Linear Functional Representation of Contextual Bandits by Signature}\label{ssec:sig_setup_bandit}
We now introduce the signature transform, a rough path tool \cite{C:54, C:57} recently developed in rough path theory \cite{lyons1998differential} to capture the geometric structure of paths via iterated integrals.
\begin{definition}[Signature transform]
Fix a continuous path $x\in \pathspaceL$ and $N \in \sN.$ We denote the time-augmented path by $\hat{x}_t = (t, x_t) \in \mathbb{R}^{d+1}$ for $t\in[0,L].$
The depth-$N$ signature transform of $\hat{x}$ is
\[
\Sig_N(\hat{x}) := \Bigg(1, \left\{\Sig^{(i)}(\hat{x})\right\}_{i=0}^d, 
\left\{\Sig^{(i,j)}(\hat{x})\right\}_{i, j=0}^d, \cdots, \left\{\Sig^{\left(i_1, \ldots, i_N\right)}(\hat{x})\right\}_{i_1, \ldots, i_N=0}^d\Bigg) \in \mathbb{R}^{m_0},
\]
where $m_0 := \frac{(d+1)^{N+1} - 1}{d}$, and for any ordered multi-index $(i_1, \dots, i_k) \in \{0, \dots, d\}^k$, the corresponding signature coordinate of $\hat{x}$ is the iterated integral
$
\Sig^{(i_1, \dots, i_k)}(\hat{x}) := \int_{0<u_1<\cdots<u_k<L} \, \d \hat{x}_{u_1}^{i_1} \cdots \, \d \hat{x}_{u_k}^{i_k}.$

\end{definition}

\paragraph{ Universal nonlinearity of signature transform.}
A key property of the signature transform, known as universal nonlinearity \cite{A:18}, allows nonlinear functionals of a path to be represented as linear functionals of its signature. This makes the signature transform particularly useful for feature augmentation; see \cite{lyons2014feature, li2017lpsnet, morrill2020generalised, GGJKL:24}. 

\begin{prop}[Universal nonlinearity of signature transform \cite{A:18}]\label{prop:universal_nonlinear}
Let $\cK \subset \pathspaceL$ for $p \geq 1$ be a compact set and 
assume that each path $x \in \cK$ admits a well-defined signature transform.
For any continuous functional $f:\cK\to \mathbb{R}$ and any $\varepsilon>0$, there exists a depth $N \in \mathbb{N}$ and a coefficient vector $\beta\in \mathbb{R}^{m_0}$ (with  $m_0 := \frac{(d+1)^{N+1} - 1}{d}$) such that
\[
\sup_{x\in\cK}\, \big| f(x) - \langle \Sig_N(\hat{x}),\beta \rangle \big| < \varepsilon.
\]
\end{prop}
Proposition~\ref{prop:universal_nonlinear} applies to deterministic path spaces, bounded discrete-time processes after linear interpolation, and more generally continuous semimartingales (\cite[Theorem 2.12]{CGS:22}).
Empirical studies suggest that truncation depths $N=3$ or $4$ often yield small approximation errors \cite{YJNL:16,WNLH:2020,cao2023risk,GGJKL:24}. 
For It\^o processes, \cite{JGCYPLBMN:25} provide theoretical approximation guarantees using a stochastic Taylor expansion and the depth-$2$ Brownian signature. 
In our setting, Proposition~\ref{prop:universal_nonlinear} is applied separately to each action-specific reward functional $f_a$. 
Thus, finite-depth signatures provide linear approximations to the nonlinear path-dependent reward functions, with the finite-depth approximation error acting as model misspecification. 
This motivates approximating the original path-dependent bandit by a misspecified disjoint linear contextual bandit.


Motivated by Proposition~\ref{prop:universal_nonlinear}, we next construct the signature-based contextual features used by our bandit algorithm. Fix the truncation depth $N \in \mathbb{N}$ and let $\cK \subset \pathspaceL$ denote a compact set containing the contextual windows $X_{[t-L,t]}$.
For each round $t=L,\dots,T$, consider the time-augmented windowed context
$\hat{X}_{[t-L,t]}$, re-parameterized as
\[
    \hat{X}_{[t-L,t]}(s) := (s,X_{t-L+s}) \in \mathbb{R}^{d+1},
    \qquad s\in[0,L]
\]
and its corresponding depth-$N$ signature feature vector $\Sig_N(\hat{X}_{[t-L,t]})$.

{\paragraph{Resolving collinearity via signature pruning.}
The collection of signature feature vectors $\Sig_N(\hat{X}_{[t-L,t]})$ may exhibit collinearity due to the algebraic structure of the signature transform (Lemma \ref{lemma:shuffle}). We therefore introduce a technique, referred to as \emph{signature pruning}, which removes linearly dependent terms from the depth-$N$ signature feature vector.

To state the result precisely, we introduce some notation. We refer to any ordered multi-index $w = (i_1, \dots, i_k) \in [d]^k$ for any $k > 0$ as a \emph{word} with length $|w| = k$.   
This pruning procedure is formalized in the following proposition:
\begin{prop}[Signature pruning] \label{prop:prune_time}  Let $\hat{X} = (\hat{X}^0,\hat{X}^1,\dots,\hat{X}^d)$ defined on $[a,b]$, be the time-augmented path of $X$ in $\mathbb{R}^{d+1}$ where
\(
  \hat{X}^0_t = t, \hat{X}^j_t = X^j_t \)  
  for $ j=1,\dots,d, $ and $ t \in [a,b].$
Fix a truncation level $N \in \mathbb{N}$. Then for any word $w$ with length $|w| \leq N$ such that  $w$ ends with the time coordinate $0$,
the corresponding signature coordinate $\Sig^w(\hat{X})$ lies in the linear span of
\begin{equation}
B_N(\hat X) := \{\Sig^v(\hat{X}) : |v| \leq N, \, \text{last}(v) \neq 0\} \cup \{1\}, 
\end{equation}
i.e., the set of coordinates whose words have length at most $N$ and whose last letter is not the time channel.
\end{prop}
\begin{remark}
\begin{enumerate}
    \item[(i)] Proposition \ref{prop:prune_time} implies that signature coordinates ending in the time coordinate are redundant. Indeed, coordinates whose words end in $0$ and contain at least one spatial component can be expressed as linear combinations of those not ending in $0$, while pure time coordinates are constants. Hence, we remove all signature features ending in the time coordinate, retaining only one constant term.

    \item[(ii)] The size of $B_N(\hat X)$, including the empty word $1$, is \(m:=(d+1)^N.\) By contrast, the full truncated signature of the $(d+1)$-dimensional time-augmented path up to depth $N$, including $1$, has dimension \(\frac{(d+1)^{N+1}-1}{d}.\) Hence, as $N\to\infty$, the pruned representation reduces the truncated signature dimension by an asymptotic factor of $1/(d+1)$.
\end{enumerate}
\end{remark}


Let \(B_N(\hat X)\subset \Sig_N(\hat X)\) denote the resulting pruned set of signature features, with dimension \(m < m_0\).
Then define the corresponding depth-$N$ augmented context feature by
\begin{equation}\label{eq: sig_feature}
\bar x_t \coloneqq \left[X_t, \, B_N(\hat{X}_{[t-L,t]})\right] \in \sR^{d+m}.
\end{equation}
}
We include the terminal point $X_t$ to allow the feature representation to capture both the signature of the path increments and the absolute level of the contextual window.


\paragraph{Finite-depth truncation error.}
We next define the finite-depth approximation error from applying Proposition~\ref{prop:universal_nonlinear}. 
Fix $S_N > 0$, then define for each action $a \in [K]$, the unknown coefficient vector,
\begin{equation}
    \beta^*_a \in \argmin_{|\beta| \leq S_N} \sup_{X \in \cK} |f_a(X) - \langle \beta, \bar{x}^N(X)\rangle|,
\end{equation}
where $\bar{x}^N(X) = [X(L), B_N(\hat{X})]$ given by \eqref{eq: sig_feature}.
We define the corresponding finite-depth approximation error as
\begin{equation}\label{eq:finite_depth_truncation}
    \vep_N := \max_{a \in [K]} \sup_{X \in \cK} |f_a(X) - \langle \beta^*_a, \bar{x}^N(X)\rangle|.
\end{equation}
Hence, for every action $a \in [K]$ and $X \in \cK$, 
\begin{equation} \label{eq:approx_error}
|f_a(X) - \langle \beta^*_a, \bar{x}^N(X)\rangle| \leq \vep_N.
\end{equation}

Thus, $\vep_N$ explicitly captures the approximation error induced by the finite-depth representation.
Proposition~\ref{prop:universal_nonlinear} motivates the use of signature features through their universal approximation property, but does not provide a quantitative relationship among the approximation error $\vep_N$, the depth $N$, the feature dimension $d$, and the coefficient norm $S_N$. 
We therefore retain these quantities explicitly in the regret analysis of Section \ref{sec:regret}.



As a direct consequence of Proposition \ref{prop:universal_nonlinear} with \eqref{eq:regret_1} and \eqref{eq:finite_depth_truncation}, we can relate the regret defined in \eqref{eq:regret_1} to the regret under the finite-depth linear representation. 
\begin{theorem}\label{thm:linear_regret} The cumulative regret $R(T)$ in \eqref{eq:regret_1} satisfies
\begin{equation}
\left|R(T) - \sum_{t=L}^T \max_{a \in [K]} \left\langle \beta_a^* - \beta_{A_t}^*, \bar{x}_t \right\rangle \right|
\leq 2(T - L + 1)\vep_N,
\end{equation}
where \(A_t\) denotes the action selected by the learner at round \(t\).
\end{theorem}




\section{\texttt{DisSigUCB} Algorithm}
\label{sec:dissigucb}

This section develops an  upper-confidence-bound (UCB)-based algorithm \texttt{DisSigUCB} via the signature approach.

This algorithm combines the key idea of \texttt{DisLinUCB} \cite{LCLLCB:11} for {\it linear} reward functions and the linear representation established in Theorem \ref{thm:linear_regret} for general {\it nonlinear}  reward functions. 
It has two objectives: (i) estimating the coefficients \(\beta_a^*\) for each action \(a \in [K]\) to approximate  {\it nonlinear}  reward functions, and (ii) selecting actions \(A_t\) using a UCB-based strategy to encourage exploration.

\paragraph{Estimating parameter $\beta^*_a$ via regression.} Given the linear representation    for general {\it nonlinear} reward functions
(Theorem \ref{thm:linear_regret}), and given that the signature feature map may be high dimensional, it is natural to consider a regularized regression method to estimate the parameter $\beta_a^*$.
We adopt ridge regression as it admits closed-form confidence sets that are crucial for constructing UCB-style exploration bonuses.
 
Denote 
\(
\cT_{a,t} := \{s \in \{L, \dots, t-1\} : A_s = a\}
\)
as the set of rounds prior to round $t$ in which the decision-maker selected action $a \in [K]$, and let $\tau_{a,t} := |\cT_{a,t}|$ denote its cardinality. By Theorem \ref{thm:linear_regret},

\begin{equation}
    r_t = \langle \beta^*_{A_t, N}, \bar{x}_t \rangle + \vep_{A_t,N}({X_{{[t-L,t]}}}) + \eta_t,
\end{equation}
where $\vep_{a,N}(X) := f_a(X) - \langle \beta^*_{a,N}, \bar{x}^N(X) \rangle$ and $|\vep_{a,N}| \leq \vep_N$.

To estimate $\beta_a^*$ at round $t$, we use the training data $\{(\bar{x}_s, r_s)\}_{s \in \cT_{a,t}}$, where $\bar{x}_s$ is the feature vector and $r_s$ is the corresponding reward. By ridge regression,  we define for each $a \in [K]$,
\begin{equation}
\hat{\beta}_{a,t} := \arg\min_{\beta} \ \|\mathbf{r}_{a,t} - \mathbf{X}_{a,t}\beta\|^2 + \lambda \|\beta\|^2, \quad \lambda > 0,
\end{equation}
where $\mathbf{r}_{a,t} := (r_s)_{s \in \cT_{a,t}}$ and $\mathbf{X}_{a,t} := (\bar{x}_s^\top)_{s \in \cT_{a,t}}$. The closed-form solution is 
\begin{equation}
    \begin{aligned}
\label{eq: M}
\hat{\beta}_{a,t} &= \mathbf{M}_{a,t}^{-1} u_{a,t},\\
\mathbf{M}_{a,t} &:= \sum_{s \in \cT_{a,t}} \bar{x}_s \bar{x}_s^\top + \lambda \mathbf{I} = \mathbf{X}_{a,t}^\top \mathbf{X}_{a,t} + \lambda \mathbf{I}, 
\quad 
u_{a,t} := \mathbf{X}_{a,t}^\top \mathbf{r}_{a,t}.
\end{aligned}
\end{equation}
Thus, $\hat{\beta}_{a,t}$ represents our estimate of $\beta_a^*$ at round $t$ based on past observations. In practice, instead of storing the entire dataset $\{(\bar{x}_s, r_s)\}_{s \in \cT_{a,t}}$, we keep the sufficient statistics $\mathbf{M}_{a,t}$ and $u_{a,t}$. Once an action $a$ is selected, these quantities are updated using the newly observed data.

\paragraph{UCB for selecting action $A_t$.}
Having established a method to estimate the parameters \(\beta_a^*\), we adopt a UCB approach to select actions that balance reward maximization and uncertainty in the estimates. 
Specifically,
in round \(t\), we select the action according to the following upper confidence bound:
\begin{equation}
A_t = \arg\max_{a \in [K]} \left\{ \langle \bar{x}_t, \hat{\beta}_{a,t} \rangle 
+ \gamma \|\bar{x}_t\|_{\boldsymbol{M}_{a,t}^{-1}} \right\}.
\end{equation}
The exploration term $\|\bar{x}_t\|_{\boldsymbol{M}_{a,t}^{-1}}$ is motivated by
the following property (Lemma \ref{prop:x_beta}): with high probability,
$| \langle \bar{x}_t, \beta_a^* \rangle - \langle \bar{x}_t, \hat{\beta}_{a,t} \rangle|
\leq (\gamma  + \vep_N \sqrt{\tau_{a,t}})\|\bar{x}_t\|_{\boldsymbol{M}_{a,t}^{-1}}, \forall a \in [K], \; t \in [T].
$
Thus, the exploration coefficient $\gamma$ controls the stochastic estimation error, while the error arising from the signature approximation is treated separately and accounted for explicitly in the regret analysis of Section~\ref{sec:regret}.

{\small 
\begin{algorithm}[H]
\caption{\texttt{DisSigUCB}($\lambda$, $\gamma$, $L$, $N$)}
\label{alg:dissigucb}
\begin{algorithmic}[1]
{\small 
\Statex \textbf{Input:} Regularizer $\lambda > 0$, exploration coefficient $\gamma > 0 $, signature depth $N$, window length $L$
\State For every action $a \in [K]$, initialize $\hat\beta_{a, L} = \mathbf{0}$, $\mathbf{M}_{a,L} = \lambda\mathbf{I}$, $u_{a,L} = \mathbf{0}$
\For{$t=L,\dots,T$}
    \State Observe context $X_{[t-L,t]}$ 
    \State Construct signature feature $\bar{x}_t$ from $X_{[t-L,t]}$ and do signature pruning to remove collinear features \label{line:x_bar}
    \State $A_t = \arg\max_{a \in [K]} \ \langle \bar{x}_t, \hat{\beta}_{a,t} \rangle + \gamma \|\bar{x}_t\|_{\mathbf{M}_{a,t}^{-1}}$ \label{line:ucb-choice}
    \State Play action $A_t$ and observe reward $r_t$ 
    \State Set $u_{A_t, t+1} \gets u_{A_t, t} + r_t \bar{x}_t$, $\mathbf{M}_{A_t, t+1} \gets \mathbf{M}_{A_t, t} + \bar{x}_t \bar{x}_t^\top$, $\hat\beta_{A_t, t+1} \gets \mathbf{M}_{A_t,t+1}^{-1} u_{A_t, t+1}$
    \For{$a \in [K] \setminus \{A_t\}$} 
        \State Set $u_{a, t+1} \gets u_{a, t}$, $\mathbf{M}_{a, t+1} \gets \mathbf{M}_{a, t}$, $\hat\beta_{a, t+1} \gets \hat\beta_{a, t}$
    \EndFor
\EndFor
}
\end{algorithmic}
\end{algorithm}
}
This yields a UCB-based algorithm \texttt{DisSigUCB}, detailed in  Algorithm \ref{alg:dissigucb}, taking as input the ridge penalty $\lambda$, the exploration coefficient $\gamma$, the window length $L$, and the signature depth $N$.

\section{Regret Analysis of \texttt{DisSigUCB}} \label{sec:regret}
This section establishes a cumulative regret bound for \texttt{DisSigUCB}, based on the following:

\begin{outstanding_assumption} \label{assum:unified} 
\begin{enumerate}
    \item[(i)] Let $\cK$ be a compact subset of $\pathspaceL$.
    With  high probability, 
    every context window $X_{[t-L,t]}$ belongs to $\cK$ and the context process $X$ is uniformly bounded, i.e., \(\|X\|_\infty \leq H\) for some constant \(H\).
    \item[(ii)] For each action $a\in[K]$, let $s_{a,j}$ denote the round in which action $a$ is selected for the $j$-th time, and define $\mathcal{G}_{a,j}:=\widetilde{\mathcal{F}}_{s_{a,j}}$, with $\mathcal{G}_{a,0}:=\widetilde{\mathcal{F}}_{L-1}$.
    There exists a constant $\rho>0$ such that, for every action $a\in[K]$ and every $n\geq 1$ such that action $a$ has been selected at least $n$ times,
    \[
        \frac{1}{n}\sum_{j=1}^{n}
        \mathbb{E}\!\left[
            \bar{x}_{s_{a,j}}\bar{x}_{s_{a,j}}^\top
            \mid \mathcal{G}_{a,j-1}
        \right]
        \succeq \rho I
        \qquad \text{a.s.}
    \]
\end{enumerate}
\end{outstanding_assumption}


\begin{remark}[Comparison with standard non-degeneracy conditions]
    Related conditions in contextual bandits typically require pointwise non-degeneracy of the conditional second-moment matrix. For example, Chatterji et al.~\cite{CMBOSOM:20} assume
$$
\mathbb{E}[\bar x_{a,t} \bar x_{a,t}^\top\mid \cF_{t-1}] \succeq \rho \mathbf{I}
$$
for every action $a$ and round $t$, while Das and Sinha~\cite{DSLCB:24} impose the analogous condition on the action $a_t$ selected by the policy. In contrast,
Assumption~\ref{assum:unified}(ii) relaxes the pointwise requirement in two ways. 
First, non-degeneracy is required only along the sequence of rounds at which action $a$ is actually selected, i.e., only for the observations that enter its action-specific design matrix. 
Second, it only requires non-degeneracy cumulatively over the rounds in which a given action is selected. Individual observations may therefore be degenerate, provided that sufficient variation accumulates across selections for the corresponding action-specific Gram matrix to grow in all directions.
\end{remark}


\paragraph{Compactness and boundedness of context process.} Assumption~\ref{assum:unified}(i) imposes compactness and uniform boundedness on the context process $X$. This assumption also holds naturally in discrete-time models with bounded noise, and with high probability for more general diffusion processes over a finite time horizon. In the following lemma, we provide concrete conditions for a compact subset of the path space $\pathspaceL$, as required in Assumption~\ref{assum:unified}(i), and give its proof in Section~\ref{ssec:proof_compact_K}.

\begin{lemma}[Compact subsets of $\pathspaceL$]\label{lemma:compact_K_construction}
For given constants $H_1, H_2 > 0$, and $\a > \frac{1}{p}$,
consider the set $\cK \subset \pathspaceL$ defined by
\begin{equation}
  \cK \coloneqq \{ Y \in \pathspaceL : \|Y\|_\infty \leq H_1, \|Y\|_{\alpha\text{-Höl}} \le H_2 \}.  
\end{equation}
Then $\cK$ is compact with respect to $\| \cdot\|_{\cV^p}.$
\end{lemma}

\paragraph{Non-degeneracy condition via signature pruning.}
Assumption \ref{assum:unified}(ii) imposes an action-specific non-degeneracy condition on the feature distribution, which is often assumed in contextual bandits to 
ensure sufficient feature variation and growth of the empirical design matrices.
In particular, a lower bound on the minimum eigenvalue of the feature covariance matrix is needed to ensure that the parameter is identifiable and can be consistently estimated from data \cite{LCLLCB:11, APSLSB:11, DSLCB:24}.

Under the signature pruning strategy (Proposition~\ref{prop:prune_time}), we empirically verify that Assumption \ref{assum:unified}(ii) holds for diffusion processes (Brownian motion, geometric Brownian motion, and the Ornstein-Uhlenbeck process); see Section \ref{app:verify_pd} for details.

\paragraph{Boundedness of signature feature.} Assumption~\ref{assum:unified}(i) and the assumption in Proposition~\ref{prop:universal_nonlinear} imply that the context vectors $\{\bar{x}_t\}_{t\in[T]}$ \eqref{eq: sig_feature} induced by the signature transform are uniformly bounded. That is, 

\begin{lemma}[Bounded contextual signature feature] \label{prop:sig_l2}  Suppose that Assumption \ref{assum:unified}(i) and the assumption in Proposition~\ref{prop:universal_nonlinear} hold. Then, with probability at least $1 - \delta$, there exists a constant $b > 0$ such that
\begin{equation}
    |\bar{x}_{t,i}| \leq b, \quad \text{for all } t \in [T], \;i \in [d+m],
\end{equation}
where $\bar x_t$ is defined in \eqref{eq: sig_feature}. 
Consequently, 
\begin{equation} \label{eq:context_bound_B}
    |\bar{x}_t| = \sqrt{\sum_{i=1}^{d+m} |\bar{x}_{t,i}|^2} \leq b \sqrt{d+m} =: B, \quad \text{ for all } t = L, \dots, T.
\end{equation}
\end{lemma}
Its proof is detailed in Section~\ref{ssec:proof_sig_l2}.

\subsection{Regret bound} We are now ready to derive  a regret bound for \texttt{DisSigUCB} that accounts for the approximation error.
\begin{theorem}[Regret of \texttt{DisSigUCB}]
\label{thm:sigUCB_regret}
Take the same assumptions as in Proposition~\ref{prop:universal_nonlinear} and Assumption \ref{assum:unified}. Denote $S := \max_{a\in[K]}|\beta_a^*|$. Let $\delta \in (0,1)$. Then, with probability at least $1 - 3\delta$, the regret $R(T)$ of \texttt{DisSigUCB}$(\lambda,\gamma,N)$ in Algorithm~\ref{alg:dissigucb} satisfies
\begin{equation}
R(T) \lesssim d_\Sig \left(\sqrt{\log\left(\frac{K\left(1 + \frac{T}{d_\Sig}\right)}{\delta}\right)} + bS\right) \sqrt{KT \log\left(1 + \frac{T}{d_\Sig}\right)} + \frac{b}{\sqrt{\rho}}\sqrt{d_\Sig\log\left(\frac{d_\Sig KT}{\delta}\right)}T\vep_N,
\end{equation}
where $d_\Sig := d+m$.

Here, Algorithm~\ref{alg:dissigucb} adopts the regularizer $\lambda = B^2$ and the exploration coefficient

\begin{equation}\label{eq: gamma}
    \gamma \coloneqq \sqrt{d_\Sig \log\!\left(\frac{K}{\delta}\left(1 + \frac{T b^2}{\lambda }\right)\right)} + S b \sqrt{ d_\Sig}.
\end{equation}

Here, $b$ is given in Lemma~\ref{prop:sig_l2}, $m$ is the signature feature dimension, and $\rho$ is given in Assumption~\ref{assum:unified}(ii).
\end{theorem}


Theorem~\ref{thm:sigUCB_regret} establishes $R(T) = \tilde{\cO}(d_\Sig\sqrt{KT} + \vep_N T \sqrt{d_\Sig})$ with $d_\Sig := d + m$, where $d$ is the context dimension and $m:=(d+1)^N$ is the depth-$N$ signature feature dimension. 
\begin{remark}[Non-degeneracy condition and regret scaling]
The non-degeneracy condition in Assumption \ref{assum:unified}(ii) is commonly used in contextual bandit analyses to ensure sufficiently rich feature variation \cite{CMBOSOM:20, DSLCB:24, ghosh2021problem, gentile2017context}. 
In our analysis, its primary role is to guarantee the growth of the Gram matrices, $\lambda_{\min}(\mathbf{M}_{a,t}) \gtrsim \rho\tau_{a,t}$, once action $a$ has been selected sufficiently often (cf. Lemma \ref{lemma:min_eig_V}). 
This controls the misspecification-weighted uncertainty terms through $\sqrt{\tau_{a,t}} \|\bar{x}_t\|_{\mathbf{M}_{a,t}^{-1}} = \cO(B/\sqrt{\rho})$  (cf. Lemma \ref{lemma:Wsel_Wopt}) and consequently yields the approximation-error contribution $\tilde\cO(\frac{b}{\sqrt{\rho}}\sqrt{d_\Sig} T\vep_N)$.
\end{remark}

\begin{remark}[Regret bound with signature depth $N$ and universal approximation error $\vep_N$]
Throughout the regret analysis, the signature depth $N$ is fixed, with corresponding approximation error $\vep_N$.
In the special case of exact finite-depth signature approximation, i.e., $\vep_N = 0$, the high-probability regret is of sublinear order $\tilde{\cO}(d_{\Sig} \sqrt{K T})$, which recovers the regret bound in \texttt{LinUCB} \cite{APSLSB:11}.
Because the universal nonlinearity of the signature in Proposition~\ref{prop:universal_nonlinear} is a Stone-Weierstrass-type result on path space, a uniform explicit relationship between $\vep_N$ and $N$ does not exist in general. Theorem~\ref{thm:sigUCB_regret} exhibits the tradeoff: increasing $N$ enlarges the feature dimension while potentially reducing $\vep_N$. 

Nevertheless, under suitable structural assumptions (e.g., domain properties of the path space and smoothness of the reward functional), explicit quantitative relationships between $\vep_N$ and $N$ can be established. 
For example, \cite{horvath2026fast} establish a quantitative rate for Itô diffusions: for a smooth class of functionals, $\vep_N = \mathcal{O}(N^{-2\gamma})$, where $\gamma > 0$ measures the smoothness of the functional. Under such a setting, 
\begin{equation}
    R(T) = \tilde\cO((d+1)^N \sqrt{KT} + N^{-2\gamma} T\sqrt{(d+1)^N} ). 
\end{equation}
Since the first term grows exponentially in $N$ while the second term decays only polynomially in $N$, this suggests that a small signature depth $N$ is preferable.
\end{remark}

\begin{remark}[Regret comparison with existing methods]
Theorem~\ref{thm:sigUCB_regret} gives
\(R(T)=\tilde \cO(d_\Sig\sqrt{KT}
+\varepsilon_NT \sqrt{d_\Sig}) \),
which has the same qualitative decomposition as regret bounds for misspecified linear contextual bandits \cite{FGMZ:20, TIHSFKK:21}: a stochastic estimation term and an additional term proportional to the approximation error. 
When $N=1$, the signature feature dimension satisfies $d_\Sig=\cO(d)$, yielding
\(
R(T)=\tilde \cO(d\sqrt{KT}+\varepsilon_1 T\sqrt d).
\)
In particular, under exact finite-depth realization $\vep_1 = 0$, this reduces to the standard $d \sqrt{KT}$-type scaling for a disjoint linear UCB formulation.

\end{remark}

The proof of Theorem~\ref{thm:sigUCB_regret} proceeds in three steps. 
First, Lemma~\ref{lemma:min_eig_V} gives a high-probability lower bound on the smallest eigenvalue of the Gram matrix $\mathbf{M}_{a,t}$ \eqref{eq: M} once an action has been selected sufficiently many times. 
Second, Lemma~\ref{lemma:x_S0_S1} decomposes the cumulative regret into the UCB uncertainty term $W_\text{UCB}$ and the approximation-error terms $W_\text{sel}$ and $W_\text{opt}$. Finally, Lemma~\ref{lemma:Wucb} bounds $W_\text{UCB}$ by controlling the cumulative uncertainty through the growth of the determinant of the Gram matrix, while Lemma~\ref{lemma:Wsel_Wopt} uses the eigenvalue guarantee of Lemma~\ref{lemma:min_eig_V} to control $W_\text{sel}$ and $W_\text{opt}$.

Unlike \cite{CMBOSOM:20}, which assumes the contexts are independent of the past and have a fixed conditional covariance matrix, our setting allows dependent and  nonstationary contextual features. Assumption~\ref{assum:unified}(ii) instead imposes an action-specific conditional non-degeneracy condition. The resulting challenge is to control the action-specific Gram matrices under adaptive action selection. Lemma~\ref{lemma:min_eig_V} addresses this by applying an adapted matrix Chernoff inequality to the sequence of features observed when each action is selected, yielding a high-probability lower bound on the corresponding minimum eigenvalue.




\paragraph{Proof of Theorem~\ref{thm:sigUCB_regret}.} 
The first step is to
decompose the cumulative regret into the UCB uncertainty term $W_\text{UCB}$, and the misspecification terms $W_\text{sel}$ and $W_\text{opt}$. 
The proof of the lemma is provided in Section~\ref{ssec:x_S0_S1}.

\begin{lemma}[Decomposition of regret] \label{lemma:x_S0_S1}
Suppose that Assumption \ref{assum:unified} holds and $|\bar{x}_t| \leq B$ for all $t$. Let $\delta \in (0,1)$. Then with probability at least $1-\delta$, one can upper bound the cumulative regret $R(T)$ as follows,
\begin{equation}
\label{eq:RT_UB_S0_S1}
R(T) \leq 2\gamma W_\text{UCB} + \vep_N(W_\text{sel} + W_\text{opt}) + 2(T-L+1)\vep_N ,
\end{equation}
where $W_\text{UCB} := \sum_{t=L}^T \|\bar{x}_t\|_{\mathbf{M}_{A_t,t}^{-1}}$, $W_{\text{sel}} := \sum_{t=L}^T \sqrt{\tau_{A_t,t}} \|\bar{x}_t\|_{\mathbf{M}_{A_t,t}^{-1}} $, $W_{\text{opt}} := \sum_{t=L}^T \sqrt{\tau_{A^*_t,t}} \|\bar{x}_t\|_{\mathbf{M}_{A^*_t,t}^{-1}}$, and $\gamma$ is given in \eqref{eq: gamma}.
\end{lemma}

In the following lemma, we provide an upper bound for $W_\text{UCB}$ in Lemma~\ref{lemma:x_S0_S1}, whose proof is in Section~\ref{ssec:proof_Wucb}.
\begin{lemma} \label{lemma:Wucb}
Suppose Assumption~\ref{assum:unified} holds, $\lambda \geq B^2$, and $|\bar{x}_t| \leq B$ for all $t$. Then,
\begin{equation}
    W_\text{UCB} \leq \sqrt{2KT(d+m)\log\left(1 + \frac{TB^2}{\lambda(d+m)}\right)}.
\end{equation}
\end{lemma}

Next, we upper bound $W_\text{sel}$ and $W_\text{opt}$ as defined in Lemma~\ref{lemma:x_S0_S1}, whose proof is deferred to Section~\ref{ssec:proof_Wsel_Wopt}:
\begin{lemma} \label{lemma:Wsel_Wopt}
Suppose that Assumption~\ref{assum:unified} holds and $|\bar{x}_t| \leq B$ for all $t$. Then, with probability at least $1 - \delta$, 
\begin{equation}
    W_\text{sel}, \,W_\text{opt} \leq (T-L+1) B \max \left\{\sqrt{T_0/ \lambda}, \sqrt{2 / \rho}\right\},
\end{equation}
where $T_0$ is given by \eqref{eq:T0}.
\end{lemma}

These terms in Lemma~\ref{lemma:Wsel_Wopt} are subsequently controlled using the cumulative growth of the Gram matrices together with the minimum-eigenvalue guarantee established in the following lemma (cf. Lemma~\ref{lemma:min_eig_V}). Under Assumption~\ref{assum:unified}(ii), the lemma provides a high-probability lower bound on the smallest eigenvalue of the Gram matrix once the corresponding action has been selected sufficiently many times.

\begin{lemma}[Lower bound on eigenvalues of Gram matrix] \label{lemma:min_eig_V} 
Suppose Assumption~\ref{assum:unified}(ii) holds and $|\bar{x}_t| \leq B$ for all $t$. 
Fix $\lambda > 0$ and $\delta \in (0,1)$. 
Then with probability at least $1 - \delta$, we have
\begin{equation}
    \lambda_{\min}(\mathbf{M}_{a,t}) \geq \lambda + \frac{\rho \tau_{a,t}}{2},
\end{equation}
for all $a \in [K]$ and $t \in [T]$ such that $\tau_{a,t} \geq T_0$, where
\begin{equation} \label{eq:T0}
    T_0 \coloneqq \left\lceil \frac{8B^2}{\rho} \log\left(\frac{(d+m)KT}{\delta}\right)\right\rceil. 
\end{equation}   
\end{lemma}
Its proof is motivated by Lemma 7 of \cite{CMBOSOM:20} and is given in Section~\ref{ssec:proof_lemma:min_eig_V}, where we apply an adapted matrix Chernoff inequality \cite{TMC:11} to the action-specific sequence of selected features.

Now that we have constructed upper bounds for the sums $W_\text{UCB}$ (Lemma~\ref{lemma:Wucb}), as well as $W_\text{sel}$ and $W_\text{opt}$ (Lemma~\ref{lemma:Wsel_Wopt}), we finish the proof of Theorem~\ref{thm:sigUCB_regret}.


\begin{proof}[Proof of Theorem~\ref{thm:sigUCB_regret}]
By a union bound over the boundedness event in Lemma~\ref{prop:sig_l2}, the confidence event in Lemma~\ref{lemma:x_S0_S1}
and the minimum-eigenvalue event in Lemma~\ref{lemma:Wsel_Wopt}, the corresponding bounds hold simultaneously with probability at least $1 - 3\delta$. 
Substituting the bounds from Lemmas~\ref{lemma:Wucb} and \ref{lemma:Wsel_Wopt} into \eqref{eq:RT_UB_S0_S1} and setting $\lambda=B^2$, we obtain
\begin{equation} \label{eq:RT_UB_1}
    R(T) \leq 2\gamma\sqrt{2KT(d+m)\log\left(1 + \frac{T}{d+m}\right)} + 2(T-L+1)\vep_N\left(\max\left\{\sqrt{T_0}, B\sqrt{2/\rho}\right\}+ 1 \right).
\end{equation}
Recall that $B = b \sqrt{d+m}$ as in \eqref{eq:context_bound_B},  $\gamma = \sqrt{(d+m)\log\!\left(\frac{K\left(1 + \frac{T}{d+m}\right)}{\delta}\right)} + bS\sqrt{d+m}$ for $\lambda=B^2$ as in \eqref{eq: gamma} and $T_0 = \left\lceil \frac{8B^2}{\rho} \log\left(\frac{(d+m)KT}{\delta}\right)\right\rceil$ as in \eqref{eq:T0}. 
Then we have, \begin{equation}
    \sqrt{T_0} = \cO\left(b \sqrt{\frac{d+m}{\rho}\log\left(\frac{(d+m)KT}{\delta}\right)}\right).
\end{equation}
Therefore, substituting $\gamma$, $T_0$, and $B$ into \eqref{eq:RT_UB_1} and writing only the asymptotically dominating terms give
\begin{equation}
    R(T) \lesssim d_\Sig \left(\sqrt{\log\left(\frac{K\left(1 + \frac{T}{d_\Sig}\right)}{\delta}\right)} + bS\right) \sqrt{KT \log\left(1 + \frac{T}{d_\Sig}\right)} + \frac{b}{\sqrt{\rho}}\sqrt{d_\Sig\log\left(\frac{d_\Sig KT}{\delta}\right)}T\vep_N
\end{equation}
with probability at least $1 - 3\delta$ for $d_\Sig = d+m$. This implies that
\begin{equation}
    R(T) \leq \tilde\cO\left(d_\Sig\sqrt{KT} ~+ \vep_N T\sqrt{d_\Sig}\right),
\end{equation}
where $\tilde{\mathcal{O}}$ hides poly-logarithmic factors in $K,T,\delta$ and constants depending on $b,\rho,S$.
\end{proof}

\section{Statistical Inference under $\zeta$-\texttt{DisSigUCB}} \label{sec:asym_norm}

This section develops statistical inference for the action-specific
reward parameters $\beta_a^*$ under the sampling scheme of $\zeta$-\texttt{DisSigUCB}, a randomized variant of \texttt{DisSigUCB}. 
We establish a finite-sample tail bound and consistency for the action-specific ridge estimators, and then prove their asymptotic normality under additional margin and long-run design conditions.

\subsection{Algorithm design of $\zeta$-\texttt{DisSigUCB}}
We introduce the following $\epsilon$-greedy variant of \texttt{DisSigUCB} whose increased exploration allows statistical inference for the resulting estimators. 
Let $(\zeta_t)_{t=1}^\infty$ be a sequence of exploration probabilities. In $\zeta$-\texttt{DisSigUCB}, with probability $\zeta_t$, an action is chosen uniformly and with probability $1 - \zeta_t$, the action selected by \texttt{DisSigUCB} is chosen. 
Concretely, under $\zeta$-\texttt{DisSigUCB}, the decision rule is defined as
\begin{equation} \label{eq:eps_rule}
    A_t = \begin{cases}
        a_t^\texttt{DisSigUCB} & \text{w.p. } 1 - \zeta_t \\
        a \in [K] & \text{w.p. } \frac{\zeta_t}{K},
    \end{cases}
\end{equation}
where $a_t^\texttt{DisSigUCB} := \arg\max_{a \in [K]} \bar{x}_t^\top \hat{\beta}_{a,t} + \gamma \|\bar{x}_t\|_{\mathbf{M}_{a,t}^{-1}}$ is the action chosen by \texttt{DisSigUCB}.

We also define a corresponding randomized version of the optimal oracle. 
With probability $1 - \zeta_t$, the oracle chooses the per-round optimal action. 
Concretely, under $\zeta$-oracle, the decision rule is defined as
\begin{equation}\label{eq:zeta_oracle}
    A_t^{\zeta,*} = \begin{cases}
        a_t^* & \text{ w.p. } 1 - \zeta_t \\
        a \in [K] & \text{ w.p. } \frac{\zeta_t}{K},
    \end{cases}
\end{equation}
where $a_t^* = \arg\max_{a \in [K]} \bar{x}_t^\top \beta^*_a$ is the optimal action in round $t$.



\paragraph{Assumptions for statistical inference.} Throughout Section~\ref{sec:asym_norm}, we continue to use the assumptions of Theorem~\ref{thm:sigUCB_regret} and assume that the approximation error from universal nonlinearity is zero ($\vep_N = 0$). For an extension to misspecified linear models, we refer to \cite{CLS:21}. In particular, for Assumption~\ref{assum:unified}(ii), the selection times $s_{a,j}$ and filtrations $\cG_{a,j}$ are understood to be those induced by $\zeta$-\texttt{DisSigUCB}. We additionally impose the following assumptions for asymptotic inference.



We first impose a weak time-averaged margin condition. Such conditions are commonly used in contextual bandit analyses to control the probability that the best and second-best actions have nearly equal expected rewards \cite{GZ:13,BB:20,BBK:21,CLS:21}. Unlike the usual i.i.d.\ margin condition, ours only requires this probability to vanish on average over time as the gap threshold approaches zero, allowing nonstationary contextual processes.

\begin{assumption} \label{assumption:ties}
    Define the oracle gap in round $t$ as $\Delta_t = \bar{x}_t^\top \beta^*_{a_t^*} - \max_{b \neq a_t^*} \bar{x}_t^\top \beta^*_b$. Assume the following,
    \begin{equation}
    \lim_{\delta \downarrow0} \limsup_{t \to \infty} \frac{1}{t} \sum_{s=L}^t \mathbb{P}(\Delta_s \leq \delta) = 0.
    \end{equation}
\end{assumption}

We also require the long-run oracle-weighted design matrix for each action to stabilize. 
This is the non-i.i.d.\ analogue of the limiting design matrix in standard asymptotic analyses under adaptive data collection \cite{CLS:21}. 
Its convergence determines the limiting covariance in the martingale central limit theorem. 
Without it, the normalized design matrix may have different subsequential limits, preventing identification of a single deterministic asymptotic covariance matrix.

\begin{assumption} \label{assumption:oracle_conv}
    Let the $\zeta$-oracle selection probability of action $a$ in round $s$ be 
    $
    \pi_{a,s}^* = \frac{\zeta_s}{K} + (1 - \zeta_s) \mathbf{1}\{a_s^* = a\}.
    $
    Assume that for every action $a$, there exists a deterministic matrix $Q_a$ such that
    \begin{equation}
    \frac{1}{t} \sum_{s=L}^{t-1} \mathbb{E}[\pi_{a,s}^* \bar{x}_s \bar{x}_s^\top | \tilde\cF_{s-1}] \stackrel{\sP}{\to} Q_a. 
    \end{equation}



\end{assumption}

We include some settings where Assumption~\ref{assumption:oracle_conv} holds. 

\begin{example} Consider i.i.d.\ contextual features with $\bar{x}_s \sim P_x$. Then the oracle-weighted design converges to
\begin{equation}
Q_a = \frac{\zeta_\infty}{K}\mathbb{E}[\bar{x}\bar{x}^\top]
+
(1-\zeta_\infty)
\mathbb{E}\left[
\bar{x}\bar{x}^\top \mathbf{1}\{a^*(\bar{x})=a\}
\right],
\end{equation}
where $\bar{x} \sim P_x$. This is the analogue of the limiting design matrix in the i.i.d.\ two-action setting in \cite{CLS:21}.
\end{example}

\begin{example} Consider the case where the contextual process alternates among $J$ stationary regimes,  with empirical regime frequencies converging to $p_j$ for $j= 1,\ldots,J$. Conditional on regime $j$, suppose the signature feature has distribution $P_j$. If $\zeta_s\to\zeta_\infty$, then the oracle-weighted design converges to
\begin{equation}
Q_a = \frac{\zeta_\infty}{K}
\sum_{j=1}^J p_j \mathbb{E}[\bar{x}_j\bar{x}_j^\top] + (1-\zeta_\infty) \sum_{j=1}^J p_j \mathbb{E}\left[\bar{x}_j\bar{x}_j^\top \mathbf{1}\{a^*(\bar{x}_j)=a\}
\right],
\end{equation}
where $\bar{x}_j \sim P_j$. Hence the assumption can hold even for nonstationary data, provided the long-run regime frequencies stabilize.
\end{example}

\subsection{Consistency of ridge estimator}\label{ssec:consistency}
We begin with a finite-sample concentration bound for the action-specific ridge estimator under $\zeta$-\texttt{DisSigUCB}. Persistent exploration ensures that every action is sampled sufficiently often, while the diversity condition prevents the action-specific design matrices from becoming degenerate. Consequently, each estimator $\hat\beta_{a,t}$ concentrates around its target $\beta_a^*$, which will be used to establish consistency and policy stabilization. Its proof is given in Section~\ref{ssec:proof_ridge_tail}.

\begin{prop}[Tail bound for the ridge estimator under $\zeta$-\texttt{DisSigUCB}]
\label{prop:ridge_tail_eps_dissigucb} 
Suppose the assumptions of Theorem~\ref{thm:sigUCB_regret} hold under \(\zeta\)-\texttt{DisSigUCB} \eqref{eq:eps_rule} and $\zeta_t \in [0,1]$ is a non-increasing deterministic sequence. Then, for any fixed $t \geq L+1$, $a \in [K]$, and failure probability $\delta_C \in (0,1)$, 
\begin{equation}
    \mathbb{P}\left(\|\hat{\beta}_{a,t} - \beta^*_a\|_1 \leq \frac{\gamma_t \sqrt{d+m}}{\sqrt{\lambda + \rho(t-L)\zeta_t / (4K)}}\right) \geq 1 - \delta_C - \exp\left(-\frac{(t-L)\zeta_t}{8K}\right) - (d+m) \exp\left(-\frac{\rho(t-L) \zeta_t}{16KB^2}\right),
\end{equation}
where
\begin{equation} \label{eq:gamma_t}
\gamma_t(\delta_C) := \sqrt{ (d+m)  \log \left(\frac{1 + \frac{tB^2}{\lambda(d+m)}}{\delta_C}\right)} + S\sqrt{\lambda}.
\end{equation}
\end{prop}

An immediate consequence of Proposition \ref{prop:ridge_tail_eps_dissigucb} is the consistency of the ridge estimates under $\zeta$-\texttt{DisSigUCB}, formalized in the following corollary, whose proof is also included in Section~\ref{ssec:proof_ridge_tail}. 

\begin{corollary}[Consistency under $\zeta$-\texttt{DisSigUCB}] \label{cor:eps_dissigucb_consistency} Suppose $\zeta_t \in [0,1]$ is a non-increasing deterministic sequence and that $t\zeta_t / \log t \to \infty$ holds. Then for every fixed action $a \in [K]$, 
    $\hat{\beta}_{a,t} \xrightarrow{\sP} \beta^*_a.$
\end{corollary}

\subsection{Asymptotic normality}\label{ssec:asymptotic_normality}
We find that under $\zeta$-\texttt{DisSigUCB}, for every action $a$, asymptotic normality holds. This finding is formalized in the following theorem. 

\begin{theorem}[Asymptotic Normality under $\zeta$-\texttt{DisSigUCB}] \label{thm:asymptotic_normality_eps_dissigucb}
Suppose that the assumptions of Theorem~\ref{thm:sigUCB_regret} and Assumptions~\ref{assumption:ties} and \ref{assumption:oracle_conv} hold. Let $\zeta_t \in [0,1]$ be a non-increasing deterministic sequence that converges to $\zeta_\infty \in (0,1)$ and suppose $\vep_N = 0$. Then for every action $a \in [K]$, we have $Q_a \succ 0$ and 
\begin{equation}
    \sqrt{t} \left(\hat{\beta}_{a,t} - \beta_a^* \right) \stackrel{d}{\to} \cN\left(0, \sigma_a^2 Q_a^{-1}\right).
\end{equation}
\end{theorem}

Our proof of Theorem~\ref{thm:asymptotic_normality_eps_dissigucb} follows the general martingale-based strategy of \cite{CLS:21} for statistical inference under $\epsilon$-greedy contextual bandits. In particular, both analyses establish convergence of the empirical design matrix and apply a martingale central limit theorem to the noise term. However, \cite{CLS:21} assume i.i.d.\ contexts and greedy exploitation, whereas our setting allows dependent and nonstationary signature features and retains UCB-based exploitation. Consequently, our analysis additionally requires showing that the UCB bonus vanishes, that the resulting action probabilities asymptotically approach those of the corresponding oracle policy (cf. Lemma \ref{lemma:action_conv}), and that the adaptive design matrix converges under the time-varying predictable context process (cf. Lemma \ref{lemma:M_at/t_conv}).

More specifically,
the proof of Theorem~\ref{thm:asymptotic_normality_eps_dissigucb} consists of  the following four steps:
First, we decompose the estimation error \(\hat{\beta}_{a,t}-\beta_a^*\) and reduce the asymptotic normality proof to showing two key results: 1) $\mathbf{M}_{a,t}/t \to Q_a$ and 2) $\frac{1}{\sqrt{t}} \sum_{s=L}^{t-1} \mathbf{1}\{A_s = a\} \bar{x}_s \eta_s \stackrel{d}{\rightarrow} \cN(0, \sigma_a^2 Q_a)$ (Lemma~\ref{lemma:diff_decomp}). To establish these, we first show that the \(\zeta\)-\texttt{DisSigUCB} decisions converge to the \(\zeta\)-oracle decision, which implies \(\mathbf{M}_{a,t}/t \to Q_a\) (Lemmas~\ref{lemma:action_conv} and \ref{lemma:M_at/t_conv}). We then apply a martingale central limit theorem to prove the convergence in distribution of the score term (Lemma~\ref{lemma:score_conv}). Finally, combining these results and applying Slutsky's theorem to the decomposition above yields the desired asymptotic normality.


In the following lemma,  we decompose the error $\hat{\beta}_{a,t} - \beta_a^*$ into $\mathbf{M}_{a,t}/t$ and $\frac{1}{\sqrt{t}}\sum_{s=1}^{t-1} \mathbf{1}\{A_s = a\} \bar{x}_s \eta_s$:
\begin{lemma} \label{lemma:diff_decomp}
    For each action $a \in [K]$ and round $t$, the time-scaled distance between the ridge estimate and the true parameter can be decomposed as
    \begin{equation}
        \sqrt{t}(\hat{\beta}_{a,t} - \beta_a^*) = \left(\frac{\mathbf{M}_{a,t}}{t} \right)^{-1} \left( \frac{1}{\sqrt{t}}\sum_{s=L}^{t-1} \mathbf{1}\{A_s = a\} \bar{x}_s \eta_s - \frac{\lambda}{\sqrt{t}} \beta_a^*\right).
    \end{equation}
\end{lemma}

We next show that $\mathbf{M}_{a,t}/t \stackrel{\sP}{\rightarrow} Q_a$. We first show that the action allocations chosen by $\zeta$-\texttt{DisSigUCB} \eqref{eq:eps_rule} converge to those of the $\zeta$-oracle \eqref{eq:zeta_oracle}. We denote $\pi^*_{a,t} \coloneqq \mathbb{P}(A_t^{\zeta,*} =a |\cF_t) $ and $ \pi_{a,t} \coloneqq \mathbb{P}(A_t = a | \cF_t)$.
The following lemma formalizes the convergence of the policies of $\zeta$-\texttt{DisSigUCB} and of the $\zeta$-oracle, whose proof is included in Section~\ref{ssec:proof_action_conv}.

\begin{lemma} \label{lemma:action_conv} Suppose that $\zeta_t \in [0,1]$ is a non-increasing deterministic sequence that converges to $\zeta_\infty \in (0,1)$ and that Assumption~\ref{assumption:ties} holds. Then the action probabilities $\pi$ under $\zeta$-\texttt{DisSigUCB} \eqref{eq:eps_rule} and $\pi^*$ under $\zeta$-oracle \eqref{eq:zeta_oracle} converge in probability, i.e.,
   $ \frac{1}{t} \sum_{s=L}^{t}|\pi_{a,s} - \pi^*_{a,s}| \stackrel{\sP}{\to} 0.$
\end{lemma}

Having shown that the action probabilities under $\zeta$-\texttt{DisSigUCB} are asymptotically indistinguishable from those under $\zeta$-oracle, we can now transfer the oracle long-run design limit in Assumption~\ref{assumption:oracle_conv} to the empirical design matrix. The next lemma formalizes this step by showing that the convergence of action allocations implies $\mathbf{M}_{a,t}/t$ converges in probability to $Q_a$, whose proof is included in Section~\ref{ssec:proof_M_at/t_conv}.
\begin{lemma} \label{lemma:M_at/t_conv}
    Under the assumptions of Theorem~\ref{thm:asymptotic_normality_eps_dissigucb}, \(\mathbf{M}_{a,t}/t \stackrel{\sP}{\to} Q_a.\)
\end{lemma}

In the next lemma, we aim to show that the second term in the decomposition in Lemma~\ref{lemma:diff_decomp}, $\frac{1}{\sqrt{t}}\sum_{s=L}^{t-1} \mathbf{1}\{A_s = a\} \bar{x}_s \eta_s$, converges in distribution to a multivariate Gaussian distribution. Its proof is included in Section~\ref{ssec:proof_score_conv}.

\begin{lemma} \label{lemma:score_conv}
Under the assumptions of Theorem~\ref{thm:asymptotic_normality_eps_dissigucb}, we have 
$ \frac{1}{\sqrt{t}} \sum_{s=L}^{t-1} \mathbf{1}\{A_s = a\} \bar{x}_s \eta_s \stackrel{d}{\to}~\cN(0, \sigma_a^2 Q_a).$
\end{lemma}

We are now ready to 
finish the proof of the asymptotic normality by combining the results from Lemma~\ref{lemma:M_at/t_conv} and Lemma~\ref{lemma:score_conv} in the decomposition from Lemma~\ref{lemma:diff_decomp}.

\begin{proof}[Proof of Theorem~\ref{thm:asymptotic_normality_eps_dissigucb}]


We first show that $Q_a$ is positive definite. 
In the proof of Proposition~\ref{prop:ridge_tail_eps_dissigucb} and Lemma \ref{lemma:action_conv}, we show that for each fixed action $a$, 
\begin{equation}
    \mathbb{P}\left(\lambda_{\min}(\mathbf{M}_{a,t}) \geq \lambda + \frac{\rho (t-L)\zeta_t}{4K}\right) \geq 1 - \exp\left(- \frac{(t-L)\zeta_t}{8K}\right) - (d+m) \exp \left(- \frac{\rho(t-L)\zeta_t}{16KB^2} \right) \to 1. 
\end{equation}
Equivalently, 
\begin{equation} \label{eq:lm_min_M/t}
    \mathbb{P}\left(\lambda_{\min}\left(\frac{\mathbf{M}_{a,t}}{t}\right) \geq \frac{\lambda}{t} + \frac{\rho (t-L)\zeta_t}{4Kt}\right) \to 1. 
\end{equation}
Since $\zeta_t \to \zeta_\infty \in (0,1)$, the above lower bound in \eqref{eq:lm_min_M/t} satisfies
\begin{equation}
    \frac{\lambda}{t} + \frac{\rho(t-L)\zeta_t}{4Kt} \to \frac{\rho \zeta_\infty}{4K} > 0.
\end{equation}
By Lemma~\ref{lemma:M_at/t_conv} and continuity of $\lambda_{\min}(\cdot)$, 
\begin{equation} \label{eq:lm_min_Mat_Qa}
    \frac{\lambda_{\min}(\mathbf{M}_{a,t})}{t} \stackrel{\mathbb{P}}{\to} \lambda_{\min}(Q_a).
\end{equation}
Combining \eqref{eq:lm_min_M/t}-\eqref{eq:lm_min_Mat_Qa} yields $\lambda_{\min}(Q_a) \geq \frac{\rho \zeta_\infty}{4K} > 0,$ and hence $Q_a \succ 0$.

Then, by the continuous mapping theorem, 
$
\left(\frac{\mathbf{M}_{a,t}}{t}\right)^{-1} \stackrel{\sP}{\to} Q_a^{-1}.
$
We also note that 
$    \frac{\lambda}{\sqrt{t}} \beta_a^* \to 0.$ By Slutsky's theorem, we combine the three convergence results to get
\begin{equation*}
    \left(\frac{\mathbf{M}_{a,t}}{t} \right)^{-1} \left( \frac{1}{\sqrt{t}}\sum_{s=L}^{t-1} \mathbf{1}\{A_s = a\} \bar{x}_s \eta_s - \frac{\lambda}{\sqrt{t}} \beta_a^*\right) \stackrel{d}{\to} \cN(0, Q_a^{-1} (\sigma_a^2 Q_a) (Q_a^{-1})^\top).
\end{equation*}
Finally, by symmetry of $Q_a$, we have the desired convergence claim, 
$    \sqrt{t}(\hat{\beta}_{a,t} - \beta_a^*) \stackrel{d}{\to} \cN(0, \sigma_a^2 Q_a^{-1}).$
\end{proof}
 
\paragraph{Data-driven confidence interval.}
The asymptotic covariance in Theorem~\ref{thm:asymptotic_normality_eps_dissigucb} depends on the unknown quantities $\sigma_a^2$ and $Q_a$, and therefore cannot be used directly to construct confidence intervals.
We estimate these quantities from the adaptively collected observations and use the resulting covariance estimator to construct confidence intervals for the coefficient estimates.

Following the variance estimation approach in \cite{CLS:21},
we estimate the action-specific noise variance $\sigma_a^2$ and asymptotic covariance matrix $\sigma_a^2 Q_a^{-1}$ by
\begin{equation} \label{eq:cov_estimator}
    \hat{\sigma}_{a,t}^2 = \frac{1}{\tau_{a,t}} \sum_{s=L}^{t-1} \mathbf{1}\{A_s = a\}(r_s - \bar{x}_s ^\top \hat{\beta}_{a,t})^2, \qquad \hat{\Sigma}_{a,t} = \hat{\sigma}_{a,t}^2 \left(\frac{\mathbf{M}_{a,t}}{t} \right)^{-1}.
\end{equation}

The following corollary shows that these estimators are consistent and yield asymptotically valid confidence intervals for the action-specific coefficients. The proof of Corollary \ref{cor:cov_consistency} is deferred to Section~\ref{ssec:proof_conf_int}.

\begin{corollary}[Variance Estimation and Confidence Intervals] \label{cor:cov_consistency}
    Under the assumptions of Theorem~\ref{thm:asymptotic_normality_eps_dissigucb}, for every action $a \in [K]$, 
    \begin{equation} \label{eq:student_stat}
        \hat{\sigma}_{a,t}^2 \stackrel{\mathbb{P}}{\to} \sigma_a^2, \qquad \hat{\Sigma}_{a,t} \stackrel{\mathbb{P}}{\to} \sigma_a^2 Q_a^{-1}, \qquad \frac{\hat\beta_{a,t,j} - \beta_{a,j}^*}{\sqrt{[\hat{\Sigma}_{a,t}]_{j,j}/t}} \stackrel{d}{\to} \cN(0,1),
    \end{equation}
    for every coordinate $j$.
    Therefore, an asymptotically valid $100(1-\alpha)\%$ confidence interval for $\beta^*_{a,j}$ is
    \begin{equation} \label{eq:conf_int}
        \hat{\beta}_{a,t,j} \pm z_{1 - \alpha/2} \sqrt{\frac{[\hat{\Sigma}_{a,t}]_{j,j}}{t}},
    \end{equation}
    where \(z_{1-\alpha/2}\) denotes the \((1-\alpha/2)\)-quantile of the standard normal distribution, and $\sqrt{\frac{[\hat{\Sigma}_{a,t}]_{j,j}}{t}}$ denotes the estimated standard error.
\end{corollary}
In Section~\ref{ssec:verify_stat_inf}, we validate that, under an exact linear reward model, the studentized statistic in \eqref{eq:student_stat} approaches a standard normal distribution, the ratio of the estimated standard error to the empirical standard deviation of the coefficient estimates approaches 1, and the empirical coverage of the 95\% confidence intervals approaches the nominal level of 95\%.

\section{Numerical Experiments}
\label{sec:exp}

In this section, we evaluate \texttt{DisSigUCB} on synthetic experiments and three real-data applications: temperature-sensor monitoring, sleep-stage classification, and hospital nurse staffing. 
We benchmark against \texttt{DisLinUCB}~\cite{LCLLCB:11} in the synthetic experiments and against \texttt{DisLinUCB} and \texttt{KernelUCB}~\cite{VKMFCKCB:13} in the real-data applications.

Because the baselines require vector-valued contexts, we transform each contextual path using four feature maps. 
The first uses the time average of the contextual trajectory as input to both \texttt{DisLinUCB} and \texttt{KernelUCB}. 
To test whether \texttt{DisSigUCB}'s gains arise solely from richer features, we also consider three \texttt{DisLinUCB} variants with rich feature maps.
First, \texttt{LinUCB-stats} represents each contextual window using seven temporal statistics: its mean, standard deviation, minimum, maximum, first value, last value, and slope. 
Second, \texttt{LinUCB-poly} applies a degree-two polynomial map to the time-averaged context, defined by $\phi(X_{[t-L,t]}) = [\bar{X}, \bar{X}^{\odot 2}, \bar{X}_i \bar{X}_j : i < j]$, where $\bar{X}$ is the time-average of $X_{[t-L,t]}$. 
Finally, \texttt{LinUCB-rff} approximates the radial basis function (RBF) kernel using 64 random Fourier features.

\subsection{Synthetic Experiments} \label{ssec:syn_exp}
We start with three synthetic contextual processes in $\sR$: nonstationary Brownian motion (BM), geometric Brownian motion (GBM), and a stationary Ornstein--Uhlenbeck (OU) process.

We set the time horizon $T=100$ and the window length $L=1$, and we simulate all contextual processes using the Euler--Maruyama method with $1000$ grid points per unit time interval. Thus, the time step  $\Delta t=10^{-3}$, and we write $t_k=k\Delta t$. Let $\varepsilon_k\sim \mathcal N(0,\Delta t)$ be i.i.d.\ Gaussian increments. The simulated processes are
\begin{align*}
    \text{BM:}\quad
    &X^{\mathrm{BM}}_{t_{k+1}}
    =
    X^{\mathrm{BM}}_{t_k}
    +
    \varepsilon_k,
    &
    X^{\mathrm{BM}}_{t_0}=0,
    \\
    \text{GBM:}\quad
    &Y_{t_{k+1}}
    =
    Y_{t_k}
    +
    \alpha Y_{t_k}\Delta t
    +
    \nu Y_{t_k}\varepsilon_k,
    &
    Y_{t_0}=1,
    \\
    \text{OU:}\quad
    &X^{\mathrm{OU}}_{t_{k+1}}
    =
    X^{\mathrm{OU}}_{t_k}
    + \theta( \mu - X^{\mathrm{OU}}_{t_k})
    \Delta t
    +
    \sigma\varepsilon_k,
    &
    X^{\mathrm{OU}}_{t_0}\sim \mathcal N(0,1/2).
\end{align*}
For the BM and OU models, the contextual processes are directly given by $X^{\mathrm{BM}}$ and $X^{\mathrm{OU}}$, respectively. For the GBM model, we apply an additional log-normalization step, as commonly used in practice. Specifically, for an observation window $[t-L,t]$, we use the log-return path $X^{\mathrm{GBM}}_{[t-L,t]}$, defined by
\[
    X^{\mathrm{GBM}}_s
    =
    \log\left(\frac{Y_s}{Y_{t-L}}\right),
    \qquad s\in[t-L,t],
\]
as the contextual vector.

\subsubsection{Empirical Validation of Assumption \ref{assum:unified}(ii).} \label{app:verify_pd}

For each action $a$, let $s_{a,j}$ denote the round at which action $a$ is selected for the $j$-th time, and let $\cG_{a,j} := \tilde\cF_{s_{a,j}}$. 
Assumption~\ref{assum:unified}(ii) requires that the conditional second moment of the pruned signature features accumulated along this action-specific selection sequence remain non-degenerate on average:
\[
    \frac{1}{n}\sum_{j=1}^{n}
    \mathbb{E}\!\left[
        \bar{x}_{s_{a,j}}\bar{x}_{s_{a,j}}^\top
        \mid \mathcal{G}_{a,j-1}
    \right]
    \succeq \rho I.
\]
A stronger sufficient condition is the pointwise action-conditional non-degeneracy condition
\begin{equation} \label{eq:strong_cond}
    \mathbb{E}[\bar{x}_t \bar{x}_t^\top| \tilde\cF_{t-1}, A_t = a] \succeq \rho \mathbf{I}.
\end{equation}
To see that \eqref{eq:strong_cond} implies Assumption~\ref{assum:unified}(ii), fix
$a$ and $j$, and define
\(
N_{a,t-1}:=\sum_{s=L}^{t-1}\mathbf{1}\{A_s=a\}.
\)
Then
\(
\{s_{a,j}=t\}
=
\{N_{a,t-1}=j-1\}\cap\{A_t=a\},
\) and \(
\{N_{a,t-1}=j-1\}\in\tilde\cF_{t-1}.
\)
Hence, for every $t$,
\begin{align*}
&\mathbb{E}\!\left[
\mathbf{1}\{s_{a,j}=t\}\bar{x}_t\bar{x}_t^\top
\mid \mathcal{G}_{a,j-1}
\right] \\
&\quad=
\mathbb{E}\!\left[
\mathbf{1}\{N_{a,t-1}=j-1\}
\mathbb{E}\!\left[
\mathbf{1}\{A_t=a\}\bar{x}_t\bar{x}_t^\top
\mid\tilde\cF_{t-1}
\right]
\mid\mathcal{G}_{a,j-1}
\right] \\
&\quad=
\mathbb{E}\!\left[
\mathbf{1}\{N_{a,t-1}=j-1\}
\sP( A_t=a \mid\tilde\cF_{t-1})
\mathbb{E}\!\left[
\bar{x}_t\bar{x}_t^\top
\mid\tilde\cF_{t-1},A_t=a
\right]
\mid\mathcal{G}_{a,j-1}
\right] \\
&\quad\succeq
\rho I\,
\sP(s_{a,j}=t\mid\mathcal{G}_{a,j-1}),
\end{align*}
where the last inequality follows from \eqref{eq:strong_cond}. Summing
over $t$ gives
\[
\mathbb{E}\!\left[
\bar{x}_{s_{a,j}}\bar{x}_{s_{a,j}}^\top
\mid\mathcal{G}_{a,j-1}
\right]
\succeq \rho I.
\]
Therefore, averaging over $j=1,\ldots,n$ yields
the condition in Assumption~\ref{assum:unified}(ii).

We provide empirical evidence for \eqref{eq:strong_cond} using a nested Monte Carlo procedure. 
For each of the 20 independent random seeds, we run \texttt{DisSigUCB} with $\gamma = 1$ for 100 rounds. 
Every tenth round $t$, we freeze the realized history and algorithm state through round $t-1$. 
Conditional on this state, we generate 200 independent continuations of the context process to round $t$. 
For each continuation $j$, we compute the resulting pruned signature feature $\bar{x}_t^{(j)}$ and the action $A_t^{(j)}$ selected by \texttt{DisSigUCB}. 
For each action $a$, define
\(
n_{a,t} = \sum_{j=1}^{200} \mathbf{1} \{A_t^{(j)} = a\}.
\)
We then estimate $\mathbb{E}[\bar{x}_t \bar{x}_t^\top| \tilde\cF_{t-1}, A_t = a]$ by
\begin{equation}
    \hat\Sigma_{a,t} = \frac{1}{n_{a,t}} \sum_{j: A_t^{(j)} = a} \bar{x}_t^{(j)} \bar{x}_t^{(j)\top},
\end{equation}
and record $\lambda_{\min}(\hat\Sigma_{a,t})$. 
Repeating this procedure across seeds provides an empirical distribution of the minimum eigenvalue for each process, signature depth, action, and round. 

We perform this experiment for BM, OU, and GBM, with signature depths $N = 1,2,3,4$. We use $K=2$ actions with nonlinear rewards, 
\begin{equation}
f_1(X_{[t-1,t]}) = |\max X_{[t-1,t]}|, \qquad f_2(X_{[t-1,t]}) = |\min X_{[t-1,t]}|.
\end{equation}
Figure \ref{fig:val_assumption_pd} reports the median, as well as the 25\% and 75\% quantiles of $\lambda_{\min}(\hat\Sigma_{a,t})$ over the 100 rounds for action $a=1$. 
Across all three processes, the median minimum eigenvalue remains positive and relatively stable over time at each depth. 
These results provide empirical evidence for the stronger pointwise action-conditional non-degeneracy condition and are therefore consistent with the averaged selection-time condition in Assumption~\ref{assum:unified}(ii).
We also observe that $\lambda_{\min}(\hat\Sigma_{a,t})$ tends to decrease with increasing signature depth $N$, consistent with the higher feature dimension and increased collinearity among higher-order signature terms.

\begin{figure}[h]
  \centering
  \begin{subfigure}[t]{0.325\linewidth}
    \centering
    \includegraphics[width=\linewidth]{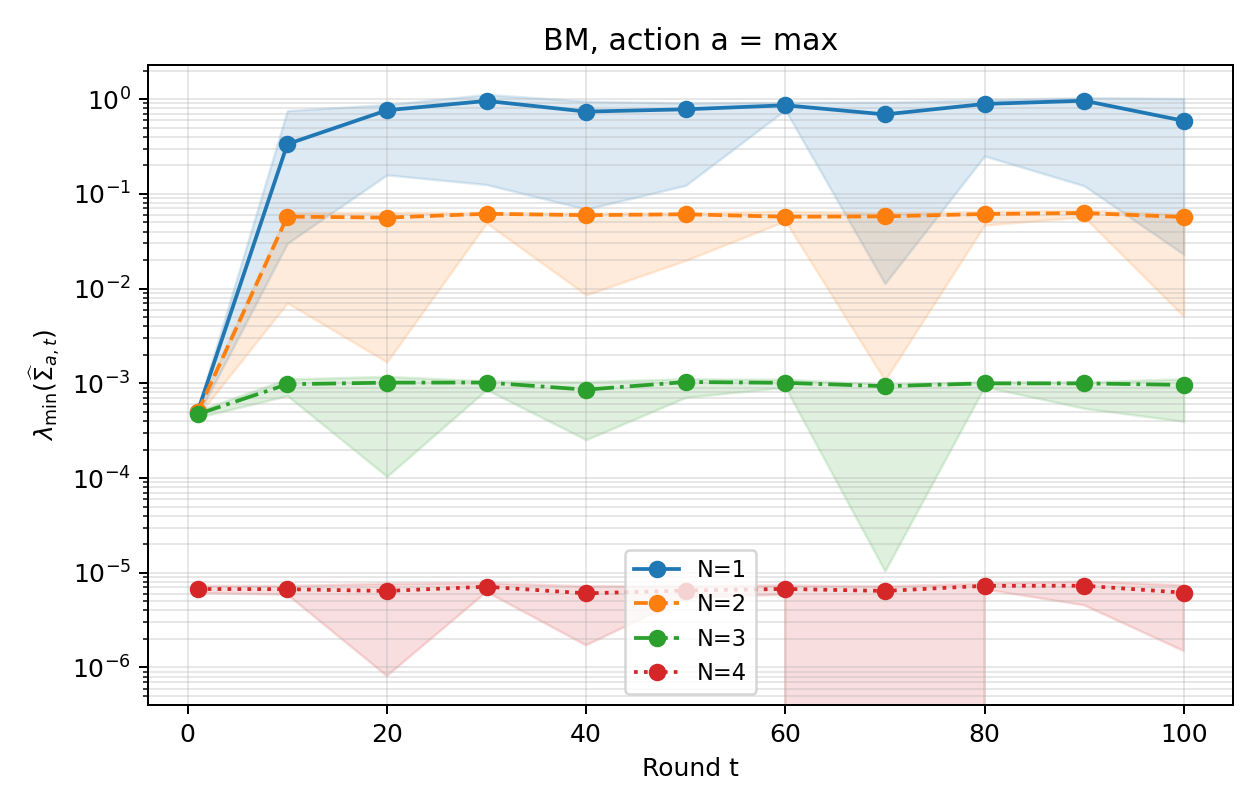}
    \caption{Brownian Motion.}
    \label{fig:pd_bm}
  \end{subfigure}
  \begin{subfigure}[t]{0.325\linewidth}
    \centering
    \includegraphics[width=\linewidth]{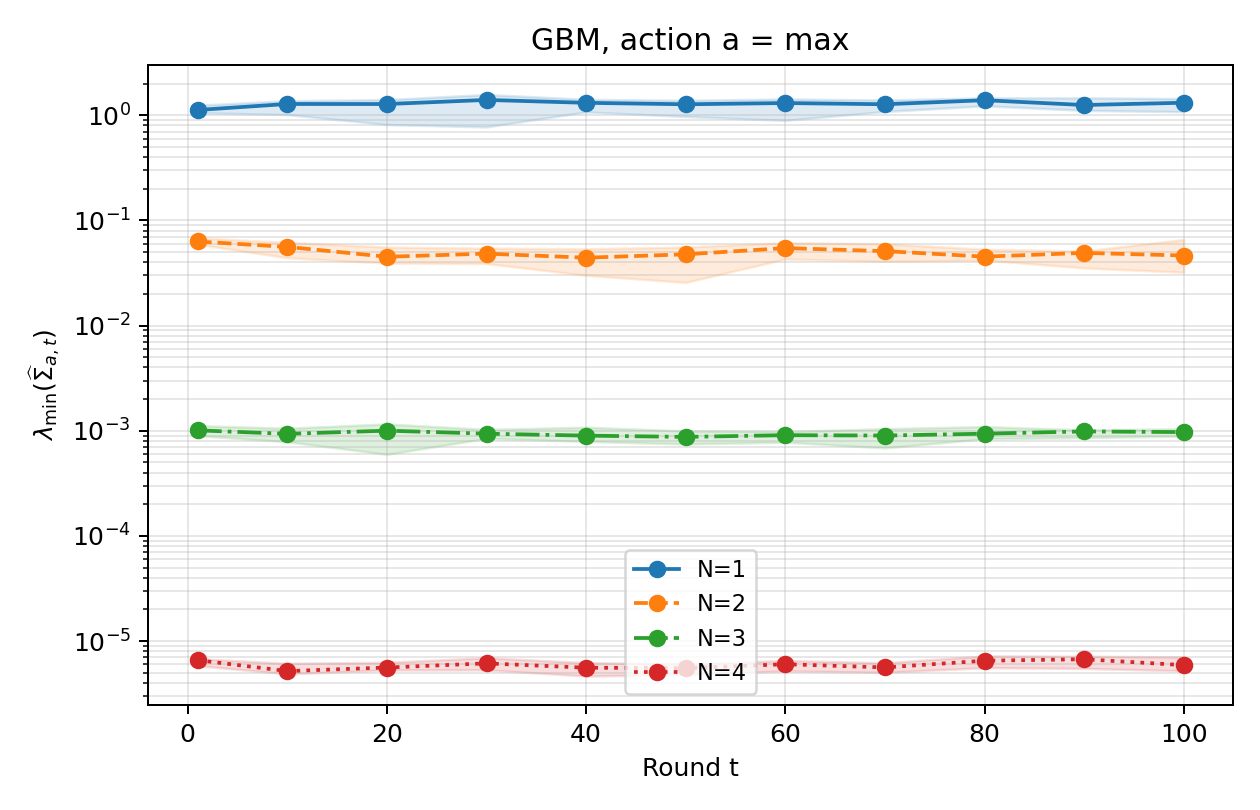}
    \caption{Geometric Brownian Motion\\ ($\alpha=0.1, \nu=1$).}
    \label{fig:pd_gbm}
  \end{subfigure}
  \begin{subfigure}[t]{0.325\linewidth}
    \centering
    \includegraphics[width=\linewidth]{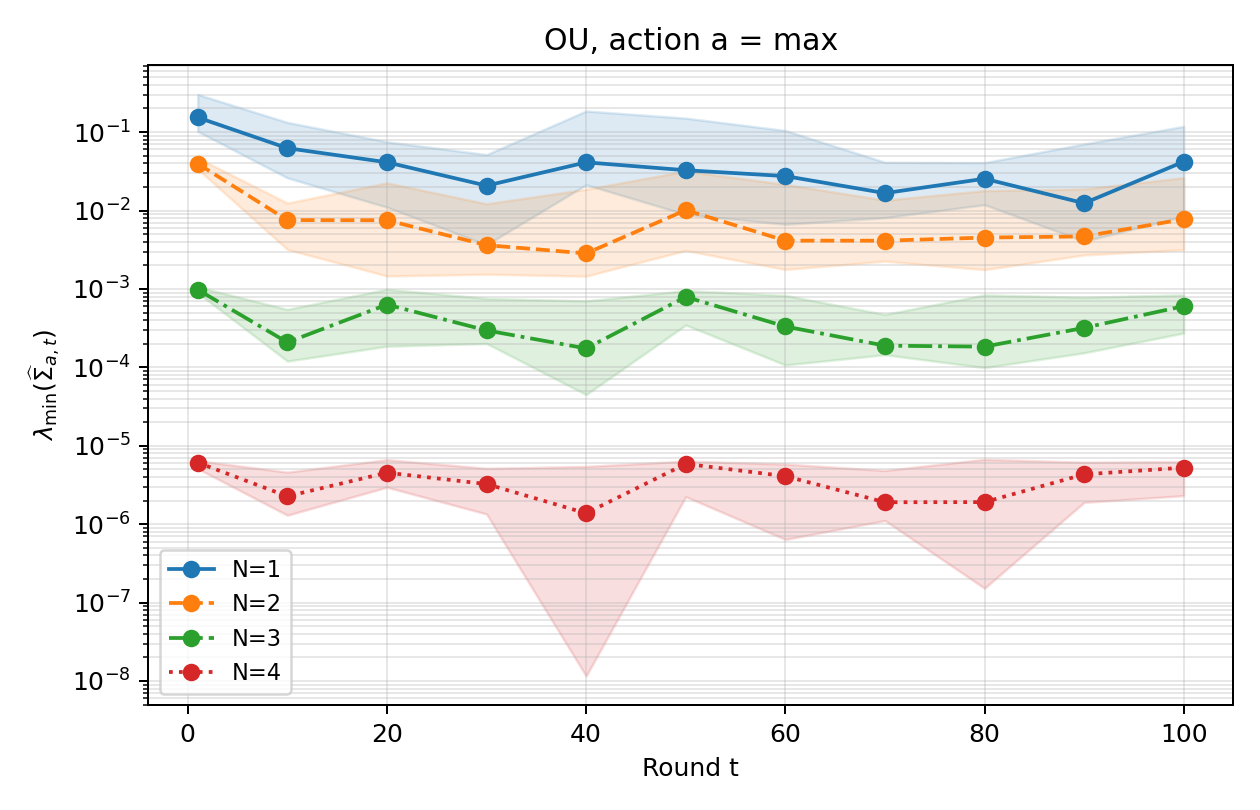}
    \caption{Ornstein-Uhlenbeck ($\theta=1, \mu=0, \sigma=1$).
    }
    \label{fig:pd_ou_s}
  \end{subfigure}

  \caption{Empirical validation of Assumption \ref{assum:unified}(ii): smallest eigenvalue of Gram matrix.}
  \label{fig:val_assumption_pd}
\end{figure}

\subsubsection{Regret performance with linear and nonlinear rewards.}
For this set of experiments, we set the number of actions $K$ to 2. 
We consider the following two different reward functions.
\begin{itemize}
     \item  Linear reward function: for each action $a \in [K]$, we randomly generate $\beta^*_a \sim \mathrm{Unif}(-1,1)$. The linear reward is then defined by
    $ f_a(X_{[t-1,t]}) = \beta^*_a \bar{X}_{[t-1,t]},$
 where $\bar{X}_{[t-1,t]}$ is the mean value of $X_{[t-1,t]}$.
 \item  Nonlinear reward function: for each action $a\in [K]$,
 \(
 f_a(X_{[t-1,t]}) = \begin{cases}
     |\max X_{[t-1,t]}| & \text{ if } a = 1, \\
     |\min X_{[t-1,t]}| & \text{ if } a = 2.
 \end{cases}
 \)
 \end{itemize}

For our proposed algorithm $\texttt{DisSigUCB}$ and the baseline $\texttt{DisLinUCB}$, we set the regularization parameter to $\lambda = 1$ and the exploration coefficient to $\gamma = 1$. 
Because $\texttt{DisLinUCB}$ cannot natively process path-valued contextual inputs, we provide it with the time-averaged context, $\bar{X}_{[t-1,t]}$, as a representative contextual vector. 
Per-round regret is calculated as the difference between the reward of the optimal action and that of the action selected by the algorithm. 
To ensure robust empirical evaluation, we perform $100$ independent trials for each experimental configuration, covering both stationary and nonstationary contextual processes across linear and nonlinear reward models. 
The resulting cumulative regret curves for signature depth $N=3$ are illustrated in Figures~\ref{fig:non_stationary_non_linear} and~\ref{fig:non_stationary_linear}.

In the nonlinear reward setting (see Figure~\ref{fig:non_stationary_non_linear}), $\texttt{DisSigUCB}$ consistently achieves substantially lower cumulative regret than $\texttt{DisLinUCB}$ across all three contextual processes. 
This performance gain supports the empirical benefits of the signature transform's universal nonlinearity. In the linear reward setting (see Figure~\ref{fig:non_stationary_linear}), $\texttt{DisSigUCB}$ continues to outperform $\texttt{DisLinUCB}$ in both the nonstationary and stationary regimes; this suggests that the signature map remains competitive even when the true reward is linear by providing a richer set of predictive features. 
Notably, this advantage persists under the stationary OU process, which exhibits short-range temporal dependence. 
This indicates that $\texttt{DisSigUCB}$ effectively leverages the underlying temporal structure within the observed context windows to enhance its predictive accuracy.

\begin{figure}[h]
   \centering
   \begin{subfigure}[t]{0.33\linewidth}
     \centering
     \includegraphics[width=\linewidth]{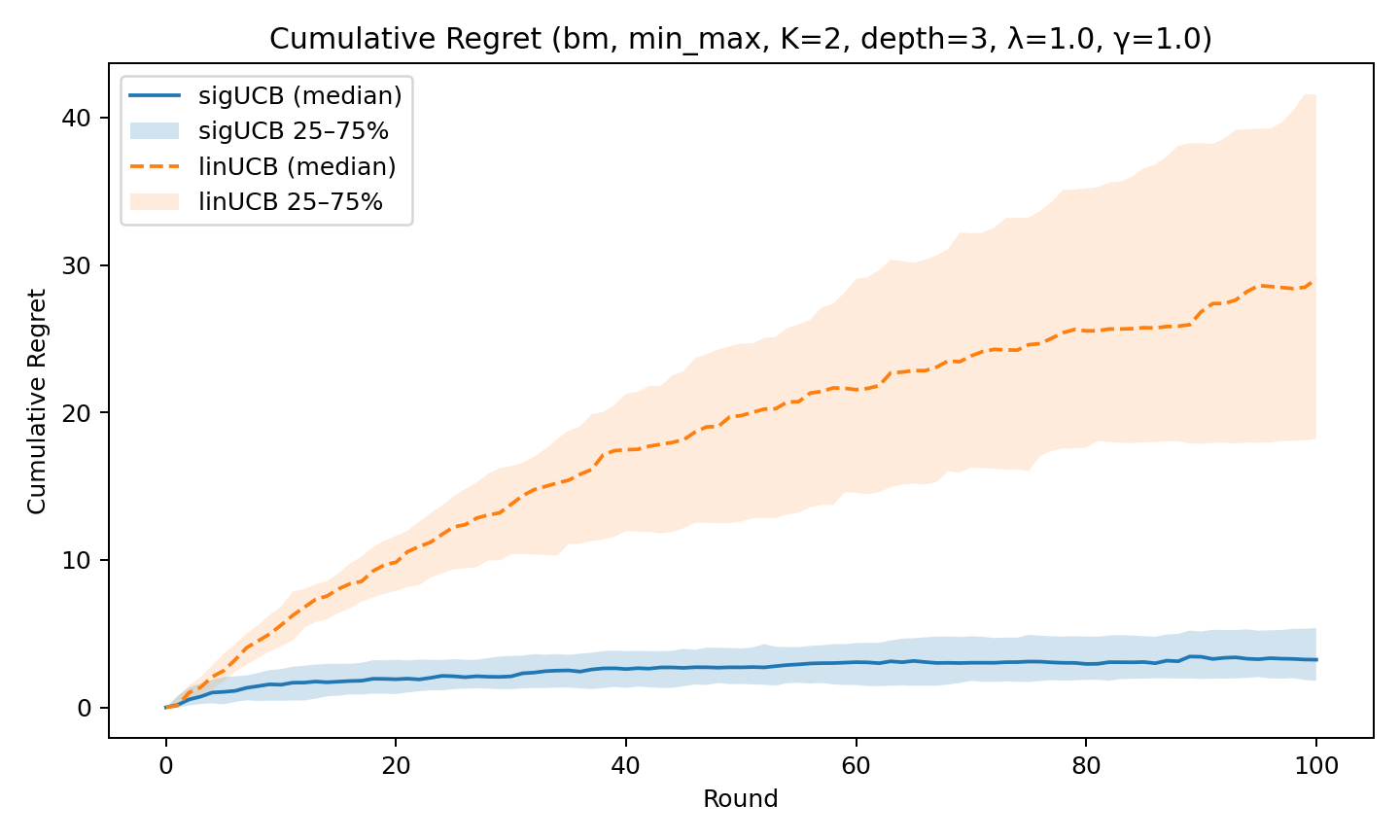}
     \caption{Brownian Motion.}
     \label{fig:bm_non_stationary_non_linear}
   \end{subfigure}\hfill
   \begin{subfigure}[t]{0.33\linewidth}
     \centering
     \includegraphics[width=\linewidth]{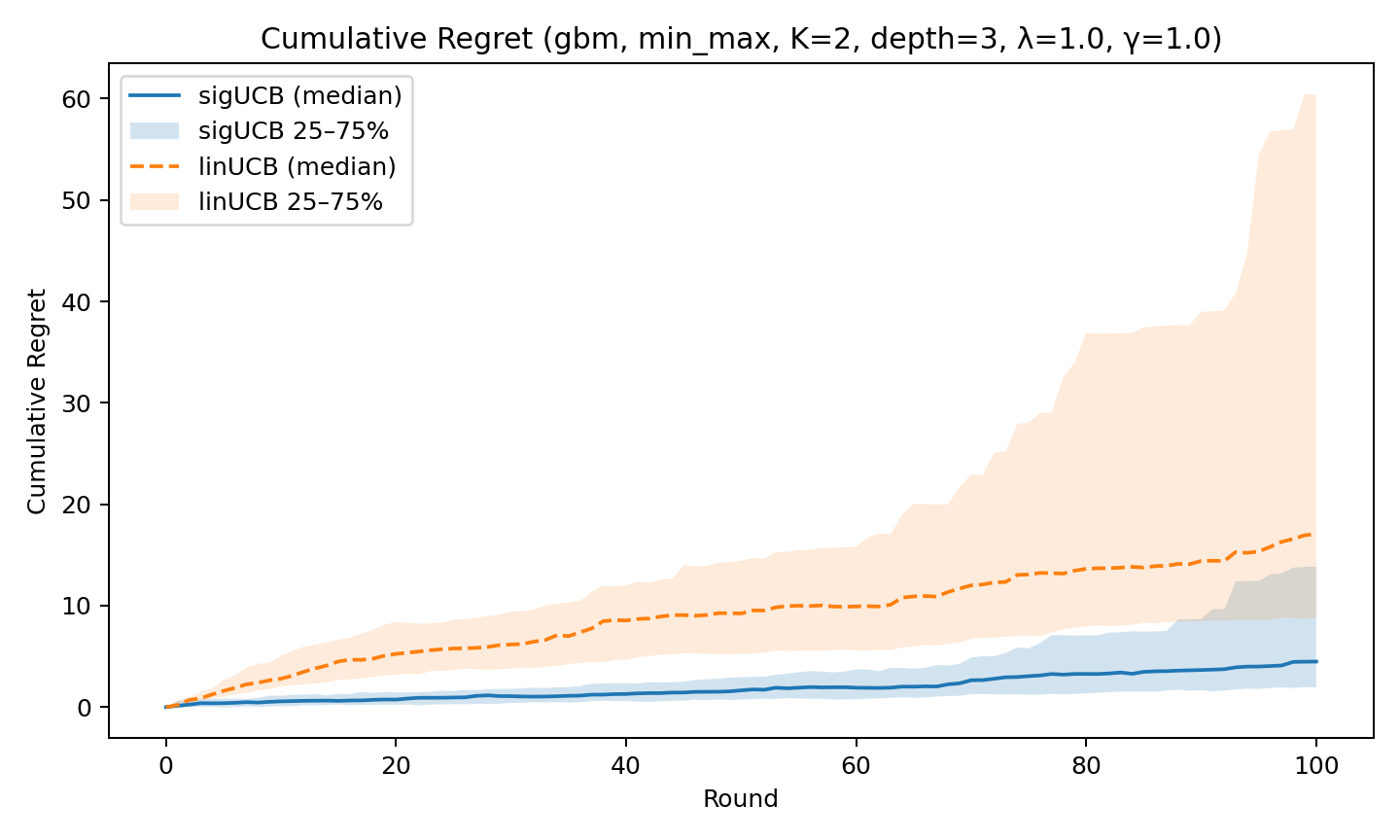}
     \caption{Geometric Brownian Motion \\ ($\alpha=0.1$, $\nu = 1$).}
     \label{fig:gbm_non_stationary_non_linear}
   \end{subfigure}
   \begin{subfigure}[t]{0.33\linewidth}
     \centering
      \includegraphics[width=\linewidth]{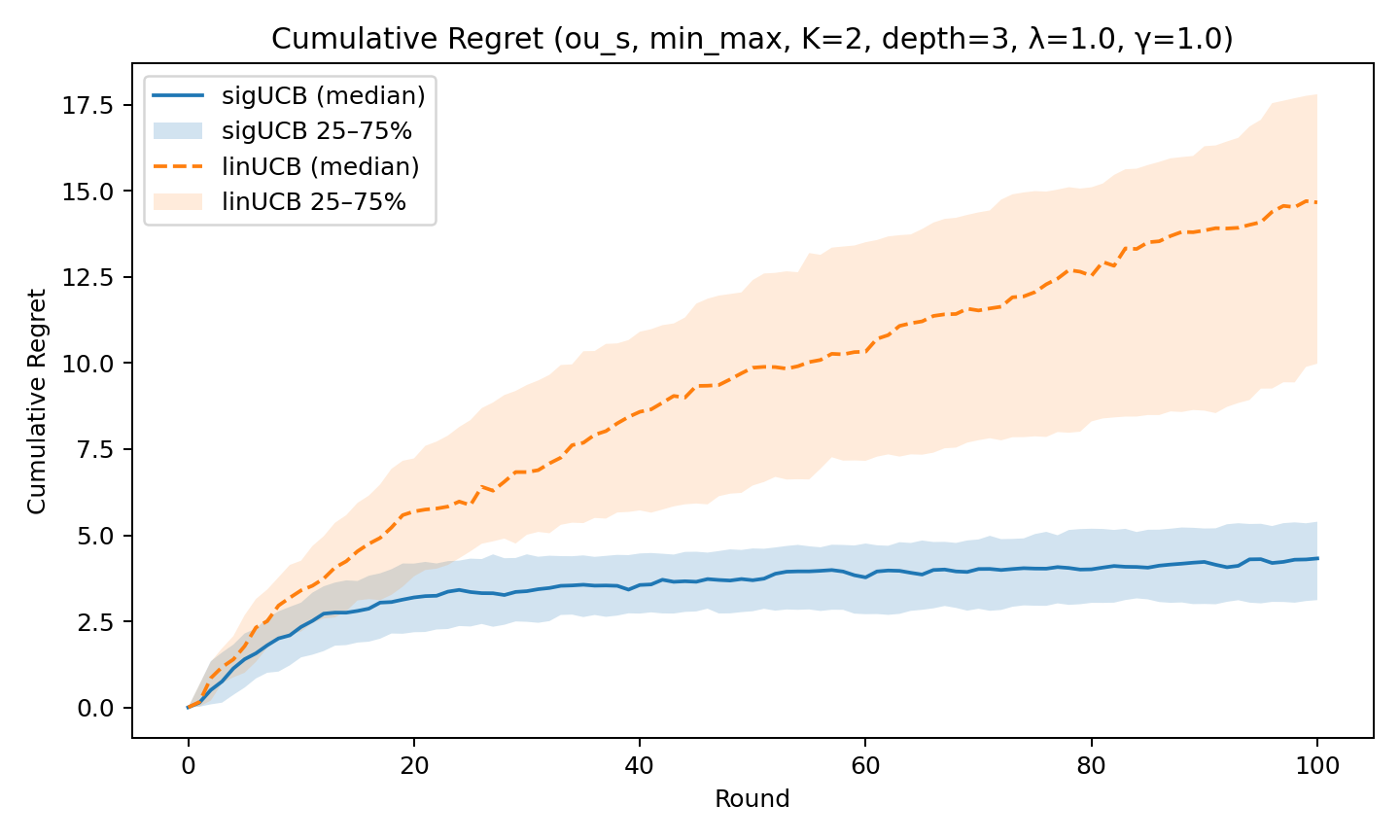}
     \caption{Stationary Ornstein-Uhlenbeck \\
     ($\mu=0, \theta=1,\sigma=1$).}
   \end{subfigure}
   \caption{Cumulative regret with nonlinear reward; (a) and (b) are nonstationary Brownian motion and geometric Brownian motion, and (c) is the stationary OU process.}
   \label{fig:non_stationary_non_linear}
\end{figure}
\begin{figure}[h]
   \centering
   \begin{subfigure}[t]{0.33\linewidth}
     \centering
     \includegraphics[width=\linewidth]{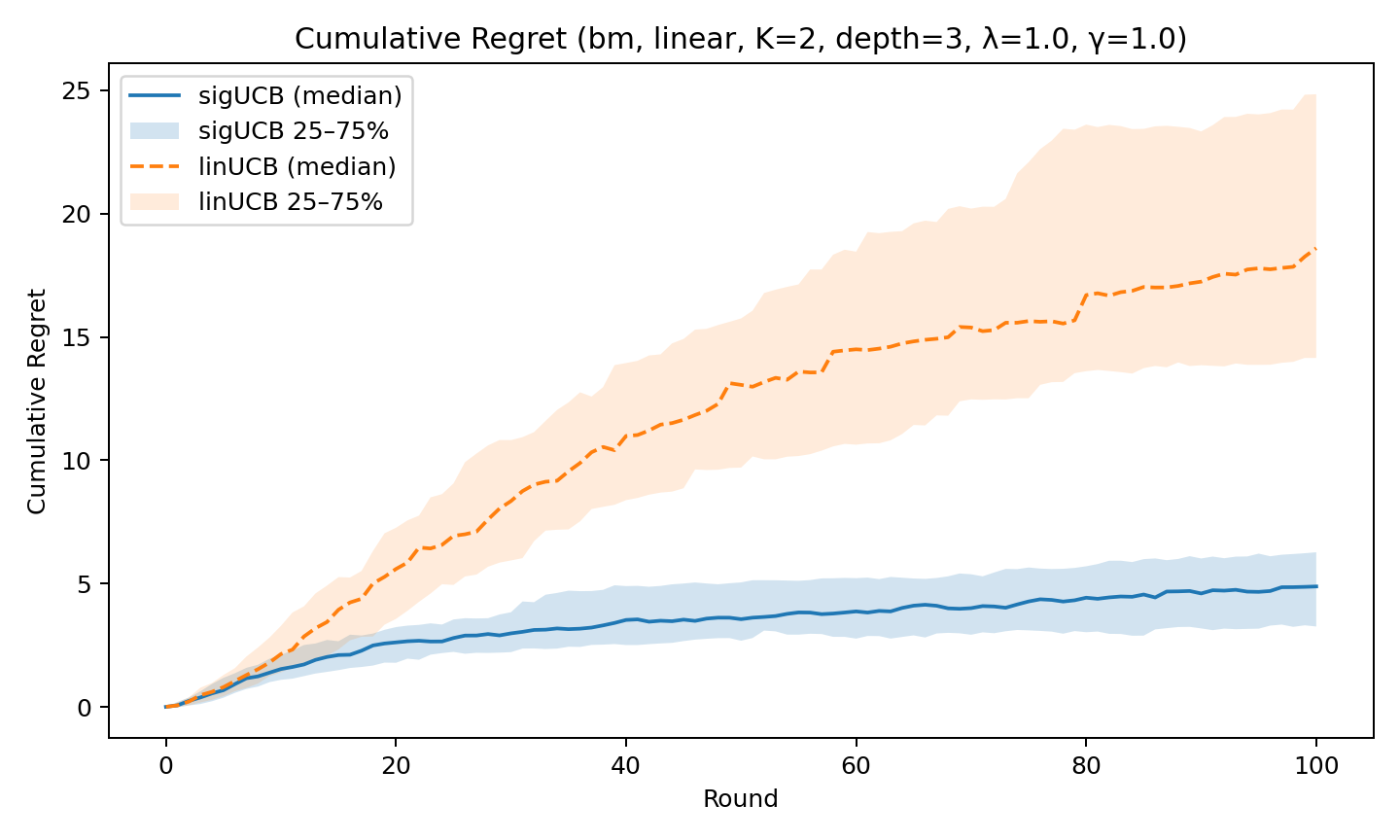}
     \caption{Brownian Motion.}
     \label{fig:bm_non_stationary_linear}
   \end{subfigure}\hfill
   \begin{subfigure}[t]{0.33\linewidth}
     \centering
     \includegraphics[width=\linewidth]{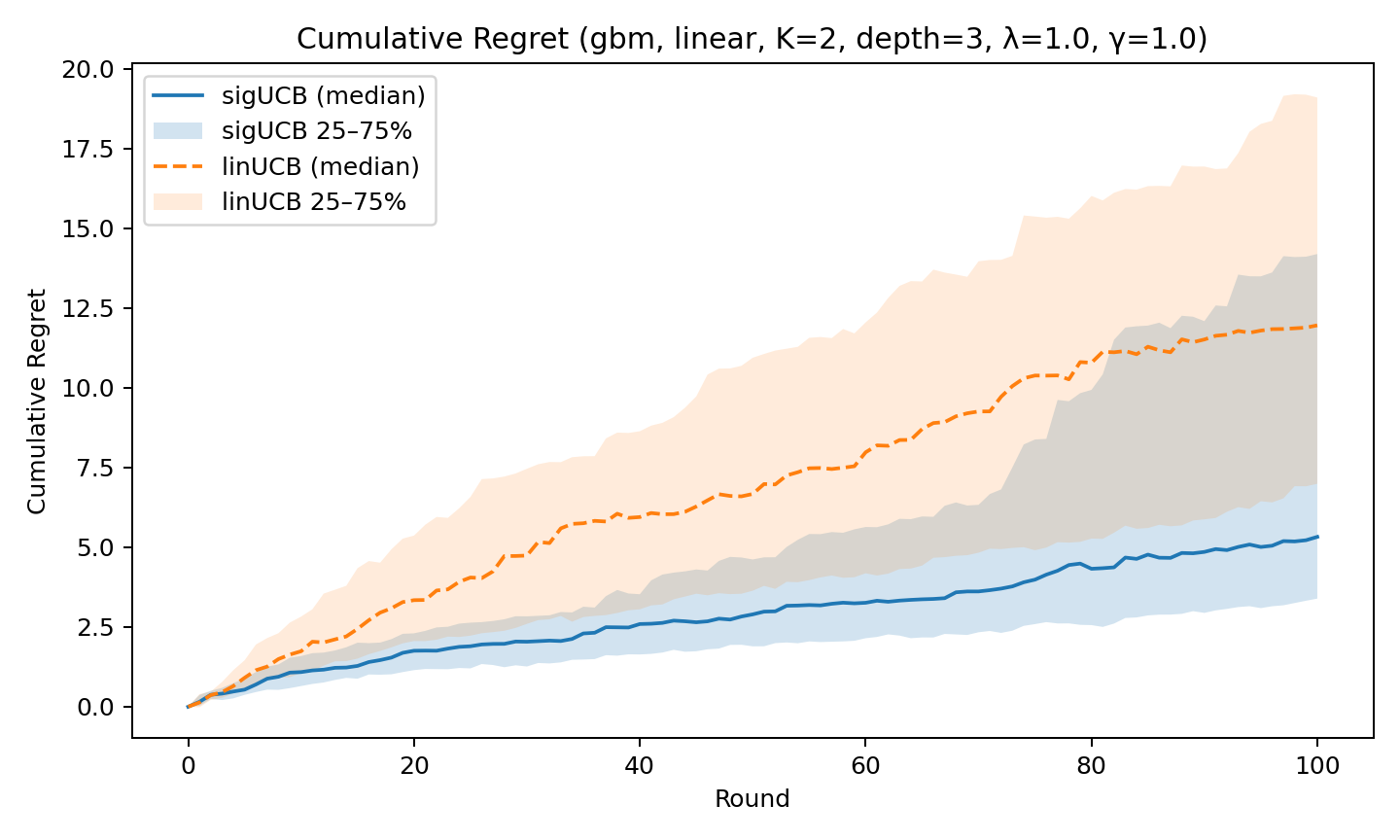}
     \caption{Geometric Brownian Motion \\ ($\alpha=0.1$, $\nu = 1$).}
     \label{fig:gbm_non_stationary_linear}
   \end{subfigure}
      \begin{subfigure}[t]{0.33\linewidth}
     \centering
     \includegraphics[width=\linewidth]{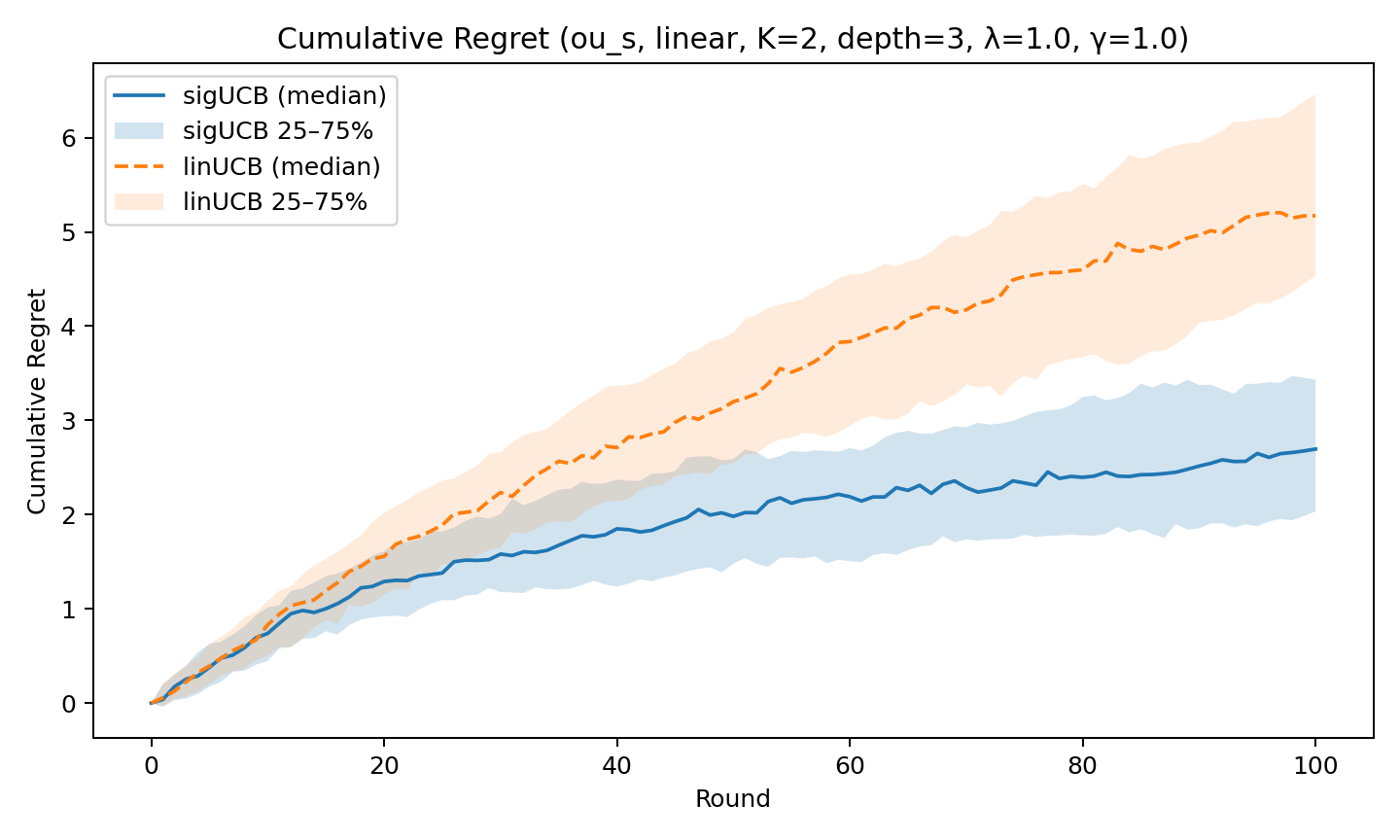}
     \caption{Stationary Ornstein-Uhlenbeck\\   ($\mu=0, \theta=1,\sigma=1$).}
     \label{fig:ou_stationary_linear}
   \end{subfigure}
   \caption{Cumulative regret with linear reward; (a) and (b) are nonstationary Brownian motion and geometric Brownian motion, and (c) is the stationary OU process.}  \label{fig:non_stationary_linear}
\end{figure}

\subsubsection{Empirical Validation of Statistical Inference.} \label{ssec:verify_stat_inf}
We next empirically evaluate the asymptotic normality and confidence intervals established in Corollary~\ref{cor:cov_consistency}.
Following \cite{CLS:21}, we assess whether the standardized estimation errors approach a standard normal distribution, the estimated standard errors agree with the empirical standard deviations of the coefficient estimates, and the confidence intervals achieve their nominal coverage probabilities.
We focus on the stationary OU context process defined in Section~\ref{ssec:syn_exp} and set $K=2$, signature depth $N=3$, exploration probability $\zeta=0.2$, and Gaussian reward noise $\eta_t\sim \cN(0,0.01)$, applied identically to both actions.

To ensure that the reward model is exactly specified ($\vep_N = 0$), we consider the path-dependent reward functions
\begin{equation}
    f_1(X_{[t-1,t]}) = \int_{t-1}^t X_s ~\d s, \qquad f_2(X_{[t-1,t]}) = - \int_{t-1}^t X_s~ \d s.
\end{equation}
For the time-augmented path over a unit-length window, we have
\begin{equation}
    \int_{t-1}^t X_s \d s = X_t - \Sig^{(0,1)}(\hat{X}_{[t-1,t]}),
\end{equation}
so  the reward is exactly linear in the augmented signature representation and $\vep_N = 0$. 
In particular, the true coefficients of $\Sig^{(0,1)}$ and $X_t$ for action 1 are -1 and 1, respectively, and vice versa for action 2. 
Since the two actions are symmetric, we report results for action 1.

We consider horizons $T \in \{500, 1000, 2000, 5000, 10000, 20000, 50000\}$, with 500 Monte Carlo replications for $T \leq 5000$ and 200 replications for larger horizons to limit computational cost. 
In each replication, we compute the ridge estimator $\hat\beta_{a,T}$ and estimate the residual variance and asymptotic covariance using \eqref{eq:cov_estimator}. 
For each coefficient $j$, we then calculate the standardized estimation error
\begin{equation}
Z_{a,T,j} = \frac{\hat\beta_{a,T,j} - \beta^*_{a,j}}{\sqrt{[\hat\Sigma_{a,T}]_{j,j}/T}},
\end{equation}
where the denominator is the estimated standard error of $\hat\beta_{a,T,j}$.

Figure~\ref{fig:qq_ou} compares the empirical quantiles of $Z_{a,T,j}$ with those of $\cN(0,1)$ at $T = 500, 5000, 50000$. 
At the smallest horizon, both coordinates exhibit noticeable finite-sample location shifts. 
As $T$ increases, the studentized distributions move toward the standard normal reference line. 
The convergence is particularly clear for $X_T$, for which the studentized statistic has empirical mean $-0.08$ and standard deviation $1.04$ at $T=50000$. 
The $\Sig^{(0,1)}$ coordinate converges more slowly, with empirical mean $0.90$ and standard deviation $0.95$ at the same horizon of $T=50000$. 

\begin{figure}[h]
    \centering
    \includegraphics[width=0.8\linewidth]{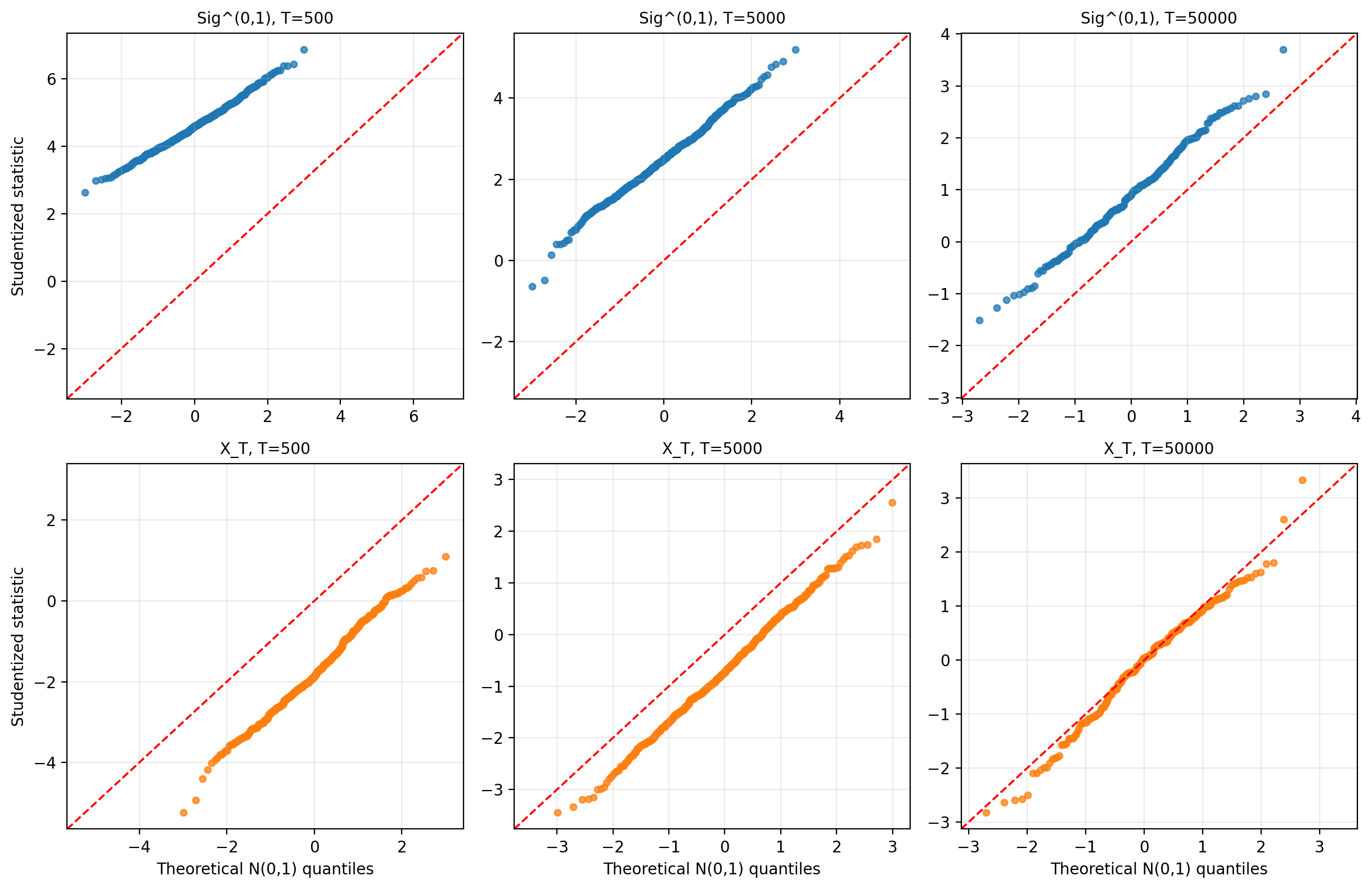}
    \caption{Normal Q–Q plots of the studentized statistics for $\Sig^{(0,1)}$ and $X_T$ at $T=500,5000,50000$.}
    \label{fig:qq_ou}
\end{figure}

We assess standard-error calibration by comparing the estimated standard errors $\sqrt{[\hat\Sigma_{a,T}]_{j,j}/T}$, averaged across Monte Carlo replications, with the empirical standard deviation of the coefficient estimates $\hat\beta_{a,T,j}$ across replications. 
We report the ratio of these two quantities, where a value close to 1 indicates that the estimated standard errors accurately reflect the variability of the coefficient estimates. 
As shown in Figure~\ref{fig:se_ou}, this ratio remains close to 1 for the coefficient of $X_T$ across all horizons. 
For the coefficient associated with $\Sig^{(0,1)}$, the estimated standard errors initially overestimate the empirical variability, but the ratio decreases from 1.56 at $T=500$ to approximately 1.05 at $T=50000$, indicating improved calibration as the horizon increases.

\begin{figure}[h]
  \centering
  \begin{subfigure}[t]{0.4\linewidth}
    \centering
    \includegraphics[width=\linewidth]{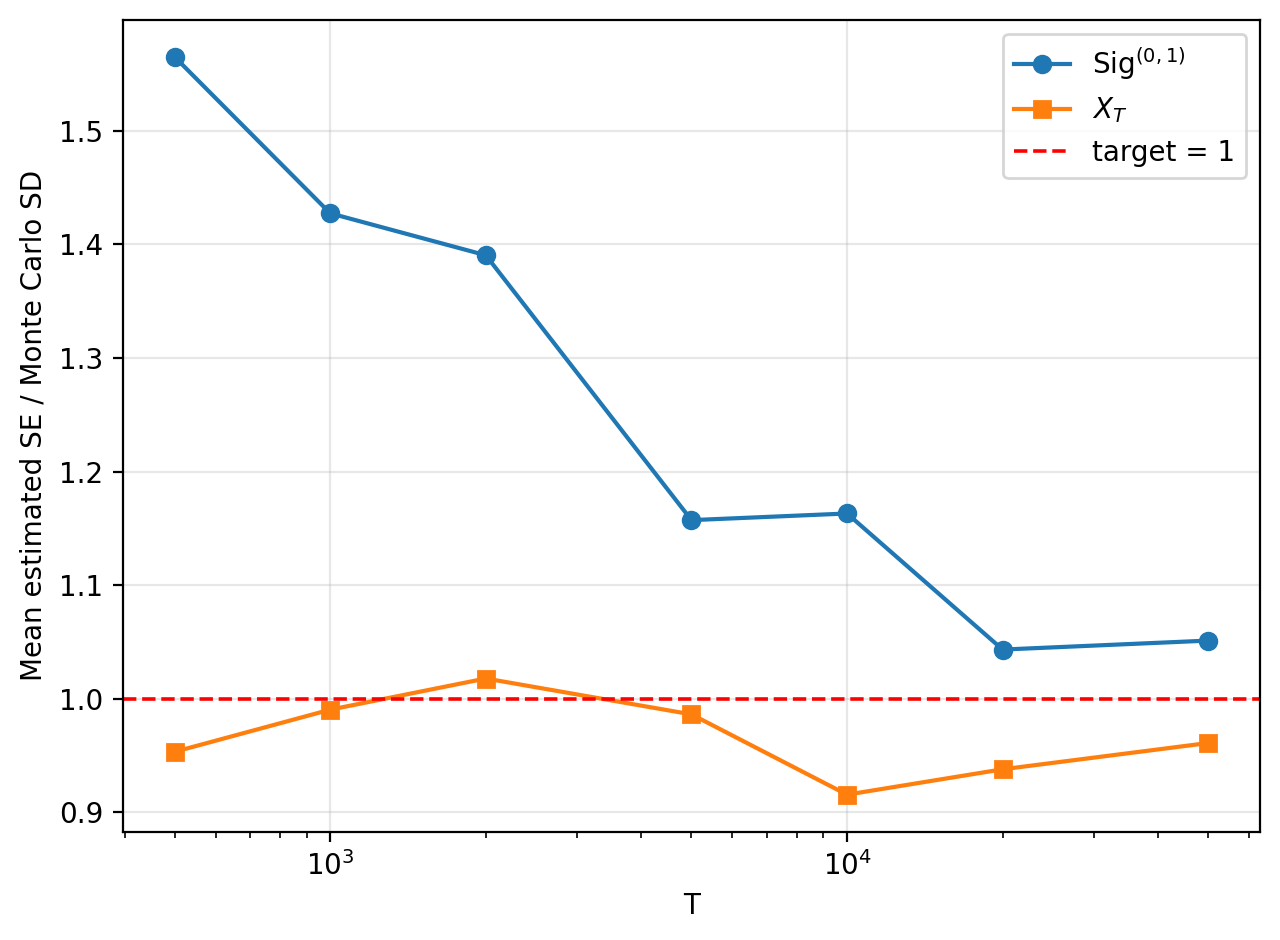}
    \caption{Standard-error calibration.}
    \label{fig:se_ou}
  \end{subfigure}
  \hspace{1mm}
  \begin{subfigure}[t]{0.4\linewidth}
    \centering
    \includegraphics[width=\linewidth]{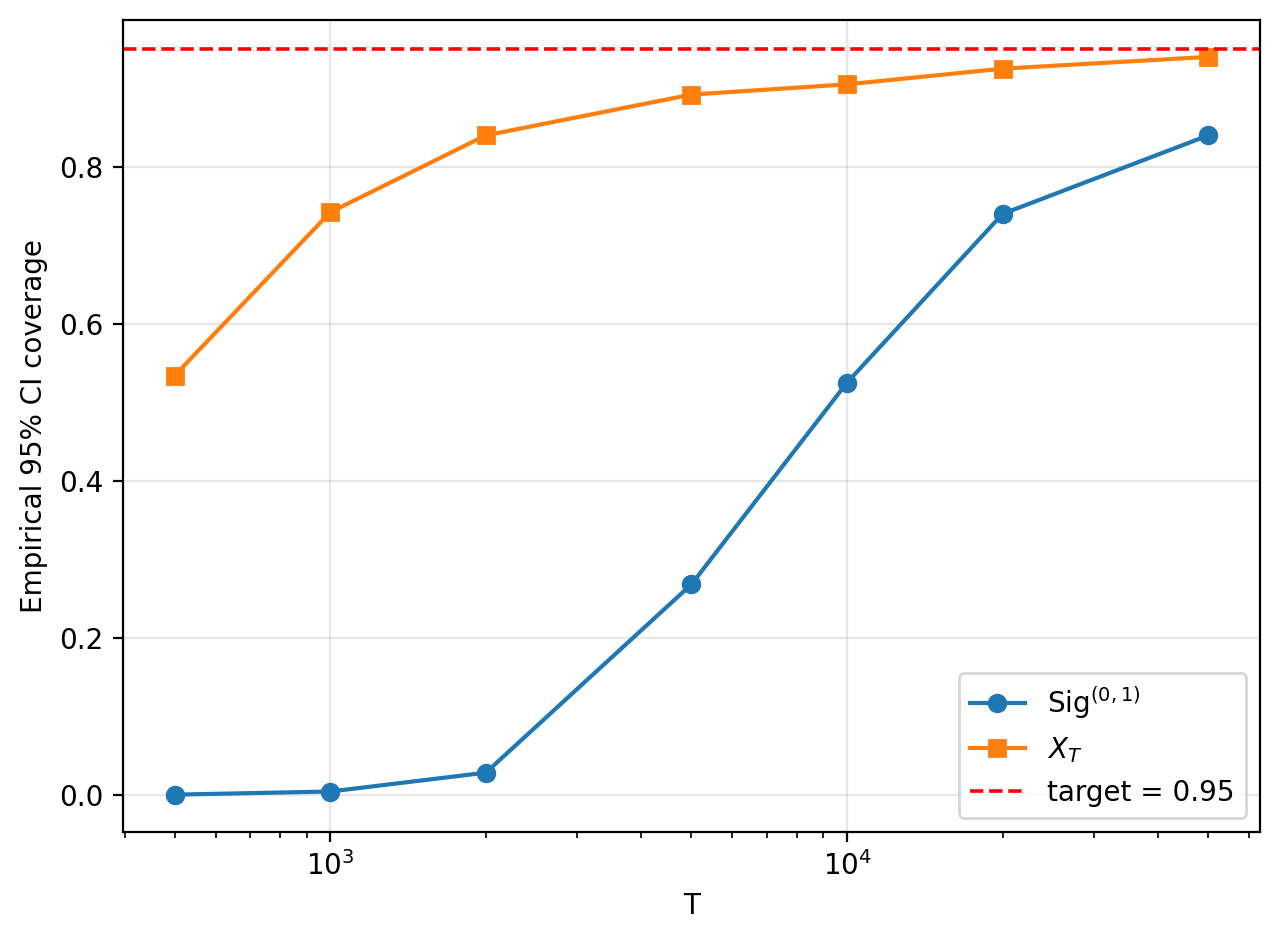}
    \caption{Empirical 95\% confidence-interval coverage.
    }
    \label{fig:ci_ou}
  \end{subfigure}
  \caption{Standard-error calibration and confidence-interval coverage under \(\zeta\)-\texttt{DisSigUCB}.}
  \label{fig:se_ci_ou}
\end{figure}

Finally, Figure~\ref{fig:ci_ou} reports the empirical coverage of the nominal 95\% confidence intervals from Corollary~\ref{cor:cov_consistency}, measured as the proportion of Monte Carlo replications in which the intervals contain the true coefficient values.
For the coefficient of $X_T$, the empirical coverage increases from 0.53 at $T=500$ to 0.94 at $T=50000$, approaching the nominal level of 0.95. 
In contrast, coverage for the coefficient of $\Sig^{(0,1)}$ increases from nearly zero at the smallest horizons to 0.84 at $T=50000$, reflecting its slower finite-sample convergence. 
Overall, these results are consistent with the asymptotic inference theory in Section~\ref{sec:asym_norm} and illustrate that convergence rates can differ substantially across coefficient coordinates.

\subsection{Learning to Monitor Temperature Sensor Networks}
We consider a spatiotemporal sensor monitoring task similar to that in \cite{SZBPWKCB:11, CGKCB:17}, using the \href{https://db.csail.mit.edu/labdata/labdata.html}{Intel Berkeley Research Lab dataset} \cite{ILD:04}. The dataset contains temperature measurements from $K=54$ sensors deployed throughout the lab approximately every 30 seconds. We assume that there are four anchor sensors always available to the learner. In each hour (round) $t$, the learner observes the anchors' recent temperature measurements and selects one of the \(K\) sensors to approximate either the network-wide average or maximum temperature during the next hour.

{\bf Context.} Let $X_{t,k}$ denote the temperature recorded by sensor $k\in[K]$ at time $t$. We aggregate the data by averaging over 5-minute intervals, giving 12 observations per hour. Let $\cA_0 \subset [K]$ be the set of anchor sensors. We define the contextual process by
$
Y_t := (X_{t,k})_{k\in \mathcal A_0} \in \mathbb{R}^{|\cA_0|}.
$ 
Then the learner observes $Y_{[t-1,t)}$ before they make a decision for hour $t$. 

{\bf Action and reward.}
At each hour $t$, the learner chooses a sensor $A_t \in [K]$. We consider two types of rewards. The first one is the negative absolute difference between the average temperature of the chosen sensor $A_t$ and the true average temperature,
\(
r_t = -|\bar{X}_{t, A_t} - \sum_{k \in [K]} \frac{\bar{X}_{t,k}}{K}|,
\)
where $\bar{X}_{t,k}$ is the average of $X_{s,k}$ over the time interval $[t, t+1)$. The second reward  is the negative absolute difference between the maximum temperature of the chosen sensor $A_t$ and the true max temperature,
\(
r_t = -|\max_{s \in [t,t+1)} X_{s, A_t} - \max_{s \in [t, t+1), k \in [K]} X_{s,k}|.
\)
Thus, for either task, in the notation of Section~\ref{sec:problem}, $f_a(y) := \mathbb{E}[r_t | Y_{[t-1,t)} = y, A_t = a].$

{\bf Results.} We use the first 24 hours from February 28, 2004, at 1 AM to tune parameters for all algorithms and then evaluate from February 29, 2004, at 1 AM to March 8, 2004, at 1 AM.  We run 10 trials, randomly selecting \(|\mathcal A_0|=4 \) anchor sensors in each trial. Restricting the signature-depth grid to \(N\leq 3\) did not noticeably affect \texttt{DisSigUCB}'s performance. Figures~\ref{fig:temp_sensor_avg_combined} and~\ref{fig:temp_sensor_max_combined} show the median cumulative regret and the true network statistic versus the median statistic of the chosen sensor for the average and maximum temperature tasks. Both statistics exhibit strong daily fluctuations and an increasing trend. The average-temperature task is easier, as reflected by the smaller errors. In both tasks, \texttt{DisSigUCB} outperforms \texttt{KernelUCB} and \texttt{LinUCB-rff}. Our algorithm is comparable to \texttt{DisLinUCB}, \texttt{LinUCB-poly}, and \texttt{LinUCB-stats} on the (easier) average-temperature task. On the (harder) maximum-temperature task, \texttt{DisSigUCB} and \texttt{LinUCB-stats} perform best, likely because \texttt{LinUCB-stats} explicitly includes the window maximum. This also suggests that \texttt{DisSigUCB} successfully learns the path-dependent maximum functional.

\begin{figure}[h]
    \centering
    \begin{subfigure}[t]{0.45\linewidth}
        \centering
        \includegraphics[width=\linewidth]{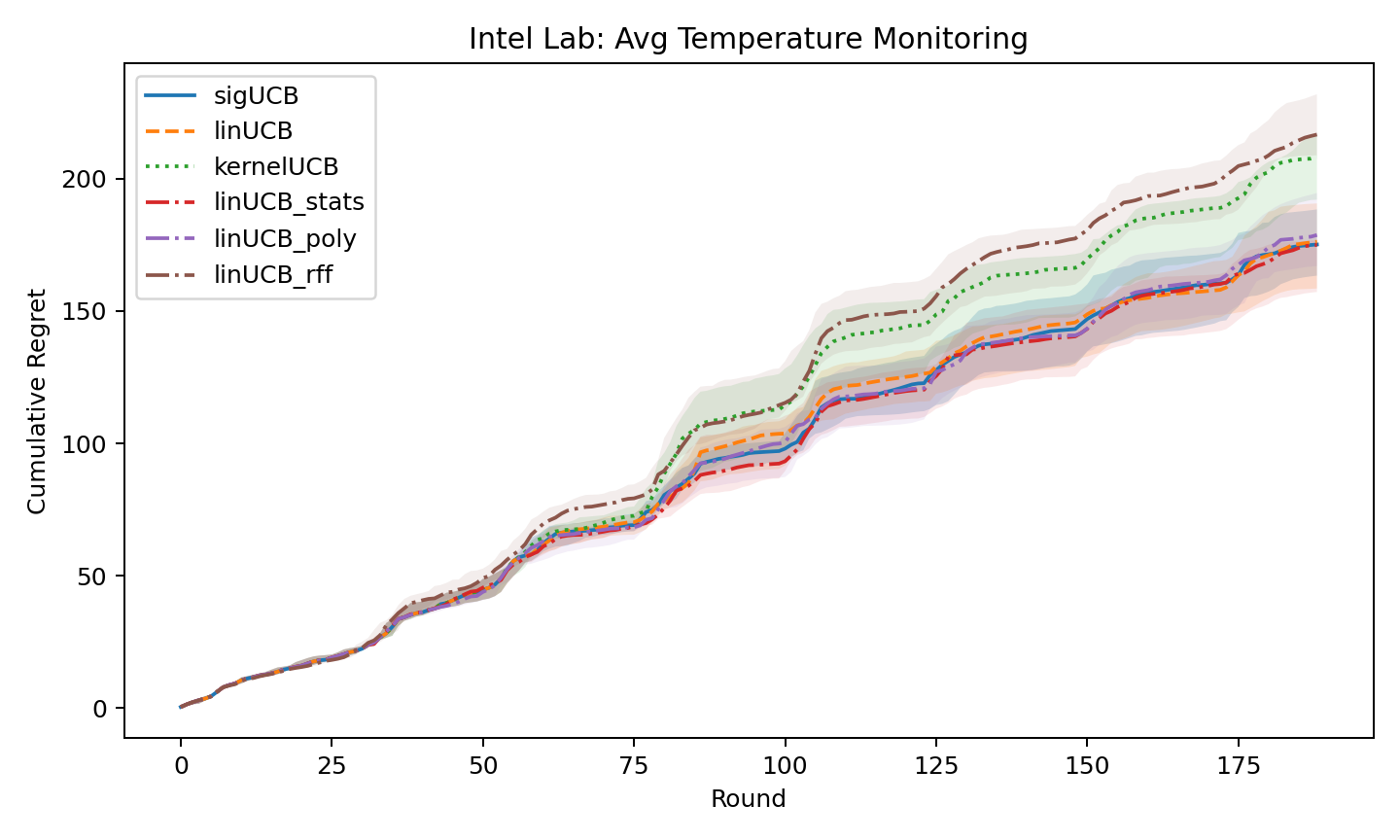}
        \caption{Cumulative regret}
        \label{fig:temp_sensor_regret_avg}
    \end{subfigure}
    \hspace{0.03\textwidth}
    \begin{subfigure}[t]{0.45\linewidth}
        \centering
        \includegraphics[width=\linewidth]{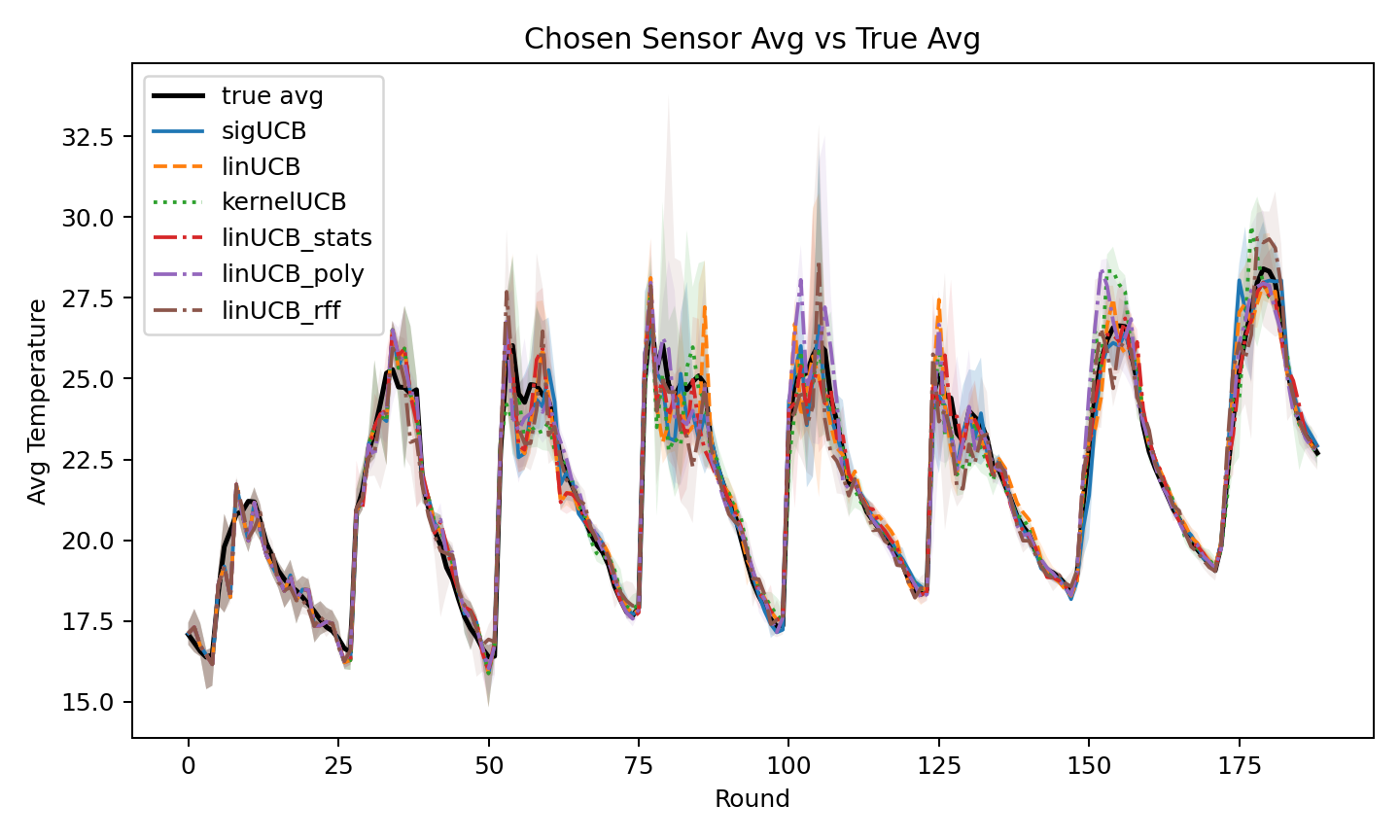}
        \caption{Average temperature of chosen sensors}
        \label{fig:temp_sensor_avg}
    \end{subfigure} 
    \caption{Average-temperature task. Lines show medians; shaded bands show 25th--75th percentiles.}
\label{fig:temp_sensor_avg_combined}
\end{figure}

\begin{figure}
    \centering
    \begin{subfigure}[t]{0.45\linewidth}
        \centering
        \includegraphics[width=\linewidth]{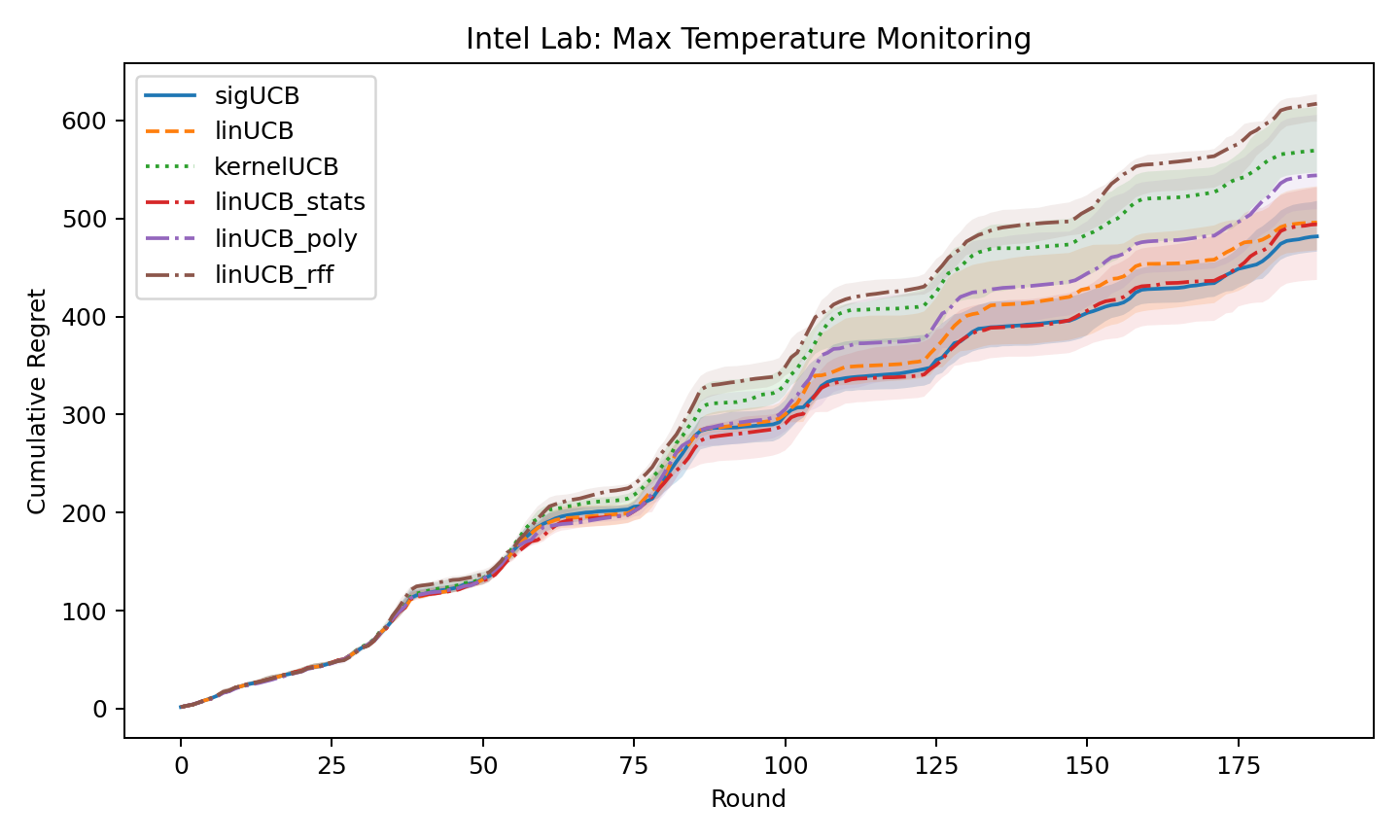}
        \caption{Cumulative regret}
        \label{fig:temp_sensor_regret_max}
    \end{subfigure}
    \begin{subfigure}[t]{0.45\linewidth}
        \centering
        \includegraphics[width=\linewidth]{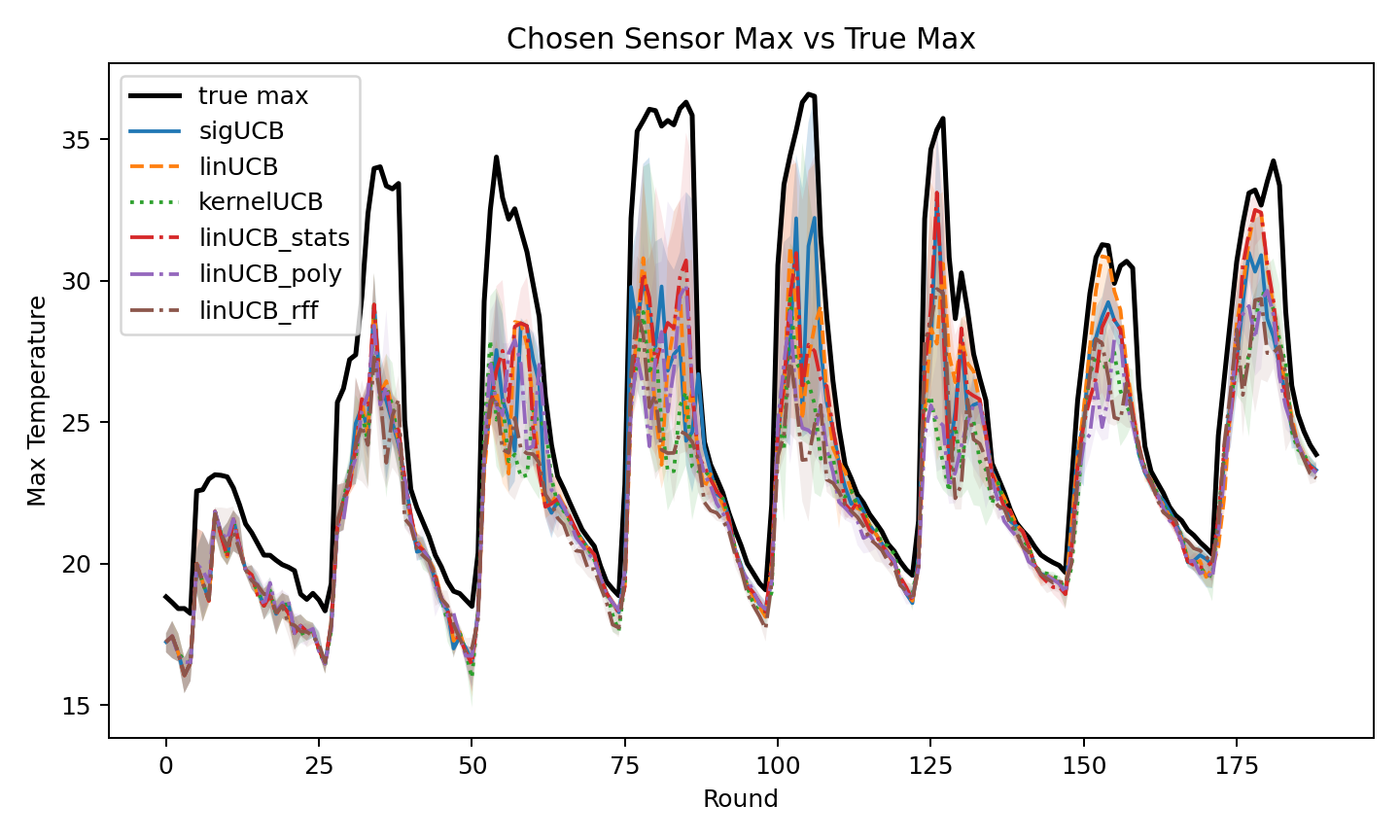}
        \caption{Max temperature of chosen sensors.}
        \label{fig:temp_sensor_max}
    \end{subfigure}
    \caption{Maximum temperature task. Lines show medians; shaded bands show 25th--75th percentiles.}
    \label{fig:temp_sensor_max_combined}
\end{figure}

\subsection{Sleep Stage Classification via Cassette-Tape Recorder}
We consider an online sleep-stage recognition task based on a stream of polysomnographic recordings, including electroencephalogram (EEG) and electrooculogram (EOG), which measure brain electrical activity and eye movements. The goal is to periodically identify the subject’s current sleep state from the most recent measurements. We use the \href{https://physionet.org/content/sleep-edfx/1.0.0/sleep-cassette/#files-panel}{Sleep-EDF Database Expanded} (publicly available under ODC-BY 1.0) \cite{KZTKO:00}, which contains 150 overnight recordings with EEG and EOG channels sampled at 100 Hz. In our setup, the learner makes a classification decision every 30 seconds (one round). 

{\bf Context.}
We down-sample the data by averaging from 100 Hz to 25 Hz, reducing window size while retaining fine temporal resolution. We use EEG Fpz-Cz, EEG Pz-Oz, and EOG horizontal as the contextual process
$
X_t := [\text{EEG}_{\text{Fpz-Cz},t}, \text{EEG}_{\text{Pz-Oz},t}, \text{EOG}_{\text{horizontal},t}] \in \mathbb{R}^3.
$
Since each 30-second window is sampled at 25 Hz, the learner observes
\(X_{[t-1,t)}\in\mathbb R^{750\times 3}\) to make their classification decision.

{\bf Action and reward.} 
The action set consists of six sleep states 
$
\cA=\{\text{Wake}, \text{N1}, \text{N2}, \text{N3}, \text{REM}, \text{Movement}\},
$
where N1, N2, and N3 are non-REM stages 1, 2, and 3, respectively. Since each window may contain multiple sleep stages, we use a soft-label reward. Let $\ell(\tau)\in\mathcal A$ be the ground-truth sleep stage at time $\tau$ seconds. For round $t$ and $a\in\mathcal A$, define
$
P_t(a)
:=
\frac{1}{750}\sum_{j=1}^{750}
\mathbf 1\!\left\{
\ell\!\left(30(t-1)+\frac{j-1}{25}\right)=a
\right\},
$
as the fraction of sampled time points in the $t$-th window labeled as
stage $a$.
So, after choosing \(A_t \in \cA\), the learner receives reward
\(
r_t=P_t(A_t).
\)
Thus, the optimal action is the label occupying the largest fraction of the current window.

{\bf Results.} 
For each sleep cassette, we use the first 100 minutes (200 rounds) after at least one minute of sleep to tune hyperparameters and then evaluate on the following 500 minutes (1000 rounds). Increasing the signature-depth grid from \(N\leq 2\) to \(N\leq 3\) did not significantly improve performance. Figure \ref{fig:sleep_edf_combined} shows the median cumulative regret and the average weighted F1-score over a rolling window of 200 rounds across all sleep cassettes.
We observe that all four \texttt{LinUCB} baselines perform the worst, with higher cumulative regret, greater variability, and lower weighted F1-scores. \texttt{KernelUCB} and \texttt{DisSigUCB} have similar median performance but \texttt{KernelUCB} is more variable. Within the first 100 rounds, \texttt{DisSigUCB} obtains the highest median weighted F1-score. 
Table \ref{tab:sleep_edf} reports classification metrics for the final 200 rounds, where \texttt{DisSigUCB} achieves higher mean accuracy and weighted F1-score. 
Among the richer feature-map baselines, \texttt{LinUCB-stats} performs notably better than \texttt{DisLinUCB}, while \texttt{LinUCB-poly} and \texttt{LinUCB-rff} show more limited improvements.


\begin{figure}
    \centering
    \begin{subfigure}[t]{0.45\linewidth}
        \centering
        \includegraphics[width=\linewidth]{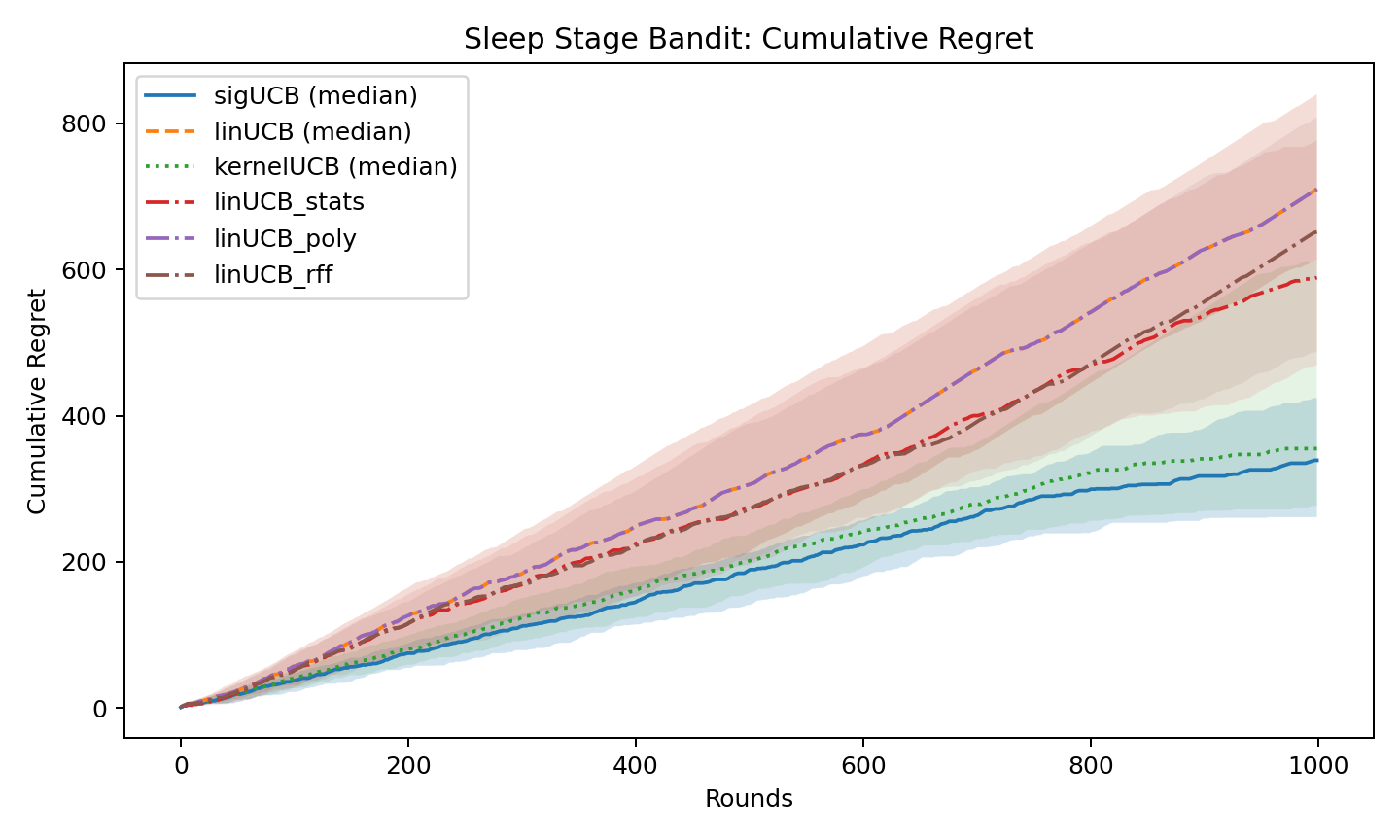}
        \caption{Cumulative regret}
        \label{fig:sleep_edf_regret}
    \end{subfigure}
    \begin{subfigure}[t]{0.45\linewidth}
        \centering
        \includegraphics[width=\linewidth]{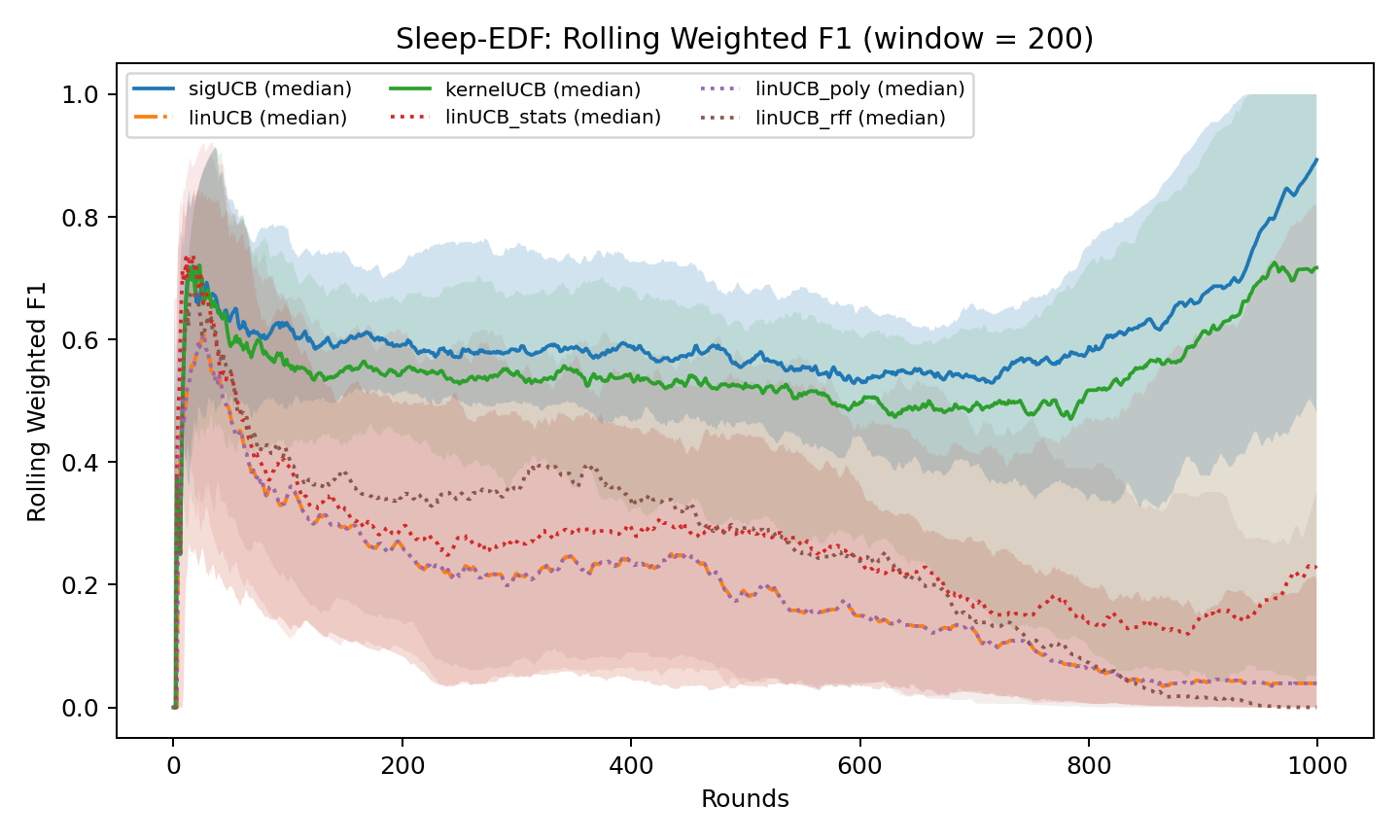}
        \caption{Rolling weighted F1-score}
        \label{fig:sleep_edf_f1}
    \end{subfigure}
    \caption{Sleep stage classification. Lines show medians; shaded bands show 25th--75th percentiles.}
    \label{fig:sleep_edf_combined}
\end{figure}

\begin{table}[h]
    \small
    \caption{Mean classification metrics for sleep stage classification with standard errors.}
    \centering
    \begin{tabular}{ccc}
        \toprule
         Algorithm & Accuracy & Weighted F1-Score \\
         \midrule
        \texttt{DisSigUCB} & \textbf{0.726} (0.028) & \textbf{0.718} (0.029) \\
        \texttt{KernelUCB} & 0.616 (0.033) & 0.595 (0.035)\\
        \texttt{DisLinUCB} & 0.149 (0.018) & 0.144 (0.035)\\
        \texttt{LinUCB-stats} & 0.394 (0.031) & 0.390 (0.033)\\
        \texttt{LinUCB-poly} & 0.149 (0.018) & 0.144 (0.035)\\
        \texttt{LinUCB-rff} & 0.238 (0.028) & 0.221 (0.029) \\
        \bottomrule
    \end{tabular}
    \label{tab:sleep_edf}
\end{table}

\subsection{Hospital Nurse Staffing}
\label{sec:nurse_staffing}
We test the nurse staffing problem considered in \cite{HW:25, KMS:25}, where one needs to select the weekly nurse staffing level based on recent emergency department (ED) visit history. We use the \href{https://a816-health.nyc.gov/hdi/epiquery/visualizations?Indicator=Vomiting&PageType=tsi&PopulationSource=Syndromic&Subtopic=All&Topic=All&Year=2025}{Syndromic Surveillance (Diseases and Conditions) dataset} \cite{NYCDOHM:26} from NYC Health
and extract the vomiting-related ED visit counts from January 1, 2019 to December 31, 2025. 

{\bf Context.}
Each round $t$ corresponds to one week. Let $V_{t,s}$ denote the total number of vomiting-related ED visits on day $s \in [7]$ of week $t$. At the end of week $t$, the learner observes the daily visit counts from that week, 
\(
X_t := \left(V_{t,s}\right)_{s \in [7]} \in \mathbb{R}^7.
\)

{\bf Actions and reward.}
After observing the context $X_t$ at the end of week $t$, the learner must decide the staffing level (in nurse units) for the following week $t+1$. Following \cite{HW:25, KMS:25}, we define the nurse demand in week $t$ as $D_t := \frac{V_t}{3}$, where $V_t:=\sum_{s \in [7]} V_{t,s}$ is the total number of vomiting-related ED visits in week $t$. The action is the staffing level selected for week $t+1$:
$
A_t \in \mathcal{A} := \{200, 210, \dots, 1240\}.
$
After week $t+1$ concludes, the demand $D_{t+1}$ is realized and the learner receives the negative Newsvendor loss
$
r_t = -( b(D_{t+1} - A_t)_+ + h(A_t - D_{t+1})_+ ),$
with underage and overage costs $b=0.7$ and $h=0.3$ as in \cite{HW:25, KMS:25}. Thus, 
$f_a(x) := - \mathbb{E}[b(D_{t+1} - a)_+ + h(a - D_{t+1})_+ | X_t = x]$. Although the reward is not observed immediately, it is revealed after week \(t+1\) and before the learner selects \(A_{t+1}\).


{\bf Results.} For each iteration, we use the first six months {(26 rounds)} to tune the hyperparameters and then evaluate on the remaining 6.5 years. Restricting the signature-depth grid to \(N\leq 2\) did not noticeably affect the performance of \texttt{DisSigUCB}. We plot the median cumulative regret in addition to weekly ED visit counts versus median assigned nurse staffing over 10 trials in Figure \ref{fig:hospital_staff_combined}, where we omit the first 100 rounds since all algorithms choose the same actions. 
In terms of cumulative regret, \texttt{KernelUCB}, \texttt{LinUCB-stats}, and \texttt{LinUCB-rff} perform significantly worse than \texttt{DisLinUCB}, \texttt{DisSigUCB}, and \texttt{LinUCB-poly}. While \texttt{DisSigUCB} begins to outperform \texttt{DisLinUCB} after 200 rounds by obtaining a lower median cumulative regret, \texttt{LinUCB-poly} achieves the lowest cumulative regret consistently. We believe this is because the reward function primarily depends on predicting the following week's demand rather than on the shape of the demand trajectory. In such a setting, \texttt{DisSigUCB}'s signature features add capacity without providing additional signal.  However, \texttt{LinUCB-poly}'s advantage does not extend to the other experiments, whereas \texttt{DisSigUCB} performs more consistently across tasks. We also observe that \texttt{KernelUCB} and \texttt{LinUCB-rff} struggle when the weekly demand is high, while \texttt{DisLinUCB} struggles when the weekly demand is low.

\begin{figure}[h]
    \centering
    \begin{subfigure}[t]{0.45\linewidth}
        \centering
        \includegraphics[width=0.9\linewidth]{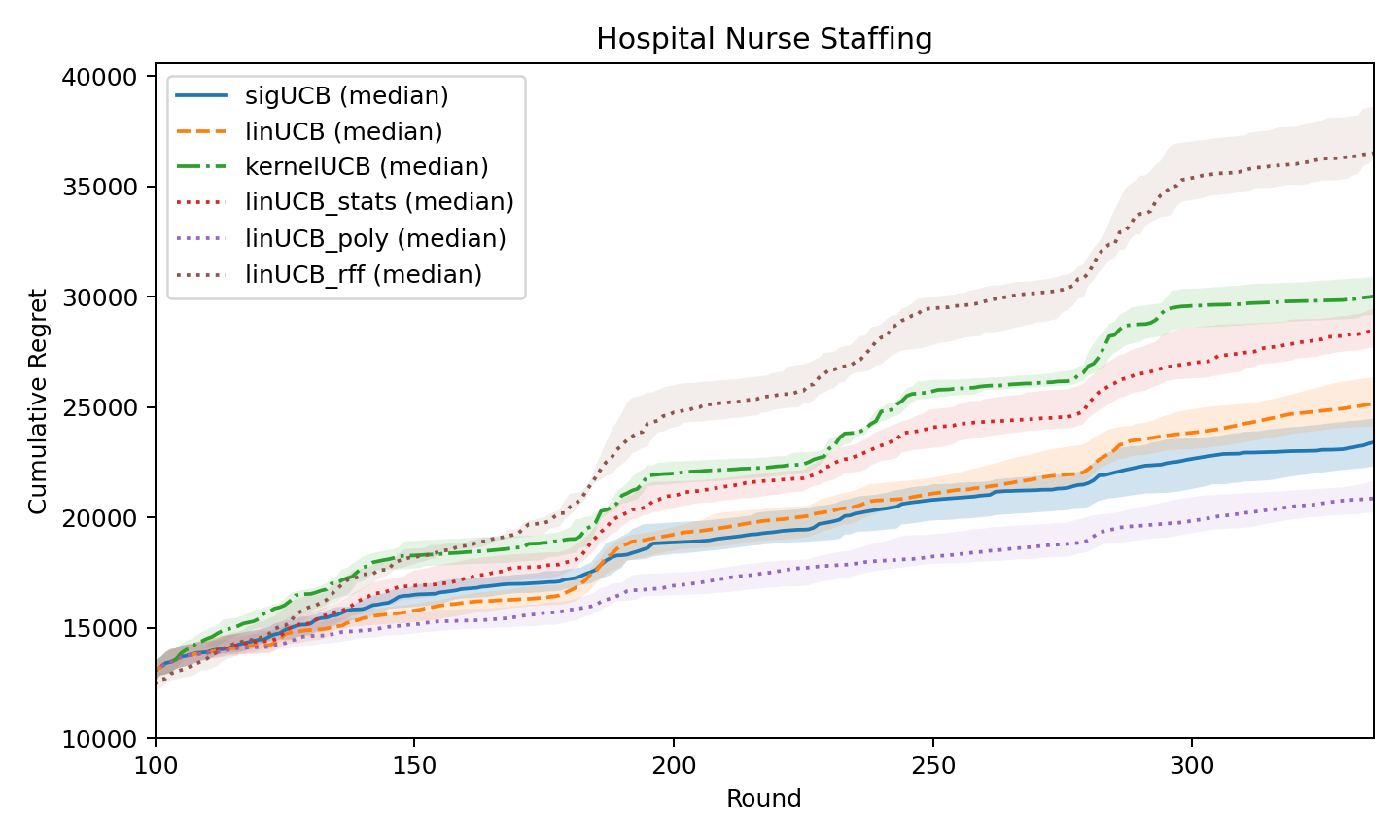}
        \caption{Cumulative regret}
        \label{fig:hospital_staff_regret}
    \end{subfigure}
    \begin{subfigure}[t]{0.50\linewidth}
        \centering
        \includegraphics[width=0.9\linewidth]{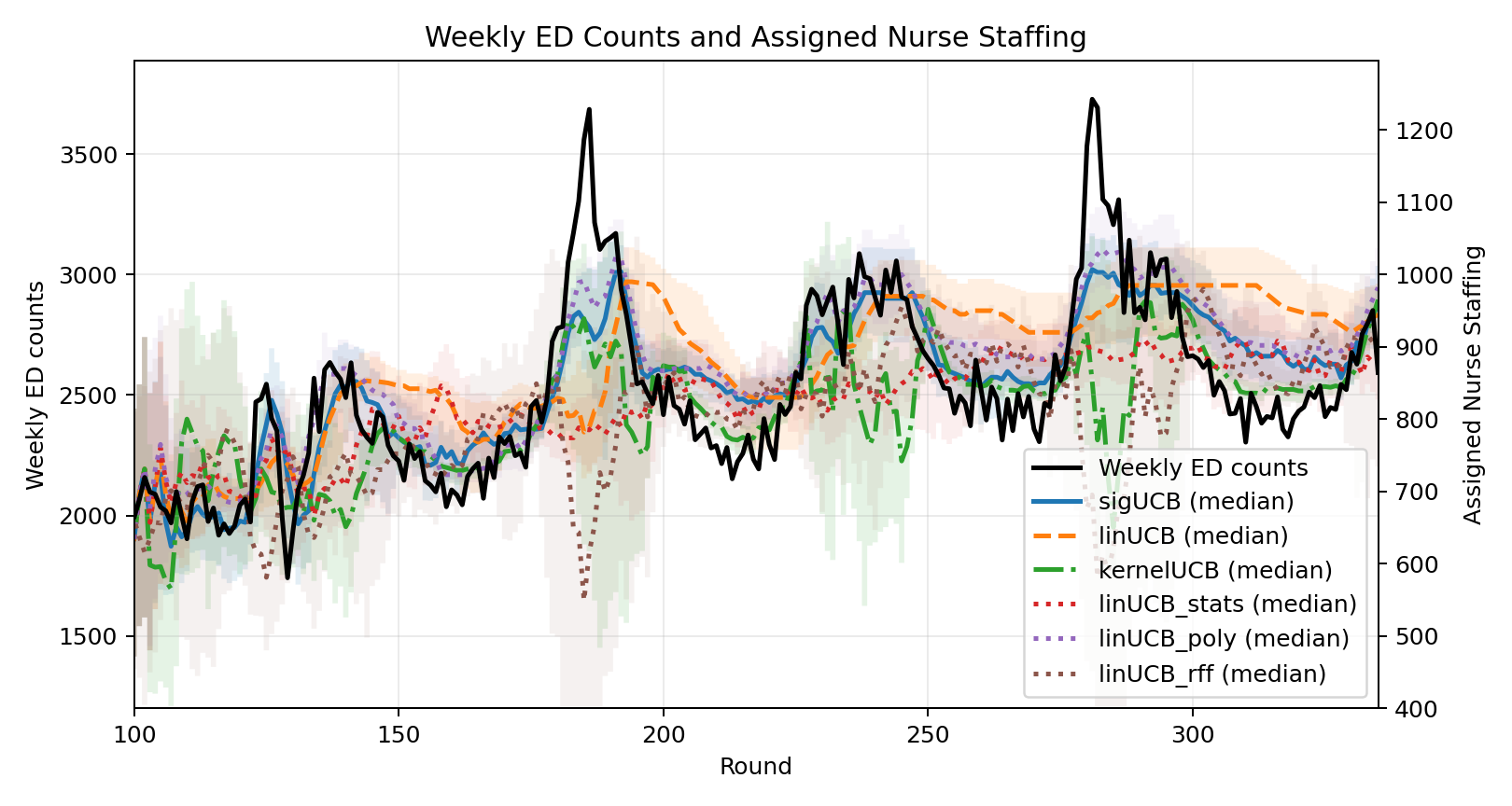}
        \caption{Weekly ED visits and chosen nurse staffing}
        \label{fig:hospital_staff_traj}
    \end{subfigure}
    \caption{Hospital nurse staffing. Lines show medians; shaded bands show 25th--75th percentiles.}
    \label{fig:hospital_staff_combined}
\end{figure}

\section{Conclusion}

This paper introduces a signature-based contextual bandit framework for nonlinear, path-dependent rewards. 
By representing contextual histories through truncated signature features, our approach retains the computational tractability of linear UCB methods while accommodating general nonlinear, history-dependent reward structure. 
We propose \texttt{DisSigUCB} and establish a finite-time regret bound that explicitly captures the tradeoff between signature feature dimension and finite-depth approximation error. 
We further introduce $\zeta$-\texttt{DisSigUCB}, a randomized variant that ensures persistent exploration, and establish consistency and asymptotic normality of the action-specific ridge estimators under additional margin and long-run design conditions. 
Numerical experiments on temperature sensor monitoring, sleep-stage classification, and hospital nurse staffing demonstrate the empirical effectiveness of the proposed framework across a range of nonlinear and path-dependent settings. 

Our \texttt{DisSigUCB} algorithm may be extended in three future directions. 
First, our current model estimates a separate coefficient vector for each action. 
We can extend this to a hybrid model with shared and disjoint parameters, as in \cite{LCLLCB:11, DSLCB:24}, or to a continuous action space, as in \cite{CGKCB:17, fan2025statistical}, allowing information to be pooled across actions. 
Second, we can also use a signature-kernel formulation \cite{KO:19} to replace the explicit truncated signature representation, potentially avoiding the rapid growth of the feature dimension with signature depth. 
Third, we may replace our current self-normalized confidence construction with more recent bandit techniques such as martingale-mixture confidence bounds \cite{FRKP:23} to obtain sharper, possibly time-uniform, confidence sets and corresponding regret guarantees.

{\footnotesize 
\bibliographystyle{abbrv}
\bibliography{references}

@article{fan2025statistical,
  title={Statistical inference under adaptive sampling with linucb},
  author={Fan, Wei and Tan, Kevin and Wei, Yuting},
  journal={arXiv preprint arXiv:2512.00222},
  year={2025}
}

@inproceedings{lyons2014feature,
  title={A feature set for streams and an application to high-frequency financial tick data},
  author={Lyons, Terry and Ni, Hao and Oberhauser, Harald},
  booktitle={Proceedings of the 2014 international conference on big data science and computing},
  pages={1--8},
  year={2014}
}

@inproceedings{li2017lpsnet,
  title={LPSNet: a novel log path signature feature based hand gesture recognition framework},
  author={Li, Chenyang and Zhang, Xin and Jin, Lianwen},
  booktitle={Proceedings of the IEEE international conference on computer vision workshops},
  pages={631--639},
  year={2017}
}

@article{cao2023risk,
  title={Risk of transfer learning and its applications in finance},
  author={Cao, Haoyang and Gu, Haotian and Guo, Xin and Rosenbaum, Mathieu},
  journal={arXiv preprint arXiv:2311.03283},
  year={2023}
}

@article{lyons1998differential,
  title={Differential equations driven by rough signals},
  author={Lyons, Terry J},
  journal={Revista Matem{\'a}tica Iberoamericana},
  volume={14},
  number={2},
  pages={215--310},
  year={1998}
}

@article{morrill2020generalised,
  title={A generalised signature method for multivariate time series feature extraction},
  author={Morrill, James and Fermanian, Adeline and Kidger, Patrick and Lyons, Terry},
  journal={arXiv preprint arXiv:2006.00873},
  year={2020}
}

@book{FV:10, 
  title={Multidimensional Stochastic Processes as Rough Paths: Theory and Applications}, 
  author={Friz, Peter and Victoir, Nicolas B.}, 
  edition={},
  publisher={Cambridge University Press},
  year={2010}, 
}

@article{CGS:22,
author = {Cuchiero, Christa and Gazzani, Guido and Svaluto-Ferro, Sara},
title = {Signature-Based Models: Theory and Calibration},
journal = {SIAM Journal on Financial Mathematics},
volume = {14},
number = {3},
pages = {910-957},
year = {2023}
}

@InProceedings{LLZGLCB:17,
  title = 	 {Provably Optimal Algorithms for Generalized Linear Contextual Bandits},
  author =       {Lihong Li and Yu Lu and Dengyong Zhou},
  booktitle = 	 {Proceedings of the 34th International Conference on Machine Learning},
  pages = 	 {2071--2080},
  year = 	 {2017},
  volume = 	 {70},
  series = 	 {Proceedings of Machine Learning Research},
  month = 	 {06--11 Aug},
  publisher =    {PMLR},
}

@inproceedings{ZXBGLCB:19,
 author = {Zhou, Zhengyuan and Xu, Renyuan and Blanchet, Jose},
 booktitle = {Advances in Neural Information Processing Systems},
 publisher = {Curran Associates, Inc.},
 title = {Learning in Generalized  Linear Contextual Bandits with Stochastic Delays},
 volume = {32},
 year = {2019}
}

@inproceedings{SZBPWKCB:11,
 author = {Krause, Andreas and Ong, Cheng},
 booktitle = {Advances in Neural Information Processing Systems},
 publisher = {Curran Associates, Inc.},
 title = {Contextual Gaussian Process Bandit Optimization},
 volume = {24},
 year = {2011}
}

@misc{VKMFCKCB:13,
      title={Finite-Time Analysis of Kernelised Contextual Bandits}, 
      author={Michal Valko and Nathaniel Korda and Remi Munos and Ilias Flaounas and Nelo Cristianini},
      year={2013},
      eprint={1309.6869},
      archivePrefix={arXiv},
}

@InProceedings{CGKCB:17,
  title = 	 {On Kernelized Multi-armed Bandits},
  author =       {Sayak Ray Chowdhury and Aditya Gopalan},
  booktitle = 	 {Proceedings of the 34th International Conference on Machine Learning},
  pages = 	 {844--853},
  year = 	 {2017},
  volume = 	 {70},
  series = 	 {Proceedings of Machine Learning Research},
  month = 	 {06--11 Aug},
  publisher =    {PMLR},
}

@inbook{DSLCB:24,
   title={Linear Contextual Bandits with Hybrid Payoff: Revisited},
   booktitle={Machine Learning and Knowledge Discovery in Databases. Research Track},
   publisher={Springer Nature Switzerland},
   author={Das, Nirjhar and Sinha, Gaurav},
   year={2024},
   pages={441–455} }

@inproceedings{LCLLCB:11,
   title={A contextual-bandit approach to personalized news article recommendation},
   booktitle={Proceedings of the 19th international conference on World wide web},
   publisher={ACM},
   author={Li, Lihong and Chu, Wei and Langford, John and Schapire, Robert E.},
   year={2010},
   month=apr, pages={661–670},
   collection={WWW ’10} 
}

@inproceedings{APSLSB:11,
 author = {Abbasi-yadkori, Yasin and P\'{a}l, D\'{a}vid and Szepesv\'{a}ri, Csaba},
 booktitle = {Advances in Neural Information Processing Systems},
 publisher = {Curran Associates, Inc.},
 title = {Improved Algorithms for Linear Stochastic Bandits},
 volume = {24},
 year = {2011}
}

@InProceedings{CMBOSOM:20,
  title = 	 {OSOM: A simultaneously optimal algorithm for multi-armed and linear contextual bandits},
  author =       {Chatterji, Niladri and Muthukumar, Vidya and Bartlett, Peter},
  booktitle = 	 {Proceedings of the Twenty Third International Conference on Artificial Intelligence and Statistics},
  pages = 	 {1844--1854},
  year = 	 {2020},
  volume = 	 {108},
  series = 	 {Proceedings of Machine Learning Research},
  month = 	 {26--28 Aug},
  publisher =    {PMLR},
}

@article{LV:07,
title = {An extension theorem to rough paths},
journal = {Annales de l'Institut Henri Poincaré C, Analyse non linéaire},
volume = {24},
number = {5},
pages = {835-847},
year = {2007},
author = {Terry Lyons and Nicolas Victoir},
}

@InProceedings{CLRS:11,
  title = 	 {Contextual Bandits with Linear Payoff Functions},
  author = 	 {Chu, Wei and Li, Lihong and Reyzin, Lev and Schapire, Robert},
  booktitle = 	 {Proceedings of the Fourteenth International Conference on Artificial Intelligence and Statistics},
  pages = 	 {208--214},
  year = 	 {2011},
  volume = 	 {15},
  series = 	 {Proceedings of Machine Learning Research},
  address = 	 {Fort Lauderdale, FL, USA},
  month = 	 {11--13 Apr},
  publisher =    {PMLR},
}

@article{HW:25,
  author  = {Huang, Chengpiao and Wang, Kaizheng},
  title   = {A Stability Principle for Learning Under Nonstationarity},
  journal = {Operations Research},
  year    = {2025},
  volume  = {73},
  number  = {6},
  pages   = {3044--3064},
}

@inproceedings{FCGS:10,
     author = {Filippi, Sarah and Cappe, Olivier and Garivier, Aur\'{e}lien and Szepesv\'{a}ri, Csaba},
     booktitle = {Advances in Neural Information Processing Systems},
     publisher = {Curran Associates, Inc.},
     title = {Parametric Bandits: The Generalized Linear Case},
     volume = {23},
     year = {2010}
}

@inbook{CK:25,
   title={A Primer on the Signature Method in Machine Learning},
   booktitle={Signature Methods in Finance},
   publisher={Springer Nature Switzerland},
   author={Chevyrev, Ilya and Kormilitzin, Andrey},
   year={2025},
   month={Jun}, 
    pages={3–64}
}

@article{KO:19,
  title   = {Kernels for Sequentially Ordered Data},
  author  = {Király, Franz J. and Oberhauser, Harald},
  journal = {Journal of Machine Learning Research},
  volume  = {20},
  number  = {31},
  pages   = {1--45},
  year    = {2019}
}

@article{BKPSL:19,
  title   = {Deep Signature Transforms},
  author  = {Bonnier, Patric and Kidger, Patrick and Perez Arribas, Imanol and Salvi, Cristopher and Lyons, Terry},
  journal = {arXiv preprint arXiv:1905.08494},
  year    = {2019}
}

@inproceedings{GGJKL:24,
  title={Transportation marketplace rate forecast using signature transform},
  author={Gu, Haotian and Guo, Xin and Jacobs, Timothy L and Kaminsky, Philip and Li, Xinyu},
  booktitle={Proceedings of the 30th ACM SIGKDD Conference on Knowledge Discovery and Data Mining},
  pages={4997--5005},
  year={2024}
}

@article{horvath2026fast,
  title={How Fast Do Signatures Learn? Statistical Theory and Applications for Path Regression},
  author={Horvath, Blanka and Su, Wen and Su, Wu and Wang, Binnan and Zhang, Ruixun},
  journal={arXiv preprint arXiv:2607.17865},
  year={2026}
}

@inproceedings{ghosh2021problem,
  title={Problem-complexity adaptive model selection for stochastic linear bandits},
  author={Ghosh, Avishek and Sankararaman, Abishek and Kannan, Ramchandran},
  booktitle={International Conference on Artificial Intelligence and Statistics},
  pages={1396--1404},
  year={2021},
  organization={PMLR}
}

@inproceedings{gentile2017context,
  title={On context-dependent clustering of bandits},
  author={Gentile, Claudio and Li, Shuai and Kar, Purushottam and Karatzoglou, Alexandros and Zappella, Giovanni and Etrue, Evans},
  booktitle={International Conference on machine learning},
  pages={1253--1262},
  year={2017},
  organization={PMLR}
}

@misc{JGCYPLBMN:25,
      title={Sig-DEG for Distillation: Making Diffusion Models Faster and Lighter}, 
      author={Lei Jiang and Wen Ge and Niels Cariou-Kotlarek and Mingxuan Yi and Po-Yu Chen and Lingyi Yang and Francois Buet-Golfouse and Gaurav Mittal and Hao Ni},
      year={2025},
      eprint={2508.16939},
      archivePrefix={arXiv},
}

@article{C:54,
    author = {Chen, Kuo-Tsai},
    title = {Iterated Integrals and Exponential Homomorphisms†},
    journal = {Proceedings of the London Mathematical Society},
    volume = {s3-4},
    number = {1},
    pages = {502-512},
    year = {1954},
    month = {01},
}

@article{C:57,
    author = {Kuo-Tsai Chen},
    journal = {Annals of Mathematics},
    number = {1},
    pages = {163--178},
    publisher = {[Annals of Mathematics, Trustees of Princeton University on Behalf of the Annals of Mathematics, Mathematics Department, Princeton University]},
    title = {Integration of Paths, Geometric Invariants and a Generalized Baker- Hausdorff Formula},
    urldate = {2026-01-20},
    volume = {65},
    year = {1957}
}

@techreport{KMS:25,
  title        = {The Nonstationary Newsvendor: Data-Driven Nonparametric Learning},
  author       = {Keskin, N. Bora and Min, Xu and Song, Jing-Sheng Jeannette},
  institution  = {SSRN},
  type         = {SSRN Working Paper},
  year         = {2025},
  month        = nov,
  doi          = {10.2139/ssrn.3866171}
}

@misc{ILD:04,
  author       = {Madden, Samuel and Bodik, Peter and Hong, Wei and Guestrin, Carlos and Paskin, Mark and Thibaux, Romain},
  title        = {Intel Lab Data},
  howpublished = {\url{https://db.csail.mit.edu/labdata/labdata.html}},
  year         = {2004},
}

@misc{A:18,
      title={Derivatives pricing using signature payoffs}, 
      author={Imanol Perez Arribas},
      year={2018},
      eprint={1809.09466},
      archivePrefix={arXiv},
      primaryClass={q-fin.CP},
      url={https://arxiv.org/abs/1809.09466}, 
}

@ARTICLE{KZTKO:00,
  author={Kemp, B. and Zwinderman, A.H. and Tuk, B. and Kamphuisen, H.A.C. and Oberye, J.J.L.},
  journal={IEEE Transactions on Biomedical Engineering}, 
  title={Analysis of a sleep-dependent neuronal feedback loop: the slow-wave microcontinuity of the EEG}, 
  year={2000},
  volume={47},
  number={9},
  pages={1185-1194},
  doi={10.1109/10.867928}}

@misc{NYCDOHM:26,
  author       = {{New York City Department of Health and Mental Hygiene}},
  title        = {{EpiQuery: Syndromic Surveillance Data -- Vomiting}},
  year         = {2026},
  howpublished = {\url{https://a816-health.nyc.gov/hdi/epiquery/visualizations?Indicator=Vomiting&PageType=tsi&PopulationSource=Syndromic&Subtopic=All&Topic=All}},
}

@INPROCEEDINGS{YJNL:16,
  author={Yang, Weixin and Jin, Lianwen and Hao Ni and Lyons, Terry},
  booktitle={2016 23rd International Conference on Pattern Recognition (ICPR)}, 
  title={Rotation-free online handwritten character recognition using dyadic path signature features, hanging normalization, and deep neural network}, 
  year={2016},
  pages={4083-4088},
  doi={10.1109/ICPR.2016.7900273}}

@misc{WNLH:2020,
      title={Signature features with the visibility transformation}, 
      author={Yue Wu and Hao Ni and Terence J. Lyons and Robin L. Hudson},
      year={2020},
      eprint={2004.04006},
      archivePrefix={arXiv},
      primaryClass={cs.LG},
      url={https://arxiv.org/abs/2004.04006}, 
}

@techreport{TMC:11,
  title        = {User-Friendly Tail Bounds for Matrix Martingales},
  author       = {Tropp, Joel A.},
  institution  = {California Institute of Technology},
  number       = {ACM 2011-01},
  year         = {2011},
  url          = {https://users.cms.caltech.edu/~jtropp/reports/Tro10-User-Friendly-Martingale-TR.pdf}
}

@article{CLS:21,
author = {Haoyu Chen and Wenbin Lu and Rui Song},
title = {Statistical Inference for Online Decision Making: In a Contextual Bandit Setting},
journal = {Journal of the American Statistical Association},
volume = {116},
number = {533},
pages = {240--255},
year = {2021},
publisher = {Taylor \& Francis},
doi = {10.1080/01621459.2020.1770098},
note ={PMID: 33737759},
URL = {
        https://doi.org/10.1080/01621459.2020.1770098
},
eprint = {  
        https://doi.org/10.1080/01621459.2020.1770098
}
}

@article{GZ:13,
  author  = {Alexander Goldenshluger and Assaf Zeevi},
  title   = {A Linear Response Bandit Problem},
  journal = {Stochastic Systems},
  year    = {2013},
  volume  = {3},
  number  = {1},
  pages   = {230--261},
  doi     = {10.1287/11-SSY032}
}

@article{BB:20,
  author  = {Hamsa Bastani and Mohsen Bayati},
  title   = {Online Decision Making with High-Dimensional Covariates},
  journal = {Operations Research},
  year    = {2020},
  volume  = {68},
  number  = {1},
  pages   = {276--294},
  doi     = {10.1287/opre.2019.1902}
}

@article{BBK:21,
  author  = {Hamsa Bastani and Mohsen Bayati and Khashayar Khosravi},
  title   = {Mostly Exploration-Free Algorithms for Contextual Bandits},
  journal = {Management Science},
  year    = {2021},
  volume  = {67},
  number  = {3},
  pages   = {1329--1349},
  doi     = {10.1287/mnsc.2020.3605}
}

@book{HH:1980,
  author    = {Peter Hall and C. C. Heyde},
  title     = {Martingale Limit Theory and Its Application},
  publisher = {Academic Press},
  address   = {New York},
  year      = {1980},
  isbn      = {978-0-12-319350-6}
}

@inproceedings{FRKP:23,
  author    = {Hamish Flynn and David Reeb and Melih Kandemir and Jan R. Peters},
  title     = {Improved Algorithms for Stochastic Linear Bandits Using Tail Bounds for Martingale Mixtures},
  booktitle = {Advances in Neural Information Processing Systems},
  volume    = {36},
  pages     = {45102--45136},
  year      = {2023},
  publisher = {Curran Associates, Inc.},
  url       = {https://proceedings.neurips.cc/paper_files/paper/2023/file/8db0d67d22e0ec08c95b810be3a66907-Paper-Conference.pdf},
  doi       = {10.52202/075280-1955}
}

@misc{ZJM:20,
      title={Inference for Batched Bandits}, 
      author={Kelly W. Zhang and Lucas Janson and Susan A. Murphy},
      year={2021},
      eprint={2002.03217},
      archivePrefix={arXiv},
      primaryClass={cs.LG},
      url={https://arxiv.org/abs/2002.03217}, 
}

@misc{KZ:24,
      title={Inference with the Upper Confidence Bound Algorithm}, 
      author={Koulik Khamaru and Cun-Hui Zhang},
      year={2024},
      eprint={2408.04595},
      archivePrefix={arXiv},
      primaryClass={stat.ML},
      url={https://arxiv.org/abs/2408.04595}, 
}

@inproceedings{GCGz:17,
  title={Misspecified Linear Bandits},
  author={Ghosh, Avishek and Chowdhury, Sayak Ray and Gopalan, Aditya},
  booktitle={Proceedings of the Thirty-First AAAI Conference on Artificial Intelligence},
  pages={3761--3767},
  year={2017}
}

@inproceedings{FGMZ:20,
  title={Adapting to Misspecification in Contextual Bandits},
  author={Foster, Dylan J. and Gentile, Claudio and Mohri, Mehryar and Zimmert, Julian},
  booktitle={Advances in Neural Information Processing Systems},
  volume={33},
  year={2020}
}

@inproceedings{TIHSFKK:21,
  title={A Parameter-Free Algorithm for Misspecified Linear Contextual Bandits},
  author={Takemura, Kei and Ito, Shinji and Hatano, Daisuke and Sumita, Hanna and Fukunaga, Takuro and Kakimura, Naonori and Kawarabayashi, Ken-ichi},
  booktitle={Proceedings of the 24th International Conference on Artificial Intelligence and Statistics},
  series={Proceedings of Machine Learning Research},
  volume={130},
  pages={3367--3375},
  year={2021},
  publisher={PMLR}
}

@article{LS:26,
author = {Pangpang Liu and Will Wei Sun},
title = {Fairness-Aware Contextual Dynamic Pricing with Strategic Buyers},
journal = {Journal of the American Statistical Association},
volume = {0},
number = {0},
pages = {1--13},
year = {2026},
publisher = {Taylor \& Francis},
doi = {10.1080/01621459.2026.2648862},
URL = {https://doi.org/10.1080/01621459.2026.2648862},
eprint = {https://doi.org/10.1080/01621459.2026.2648862}
}

@article{LYWS:25,
author = {Pangpang Liu and Zhuoran Yang and Zhaoran Wang and Will Wei Sun},
title = {Contextual Dynamic Pricing with Strategic Buyers},
journal = {Journal of the American Statistical Association},
volume = {120},
number = {550},
pages = {896--908},
year = {2025},
publisher = {Taylor \& Francis},
doi = {10.1080/01621459.2024.2370613},
URL = {https://doi.org/10.1080/01621459.2024.2370613},
eprint = {https://doi.org/10.1080/01621459.2024.2370613}
}

@article{ZJY:26,
author = {Zifeng Zhao and Feiyu Jiang and Yi Yu},
title = {Contextual Dynamic Pricing: Algorithms, Optimality, and Local Differential Privacy Constraints},
journal = {Journal of the American Statistical Association},
volume = {0},
number = {0},
pages = {1--16},
year = {2026},
publisher = {Taylor \& Francis},
doi = {10.1080/01621459.2026.2615188},
URL = {https://doi.org/10.1080/01621459.2026.2615188},
eprint = {https://doi.org/10.1080/01621459.2026.2615188}
}

@article{ZHWZS:25,
author = {Jie Zhou and Botao Hao and Zheng Wen and Jingfei Zhang and Will Wei Sun},
title = {Stochastic Low-Rank Tensor Bandits for Multi-Dimensional Online Decision Making},
journal = {Journal of the American Statistical Association},
volume = {120},
number = {549},
pages = {198--211},
year = {2025},
publisher = {Taylor \& Francis},
doi = {10.1080/01621459.2024.2311364},
URL = {https://doi.org/10.1080/01621459.2024.2311364},
eprint = {https://doi.org/10.1080/01621459.2024.2311364}
}

@article{LSL:26,
author = {Seong Jin Lee and Will Wei Sun and Yufeng Liu},
title = {Low-Rank Online Dynamic Assortment with Dual Contextual Information},
journal = {Journal of the American Statistical Association},
volume = {0},
number = {0},
pages = {1--13},
year = {2026},
publisher = {Taylor \& Francis},
doi = {10.1080/01621459.2025.2597043},
URL = {https://doi.org/10.1080/01621459.2025.2597043
},
eprint = {https://doi.org/10.1080/01621459.2025.2597043
}
}

@article{CWFWL:24,
author = {Yi Chen and Yining Wang and Ethan X. Fang and Zhaoran Wang and Runze Li},
title = {Nearly Dimension-Independent Sparse Linear Bandit over Small Action Spaces via Best Subset Selection},
journal = {Journal of the American Statistical Association},
volume = {119},
number = {545},
pages = {246--258},
year = {2024},
publisher = {Taylor \& Francis},
doi = {10.1080/01621459.2022.2108816},
URL = {https://doi.org/10.1080/01621459.2022.2108816
},
eprint = {https://doi.org/10.1080/01621459.2022.2108816}
}

@article{LKW:25,
author = {Licong Lin and Koulik Khamaru and Martin J. Wainwright},
title = {{Semiparametric inference based on adaptively collected data}},
volume = {53},
journal = {The Annals of Statistics},
number = {3},
publisher = {Institute of Mathematical Statistics},
pages = {989 -- 1014},
year = {2025},
doi = {10.1214/24-AOS2485},
URL = {https://doi.org/10.1214/24-AOS2485}
}

@article{CCLLLZ:26,
author = {Xiangyu Chang and Xi Chen and Zehua Lai and He Li and Zhihong Liu and Yichen Zhang},
title = {Online Statistical Inference for Contextual Bandits via Stochastic Gradient Descent},
journal = {Journal of the American Statistical Association},
volume = {0},
number = {0},
pages = {1--14},
year = {2026},
publisher = {Taylor \& Francis},
doi = {10.1080/01621459.2026.2621503},
URL = {https://doi.org/10.1080/01621459.2026.2621503},
eprint = {https://doi.org/10.1080/01621459.2026.2621503}
}

@article{HSZ:25,
author = {Qiyu Han and Will Wei Sun and Yichen Zhang},
title = {{Online statistical inference in decision-making with matrix context}},
volume = {53},
journal = {The Annals of Statistics},
number = {5},
publisher = {Institute of Mathematical Statistics},
pages = {1963 -- 1986},
year = {2025},
doi = {10.1214/25-AOS2526},
URL = {https://doi.org/10.1214/25-AOS2526}
}
}


\appendix

\section{Detailed Proofs in Section \ref{sec:regret}}\label{sec:detailed_proofs}

In this section, we include the proofs of the lemmas and propositions in Section \ref{sec:regret}. 
\subsection{Proof of Proposition \ref{prop:prune_time}}\label{sec:sig_prune}


An important operation on words, which reflects the multiplicative structure of signatures, is the shuffle product defined below. The shuffle product characterizes the multiplicative structure of signature coordinates, as stated in the following lemma.

\begin{definition}[Shuffle product]
Let $u = (i_1,\dots,i_k)$ and $v = (j_1,\dots,j_\ell)$ be two words. The \emph{shuffle product} $u \shuffle v$ is defined as the formal sum of all words obtained by interleaving $u$ and $v$ while preserving the relative order of the letters within $u$ and $v$. That is,
\[
u \shuffle v := \sum_{\sigma \in \mathrm{Sh}(k,\ell)} (w_{\sigma(1)}, \dots, w_{\sigma(k+\ell)}),
\]
where $w = (i_1,\dots,i_k,j_1,\dots,j_\ell)$ and $\mathrm{Sh}(k,\ell)$ denotes the set of $(k,\ell)$-shuffles, i.e., permutations $\sigma$ of $\{1,\dots,k+\ell\}$ such that
\[
\sigma^{-1}(1) < \cdots < \sigma^{-1}(k), \quad
\sigma^{-1}(k+1) < \cdots < \sigma^{-1}(k+\ell).
\]
\end{definition}

\begin{lemma}[Shuffle identity \cite{CK:25}] \label{lemma:shuffle}
Let $\hat X$ be a path and let $u, v$ be words. Then the signature satisfies the shuffle identity
\[
\Sig^u(\hat X)\,\Sig^v(\hat X) \;=\; \Sig^{u \shuffle v}(\hat X),
\]
i.e.,
\[
\Sig^u(\hat X)\,\Sig^v(\hat X)
= \sum_{w \in u \shuffle v} \Sig^w(\hat X).
\]
\end{lemma}

\begin{proof}[Proof of Proposition \ref{prop:prune_time}] \label{ssec:proof_sig_prune}
For ease of notation, we denote $S^w := S^w(\hat{X})$ for any word $w$. 
   
For a word $w$ that contains only time coordinates, $S^w$ is a constant, and hence lies in $\text{span}(1) \subseteq \text{span}(B_N).$

Now fix a word $w$ with $|w| \leq N$ that (i) contains at least one coordinate $j \neq 0$ and (ii) ends with time coordinate $0$. Note that $w$ can be uniquely written as $w = u0^m$ for some non-empty $u$ which does not end with 0 and $1 \leq m \leq |w|-1$. 

We prove $\Sig^w \in \text{span}(B_N)$ by induction on $m$. First, for $m=1$, we have $w = u0$. By the shuffle identity in Lemma \ref{lemma:shuffle}, we have
\[
S^u S^0 = \sum_{v \in u \shuffle 0} S^v = S^w + \sum_{\substack{v\in u\shuffle 0 \\ v\neq u0}} S^v,
\]
where in the second equality we separate the shuffle $v = u0$ from the sum. For any $v \in u \shuffle 0$ with $v \neq u0$, the word $v$ does not end in 0 and still contains a channel different from 0 (coming from $u$), so $S^v$ belongs to $B_N$. Furthermore, $S^u \in B_N$ since $u$ is non-empty and does not end with 0. Finally, $S^0 = b-a$ is a scalar constant by definition, which implies that $S^w$ is a linear combination of elements of $B_N$, i.e., $S^w \in \text{span}(B_N)$.

Now, suppose by induction that for $m>1$, any $S^w$ with $w = u0^m$ lies in $\text{span}(B_N)$. We aim to show that $S^{w'}$ with $w' = u0^{m+1}$ also belongs to $\text{span}(B_N)$. By the shuffle identity in Lemma \ref{lemma:shuffle}, we have
    \[
    S^{u0^m} S^0 = \sum_{v \in u0^m \shuffle 0} S^v = (m+1) S^{w'} + \sum_{v \in \sigma(u, 0) 0^m} S^v
    \]
    where $\sigma(u,0)$ denotes the set of words obtained by inserting a single $0$ in $u$, but not after its last letter. In the second equality, we separate the $m+1$ shuffles where 0 is inserted into the final block $0^m$, all of which give the same word $w'=u0^{m+1}$. Each word $v\in \sigma(u,0)0^m$ has the form $\tilde u\,0^m$, where $\tilde u$ is non-empty and does not end in $0$ and the last letter of $v$ is still $0$. Thus $v$ satisfies the induction hypothesis, which implies that $S^v\in \mathrm{span}(B_N)$ for all $v\in \sigma(u,0)\,0^m$. By the induction hypothesis we also have $S^{u0^m}\in \mathrm{span}(B_N)$. Using the equality above and the fact that $S^0 = b-a$ is a scalar constant, we conclude that $S^{w'}\in\mathrm{span}(B_N)$.
\end{proof}

\subsection{Proof of Lemma \ref{lemma:compact_K_construction}}
\label{ssec:proof_compact_K}
\begin{proof}[\unskip\nopunct]
By the Arzelà--Ascoli theorem, for a sequence $\{x_n\}$ in $\cK$, there exists a convergent subsequence $\{x_{n_k}\} \to x \in \mathcal{K}$ with respect to $\|\cdot\|_\infty$. Moreover, the uniform limit $x$ lies in $\cK$, so $\cK$ is closed and hence compact in $\|\cdot\|_\infty$. Next, we want to show that $\cK$ is compact with respect to $\|\cdot\|_{\cV^p}$. Since $\|x_{n_k} - x\|_\infty \to 0$, it is sufficient to show that $\|x_{n_k} - x\|_p \to 0$ by the definition of $\|\cdot\|_{\cV^p}$. Let $\Pi$ be a partition of $[0,L]$ with $m$ partition points and let $h_k(t) := x_{n_k}(t) - x(t)$. We have
\begin{align*}
\|x_{n_k} - x\|_p^p &= \sup_\Pi \sum_{i=1}^{m-1} |h_k(t_{i+1}) - h_k(t_i)|^p = \sup_\Pi \sum_{i=1}^{m-1} |h_k(t_{i+1}) - h_k(t_i)|^{p - 1/\alpha} |h_{k}(t_{i+1}) - h_k(t_i)|^{1/\alpha} \\
&\leq (2 \|h_k\|_\infty)^{p - 1/\alpha} \sup_\Pi \sum_{i=1}^{m-1} |h_k(t_{i+1}) - h_k(t_i)|^{1/\alpha} \leq (2 \|h_k\|_\infty)^{p - 1/\alpha} \sup_\Pi \sum_{i=1}^{m-1} (2H_2)^{1/\alpha}|t_{i+1} - t_i| \\
&= (2 \|h_k\|_\infty)^{p - 1/\alpha}(2H_2)^{1/\alpha}L,
\end{align*}
where in the second line we upper bound $h_k(t)$ by its sup norm and in the third line we use the Hölder bound on $x$. Furthermore, since $p - 1/\alpha > 0$ and $\|h_k\|_\infty \to 0$ by uniform convergence, we have that $\|x_{n_k} - x\|_p \to 0.$ 
\end{proof}

\subsection{Proof of Lemma \ref{prop:sig_l2}} \label{ssec:proof_sig_l2}
To prove  Lemma \ref{prop:sig_l2}, we first state the following useful lemma, which provides the factorial decay bounds needed to control the truncated signature. 
\begin{lemma}[Lyons' Extension Theorem \cite{LV:07}] \label{thm:lyons_ext} Let $X \in \pathspaceL$ and suppose $X$ satisfies Assumption~\ref{assum:unified}(i), so that its time-augmented signature is well-defined. Then there exists a constant $C_p>0$, depending only on \(p\), such that for any level \(k=1,\dots,N\), word $v \in \{0,\dots,d\}^k$, and $0 \leq s \leq t \leq T$,
\[
\bigl|\Sig^v(\hat X_{[s,t]})\bigr|
\le
\frac{C_p\,\omega(s,t)^{k/p}}{(k/p)!},
\]
where $\omega$ is a control function defined by
\[
\omega(s,t)
:=
\sum_{m=1}^{\lfloor p\rfloor}
\sup_{\mathcal P_{[s,t]}=\{t_0<\cdots<t_n\}}
\sum_{l=1}^n
\left|
\left\{\Sig^{(i_1,\dots,i_m)}(\hat X_{[t_{l-1},t_l]})\right\}_{i_1,\dots,i_m=0}^d
\right|^{\frac{p}{m}}, \qquad
0 \leq s \leq t \leq T.
\]
\end{lemma}

\begin{proof}[Proof of Lemma \ref{prop:sig_l2}]
Since $s \mapsto s$ has finite 1-variation on $[0,L]$ and $s \mapsto X_{t-L+s}$ has finite $p$-variation on $[0,L]$, the time-augmented path $\hat{X}_{[t-L, t]}$ also has finite $p$-variation. 
Moreover, because the contexts $\{ X_{[t-L,t]}\}_{t=L}^T$ lie in the compact finite-$p$-variation class considered in Assumption~\ref{assum:unified}(i), according to \cite{FV:10}, there exists $M < \infty$ such that the control in Lemma~\ref{thm:lyons_ext} satisfies \(\omega(0,L) \leq M\).

We now bound each coordinate of $\bar{x}_t = [X_t, \Sig_N(\hat{X}_{[t-L,t]})]$. 
The boundedness of the first $d$ components is given by Assumption~\ref{assum:unified}(i) as $|X_{t,i}| \leq H.$
Now consider a signature coordinate corresponding to word $v$ of level $k \leq N$. By Lemma~\ref{thm:lyons_ext},
\begin{equation}
    |\Sig^v(\hat{X}_{[t-L,t]})| \leq \frac{C_p M^{k/p}}{(k/p)!}.
\end{equation}
Since the factorial term in the denominator grows faster than $M^{k/p}$, there exists $b_\Sig < \infty$ such that
\begin{equation}
    \sup_{k \geq 1} \frac{C_p M^{k/p}}{(k/p)!} \leq b_\Sig.
\end{equation}
Thus, for $b := \max\{H, 1, b_\Sig\}$ where the 1 accounts for the level-zero signature coordinate, we have $|\bar{x}_{t,i}| \leq b$ for every $t \in [T]$ and every coordinate $i$. 
Since the pruned signature features form a subset of the full signature features, the same coordinate-wise bound applies to the pruned feature vector.
\end{proof}

\subsection{Proof of Lemma \ref{lemma:x_S0_S1}} \label{ssec:x_S0_S1}

In this section, we prove Lemma \ref{lemma:x_S0_S1}, which upper bounds the cumulative regret in terms of three quantities, $W_\text{UCB}$, $W_\text{sel}$, and $W_\text{opt}$. 
We begin by stating Lemma~\ref{thm:self_bound_martingale1}, which provides a high-probability bound for the cumulative noise term in ridge regression and is then used in Lemma \ref{prop:x_beta} to derive an error bound for the reward estimate.

\begin{lemma}[Theorem 1 of \cite{APSLSB:11}]
\label{thm:self_bound_martingale1}
Let $\{\cF_t\}_{t=0}^\infty$ be a filtration. 
Let $\{\eta_t\}_{t=1}^\infty$ be a real-valued stochastic process such that $\eta_t$ is $\tilde\cF_t$-measurable, conditionally mean-zero, and 1-sub-Gaussian given $\cF_{t}$. 
Let $\{\bar{x}_t\}_{t=1}^\infty$ be an $\mathbb{R}^{d+m}$-valued stochastic process such that $\bar{x}_t$ is $\cF_t$-measurable.
Let $\lambda > 0$. 
For any $t \geq 0$, define
\[
\overline V_t := \lambda \mathbf{I} + \sum_{s=1}^t \bar x_s \bar x_s^\top, \qquad S_t = \sum_{s=1}^t \eta_s \bar{x}_s.
\]
Then for any $\delta>0$, with probability at least $1 - \delta$, for all $t \geq 0$, 
\begin{equation}
    \|S_t\|^2_{\overline V_t^{-1}} \leq 2 \log \left(\frac{\det(\overline V_t)^{1/2}}{\delta \lambda^{(d+m)/2}}\right).
\end{equation}
\end{lemma}


\begin{lemma} \label{prop:x_beta}  Suppose that Assumption \ref{assum:unified} holds and there exists $S$ such that $|\beta_a^*| \leq S$ for all $a \in [K]$, and $|\bar{x}_t| \leq B$ for all $t$. Then with probability at least $1 - \delta$, we have  for all $t \in [T]$
\begin{equation} \label{eq:x_beta}
|\langle \bar{x}_t, \beta^*_a - \hat{\beta}_{a,t} \rangle | \leq \left(\gamma + \vep_N\sqrt{\tau_{a,t}}\right) \|\bar{x}_t\|_{\mathbf{M}_{a,t}^{-1}},
\end{equation}
where $\gamma$ is given in \eqref{eq: gamma}.
\end{lemma}

\begin{proof}
Recall that when action $a$ is selected in round $s$, the reward is given by
\begin{equation} \label{eq:reward_form}
    r_s = \langle \beta^*_a, \bar{x}_s\rangle + \vep_{a,N}(X_{[s-L,s]}) + \eta_s.
\end{equation}
Moreover, 
$
\mathbf{M}_{a,t} = \lambda \mathbf{I} + \sum_{s \in \cT_{a,t}} \bar{x}_s \bar{x}_s^\top
$
and 
$\hat{\beta}_{a,t} = \mathbf{M}_{a,t}^{-1} \sum_{s \in \cT_{a,t}} \bar{x}_s r_s.
$
Substituting \eqref{eq:reward_form} yields,
\begin{align}
    \mathbf{M}_{a,t} \hat{\beta}_{a,t} = \left(\sum_{s \in \cT_{a,t}} \bar{x}_s \bar{x}_s^\top \right) \beta^*_a + \sum_{s \in \cT_{a,t}} \bar{x}_s \vep_{a,N}(X_{[s-L, s]}) + \sum_{s \in \cT_{a,t}} \bar{x}_s \eta_s.
\end{align}
Since $\sum_{s \in \cT_{a,t}} \bar{x}_s \bar{x}_s^\top = \mathbf{M}_{a,t} - \lambda \mathbf{I}$, we obtain 
\begin{equation}
    \mathbf{M}_{a,t} (\hat\beta_{a,t} - \beta^*_a) = \sum_{s \in \cT_{a,t}} \bar{x}_s \eta_s + \sum_{s \in \cT_{a,t}} \bar{x}_s \vep_{a,N}(X_{[s-L,s]}) - \lambda \beta_a^*.
\end{equation}
Therefore, the design matrix norm of the coefficient-estimation residual is bounded above as follows,
\begin{equation} \label{eq:beta_M_bound}
    \|\hat\beta_{a,t} - \beta^*_a\|_{\mathbf{M}_{a,t}} \leq \left\|\sum_{s \in \cT_{a,t}} \bar{x}_s \eta_s \right\|_{\mathbf{M}_{a,t}^{-1}} + \left\|\sum_{s \in \cT_{a,t}} \bar{x}_s \vep_{a,N}(X_{[s-L,s]})\right\|_{\mathbf{M}_{a,t}^{-1}} + \lambda \|\beta^*_a\|_{\mathbf{M}_{a,t}^{-1}}.
\end{equation}
We bound the three terms separately.

To bound the first term, we apply Lemma~\ref{thm:self_bound_martingale1} with failure probability $\delta/K$ and take a union bound over all actions. Thus, with probability at least $1 - \delta$, for all actions and rounds, we have
\begin{equation} \label{eq:noise_bound1}
    \left\|\sum_{s \in \cT_{a,t}} \bar{x}_s \eta_s \right\|_{\mathbf{M}_{a,t}^{-1}}^2 \leq 2 \log \left(\frac{K\det(\mathbf{M}_{a,t})^{1/2}}{\delta \lambda^{(d+m)/2}}\right).
\end{equation}
Under Assumption \ref{assum:unified}, Lemma~\ref{prop:sig_l2} ensures that $|\bar{x}_t| \leq B$ for all $t \in [T]$. This implies that 
\begin{equation}
    \frac{\det(\mathbf{M}_{a,t})}{\lambda^{(d+m)}} \leq \left(1 + \frac{\tau_{a,t}B^2}{\lambda(d+m)}\right)^{d+m} \leq \left(1 + \frac{TB^2}{\lambda(d+m)}\right)^{d+m},
\end{equation}
where we use $\tau_{a,t} \leq t$ in the second inequality.
Substituting this bound into \eqref{eq:noise_bound1} gives
\begin{equation} \label{eq:xbar_eta_M_bound}
    \left\|\sum_{s \in \cT_{a,t}} \bar{x}_s \eta_s \right\|_{\mathbf{M}_{a,t}^{-1}} 
    \leq \sqrt{(d+m) \log\left(1 + \frac{TB^2}{\lambda(d+m)}\right) + 2 \log \left(\frac{K}{\delta}\right)}
    \leq \sqrt{(d+m) \log\left(\frac{K\left(1 + \frac{TB^2}{\lambda(d+m)}\right)}{\delta}\right)}. 
\end{equation}

Next, we bound the approximation-error term in \eqref{eq:beta_M_bound}. Since $\mathbf{M}_{a,t} = \mathbf{X}_{a,t}^\top \mathbf{X}_{a,t} + \lambda\mathbf{I}$, we have $\mathbf{X}_{a,t} \mathbf{M}_{a,t}^{-1} \mathbf{X}_{a,t}^\top \preceq \mathbf{I}$. Hence,
\begin{equation} \label{eq:xbar_vep_bound}
    \left\|\sum_{s \in \cT_{a,t}} \bar{x}_s \vep_{a,N}(X_{[s-L,s]})\right\|_{\mathbf{M}_{a,t}^{-1}} = \sqrt{\vep_{a,t}^\top \mathbf{X}_{a,t} \mathbf{M}_{a,t}^{-1} \mathbf{X}_{a,t} ^\top \vep_{a,t}} \leq |\vep_{a,t}| \leq \vep_N \sqrt{\tau_{a,t}},
\end{equation}
where $\vep_{a,t} := (\vep_{a,N}(X_{[s-L,s]}))_{s \in \cT_{a,t}}$.
Lastly, we bound the regularization term in \eqref{eq:beta_M_bound}. Since $\mathbf{M}_{a,t}^{-1} \preceq \frac{1}{\lambda}\mathbf{I}$, we have
\begin{equation} \label{eq:lambda_beta_bound}
    \lambda \|\beta^*_a\|_{\mathbf{M}_{a,t}^{-1}} \leq \sqrt{\lambda} |\beta^*_a| \leq S \sqrt{\lambda}.
\end{equation}

Substituting \eqref{eq:xbar_eta_M_bound}, \eqref{eq:xbar_vep_bound}, and \eqref{eq:lambda_beta_bound} in \eqref{eq:beta_M_bound}, with probability at least $1 - \delta$, for all actions and rounds, we have
\begin{equation} \label{eq:beta_M}
    \|\hat\beta_{a,t} - \beta^*_a\|_{\mathbf{M}_{a,t}} \leq \gamma + \vep_N \sqrt{\tau_{a,t}},
\end{equation}
for $\gamma$ as in \eqref{eq: gamma}.

Finally, since $\mathbf{M}_{a,t} \succeq \lambda \mathbf{I}$ is symmetric and positive definite by construction, we denote by $\mathbf{M}^{1/2}_{a,t}$ the symmetric positive-definite square root of \(\mathbf{M}_{a,t}\). 
Then, we have for any action $a \in [K]$ and round $t \in [T]$,
\begin{align*}
    |\langle\bar{x}_t, \beta^*_a - \hat{\beta}_{a,t} \rangle| &= |(\mathbf{M}_{a,t}^{-1/2} \bar{x}_t)^\top \mathbf{M}_{a,t}^{1/2} (\beta^*_a - \hat{\beta}_{a,t}) | \\
    &\leq |\mathbf{M}_{a,t}^{-1/2} \bar{x}_t| \cdot |\mathbf{M}_{a,t}^{1/2} (\beta^*_a - \hat{\beta}_{a,t})| \\
    &= \|\bar{x}_t\|_{\mathbf{M}_{a,t}^{-1}} \cdot \|\beta^*_a - \hat{\beta}_{a,t}\|_{\mathbf{M}_{a,t}} \\
    &\leq (\gamma + \vep_N \sqrt{\tau_{a,t}}) \|\bar{x}_t\|_{\mathbf{M}_{a,t}^{-1}},
\end{align*}
where the second line follows from Cauchy-Schwarz and the last inequality follows from \eqref{eq:beta_M} with probability at least $1-\delta$.
\end{proof}

\begin{proof}[Proof of Lemma~\ref{lemma:x_S0_S1}]
For each round $t=L, \dots, T$, let $A_t$ be the action chosen by \texttt{DisSigUCB} and $A^*_t \in \arg \max_{a \in [K]} \langle \beta^*_a, \bar{x}_t \rangle$ be the action that optimizes the linear approximation. 
By Theorem~\ref{thm:linear_regret}, we have $R(T)\leq \sum_{t=L}^T \langle \beta^*_{A^*_t} - \beta^*_{A_t}, \bar{x}_t \rangle + 2(T-L+1)\vep_N$.

We next bound each term in the sum. By Lemma~\ref{prop:x_beta}, with probability at least $1 - \delta$, we have
\begin{align*}
    \langle \bar{x}_t, \beta_{A^*_t}^* \rangle - \langle \bar{x}_t, \beta^*_{A_t}\rangle
    &= \langle \bar{x}_t, \beta_{A^*_t}^* - \hat{\beta}_{A^*_t, t} \rangle + \langle  \bar{x}_t, \hat{\beta}_{A^*_t, t} \rangle -\left( \langle \bar{x}_t, \beta_{A_t}^* - \hat{\beta}_{A_t, t} \rangle + \langle  \bar{x}_t, \hat{\beta}_{A_t, t} \rangle \right)\\
    &\leq \langle  \bar{x}_t, \hat{\beta}_{A^*_t, t}\rangle + \gamma \|\bar{x}_t\|_{\mathbf{M}_{A^*_t, t}^{-1}} - \langle  \bar{x}_t, \hat{\beta}_{A_t, t} \rangle + \gamma \|\bar{x}_t \|_{\mathbf{M}_{A_t, t}^{-1}} \\
    &\qquad + \vep_N\left(\sqrt{\tau_{A^*_t,t}} \|\bar{x}_t\|_{\mathbf{M}_{A^*_t,t}^{-1}} + \sqrt{\tau_{A_t,t}} \|\bar{x}_t\|_{\mathbf{M}_{A_t,t}^{-1}}\right)\\
    &\leq \langle \bar{x}_t, \hat{\beta}_{A_t, t}\rangle + \gamma \|\bar{x}_t\|_{\mathbf{M}_{A_t, t}^{-1}} - \langle  \bar{x}_t, \hat{\beta}_{A_t, t} \rangle + \gamma \|\bar{x}_t \|_{\mathbf{M}_{A_t, t}^{-1}} \\
    &\qquad + \vep_N\left(\sqrt{\tau_{A^*_t,t}} \|\bar{x}_t\|_{\mathbf{M}_{A^*_t,t}^{-1}} + \sqrt{\tau_{A_t,t}} \|\bar{x}_t\|_{\mathbf{M}_{A_t,t}^{-1}}\right)\\
    &\leq 2 \gamma \|\bar{x}_t\|_{\mathbf{M}_{A_t, t}^{-1}} + \vep_N\left(\sqrt{\tau_{A^*_t,t}} \|\bar{x}_t\|_{\mathbf{M}_{A^*_t,t}^{-1}} + \sqrt{\tau_{A_t,t}} \|\bar{x}_t\|_{\mathbf{M}_{A_t,t}^{-1}}\right),
\end{align*}
where the first inequality follows from Lemma \ref{prop:x_beta} and the second inequality follows from the action-selection rule in \texttt{DisSigUCB} (line \ref{line:ucb-choice} in Algorithm \ref{alg:dissigucb}). Summing the above bound over rounds $t=L, \dots, T$, we obtain
\begin{equation}
    R(T) \leq 2 \gamma W_\text{UCB} + \vep_N (W_{\text{sel}} + W_{\text{opt}}) + 2(T-L+1) \vep_N.
\end{equation}
\end{proof}

\subsection{Proof of Lemma~\ref{lemma:Wucb}} \label{ssec:proof_Wucb}
To prove Lemma \ref{lemma:Wucb}, which provides an upper bound for $W_\text{UCB}$ defined in Lemma~\ref{lemma:x_S0_S1}, we first introduce the following two lemmas (cf.\ Lemmas \ref{lemma:det_trace_ineq}  and \ref{lemma:sum_star_ineq}), where Lemma \ref{lemma:det_trace_ineq} establishes an upper bound on the determinant of the Gram matrices and Lemma \ref{lemma:sum_star_ineq} bounds the norm of contexts with the determinant of Gram matrices.

\begin{lemma}[Lemma 10 of \cite{APSLSB:11}] 
\label{lemma:det_trace_ineq}  Suppose $x_1, x_2, \dots, x_t \in \mathbb{R}^d$ and, for every $1 \leq s \leq t$, $|x_s| \leq B$. Fix $\lambda > 0$. Then, 
\(
\det\left( \lambda \mathbf{I} + \sum_{s=1}^t x_s x_s^\top \right) \leq (\lambda + tB^2/d)^d.
\)
\end{lemma}
\begin{lemma}[Lemma 11 of \cite{APSLSB:11}] \label{lemma:sum_star_ineq} Let $\{x_t\}_{t=1}^\infty$ be a sequence in $\mathbb{R}^d$, $V$ a positive definite matrix in $\sR^{d\times d},$ and define $\overline{V}_t = V + \sum_{s=1}^t x_s x_s^\top$. If $|x_t| \leq B$ for all $t \in \sN$ and $\lambda_{\min}(V) \geq \max(1, B^2)$, then for any $n\in \sN,$
\(
\sum_{t=1}^n \|x_t\|^2_{\overline{V}_{t-1}^{-1}} \leq 2 \log \left(\frac{\operatorname{det}(\overline{V}_n)}{\det(V)}\right).
\)
\end{lemma}
\begin{proof}[Proof of Lemma~\ref{lemma:Wucb}] 
For each action $a \in [K]$, we have
\begin{equation} \label{eq:sum_xbar_M}
    \sum_{s \in \cT_{a, T+1}} \|\bar{x}_s\|_{\mathbf{M}_{a,s}^{-1}}^2 \leq 2 \log\left(\frac{\det(\mathbf{M}_{a,T+1})}{\det(\lambda\mathbf{I})}\right) \leq 2(d+m)\log\left(1 + \frac{TB^2}{\lambda(d+m)}\right),
\end{equation}
where the first inequality follows from Lemma~\ref{lemma:sum_star_ineq} since $\lambda \geq B^2 \geq 1$ and the second inequality follows from Lemma~\ref{lemma:det_trace_ineq}. Then we re-index the sum $W_\text{UCB}$ from rounds $t=L,\dots,T$ to actions $a \in [K]$ and rounds $s \in \cT_{a,T+1}$ when action $a$ is chosen and get
\begin{align}
    W_\text{UCB} = \sum_{a=1}^K \sum_{s \in \cT_{a, T+1}} \|\bar{x}_s\|_{\mathbf{M}_{a,s}^{-1}}
    \leq \sum_{a=1}^K \sqrt{\tau_{a,T+1}} \sqrt{\sum_{s \in \cT_{a, T+1}} \|\bar{x}_s\|_{\mathbf{M}_{a,s}^{-1}}^2}
    \leq \sqrt{2KT(d+m)\log\left(1 + \frac{TB^2}{\lambda(d+m)}\right)},
\end{align}
where the first inequality follows from Cauchy-Schwarz over $s \in \cT_{a,T+1}$ and the second follows from \eqref{eq:sum_xbar_M} and Cauchy-Schwarz over $a \in [K]$ together with $\sum_{a=1}^K \tau_{a,T+1} = T-L+1 \leq T$. 
\end{proof}

\subsection{Proof of Lemma~\ref{lemma:Wsel_Wopt}} \label{ssec:proof_Wsel_Wopt}
We bound the two terms $W_\text{sel}$ and $W_\text{opt}$ associated with the approximation error $\vep_N$ using the minimum-eigenvalue guarantee from Lemma~\ref{lemma:min_eig_V}. 

\begin{proof}
Define the event where the minimum eigenvalue bound holds for all actions chosen at least $T_0$ times,
\[
E := \left\{\forall a \in [K], t \in [T] \text{ with } \tau_{a, t} \geq T_0: \lambda_{\min}(\mathbf{M}_{a,t}) \geq \lambda + \frac{\rho \tau_{a, t}}{2} \right\}.
\]
By Lemma~\ref{lemma:min_eig_V}, $\mathbb{P}(E) \geq 1 - \delta$. On the event $E$, for any action $a \in [K]$ and round $t$, we consider two cases. 

When $\tau_{a,t} < T_0$, 
\begin{equation}
    \sqrt{\tau_{a,t}} \|\bar{x}_t\|_{\mathbf{M}_{a,t}^{-1}} \leq \frac{\sqrt{\tau_{a,t}} |\bar{x}_t|}{\sqrt{\lambda_{\min}(\mathbf{M}_{a,t})}} \leq B \sqrt{\frac{T_0}{\lambda}},
\end{equation}
where the last inequality  follows from
$\mathbf{M}_{a,t} \succeq \lambda \mathbf{I}$.

When $\tau_{a,t} \geq T_0$,  by definition of $E$, we have
\begin{equation}
    \sqrt{\tau_{a,t}} \|\bar{x}_t\|_{\mathbf{M}_{a,t}^{-1}} \leq \frac{\sqrt{\tau_{a,t}} |\bar{x}_t|}{\sqrt{\lambda_{\min}(\mathbf{M}_{a,t})}} \leq \frac{\sqrt{\tau_{a,t}} B}{\sqrt{\rho\tau_{a,t}/2}} \leq B \sqrt{\frac{2}{\rho}}.  
\end{equation}
Hence, for every $a$ and $t$, 
\begin{equation} \label{eq:approx_term_bound}
     \sqrt{\tau_{a,t}} \|\bar{x}_t\|_{\mathbf{M}_{a,t}^{-1}} \leq B \max\left\{\sqrt{T_0 / \lambda}, \sqrt{2 / \rho}\right\}.
\end{equation}
Applying \eqref{eq:approx_term_bound} with $a = A_t$ and $a = A_t^*$, respectively, and summing over $t = L, \dots, T$ yields
\begin{equation*}
    W_\text{sel}, \, W_\text{opt} \leq (T-L+1) B \max\left\{\sqrt{T_0 / \lambda}, \sqrt{2 / \rho}\right\}.
\end{equation*}
\end{proof}

\subsection{Proof of Lemma \ref{lemma:min_eig_V}} \label{ssec:proof_lemma:min_eig_V}

We begin by recalling an adapted matrix Chernoff inequality, which provides high-probability lower-tail control of the minimum eigenvalue of a sum of adapted positive-semidefinite matrices and  serves as the main tool in the proof of Lemma \ref{lemma:min_eig_V}.

\begin{lemma}[Matrix Chernoff: Adapted Sequences, Theorem 3.1 of \cite{TMC:11}] \label{lemma:matrix_chernoff}
Consider a finite $\cG_s$-adapted sequence $\{\mathbf{X}_s\}$ of positive-semidefinite matrices with dimension $d+m$, and suppose that 
\(
\lambda_{\max}(\mathbf{X}_s) \leq B^2
\) almost surely.
Define the finite series
\(
\mathbf{Y}_t := \sum_{s=1}^t \mathbf{X}_s \; \text{ and } \; \mathbf{W}_t := \sum_{s=1}^t \mathbb{E}[\mathbf{X}_s | \cG_{s-1}].
\)
Then for all $\mu_t \geq 0$, 
\begin{align*}
    \mathbb{P}(\lambda_{\min}(\mathbf{Y}_t) \leq (1 - \delta) \mu_t \text{ and } \lambda_{\min}(\mathbf{W}_t) \geq \mu_t) &\leq (d + m) \cdot \left(\frac{e^{-\delta}}{(1 - \delta)^{1 - \delta}}\right)^{\mu_t/B^2} \; \text{ for } \delta \in [0,1).
\end{align*}
\end{lemma}


\begin{proof}[Proof of Lemma \ref{lemma:min_eig_V}]
Fix any action $a \in [K]$ and consider the action-specific sequence of positive-semidefinite matrices $\{\bar{x}_s \bar{x}_s^\top\}_{s \in \cT_{a,T+1}}$. 
Let $s_{a,j}$ and $\cG_{a,j}$ be defined in the same way as in Assumption \ref{assum:unified}(ii) and define $\mathbf{Z}_{a,j} := \bar{x}_{s_{a,j}} \bar{x}_{s_{a,j}}^\top$ for $j \in [\tau_{a,T+1}].$ 
Then, the sequence $\{\mathbf{Z}_{a,j}\}_{j=1}^{\tau_{a,{T+1}}}$ is adapted to $\{\cG_{a,j}\}_{j\geq 0}$. 
To apply Lemma~\ref{lemma:matrix_chernoff} to $\{\mathbf{Z}_{a,j}\}_{j \in [\tau_{a,{T+1}}]}$, we verify that its conditions hold. 
First, by construction,  $\mathbf{Z}_{a,j}$ is positive semidefinite. Second, by Lemma \ref{prop:sig_l2}, $\lambda_{\max}(\mathbf{Z}_{a,j}) = |\bar{x}_{s_{a,j}}|^2 \leq B^2$.

Next, to apply Lemma~\ref{lemma:matrix_chernoff}, for any $n \geq 1$ such that action $a$ has been selected at least $n$ times, define
  \(  W_{a,n} \coloneqq \sum_{j=1}^n \mathbb{E}[{
  {\bf Z}_{a,j}} | \cG_{a,j-1}]. \)
Then by Assumption~\ref{assum:unified}(ii), we have
\begin{equation}
    \frac{1}{n}W_{a,n} = \frac{1}{n} \sum_{j=1}^n \mathbb{E}[\mathbf{Z}_{a,j} | \cG_{a,j-1}] \succeq \rho \mathbf{I},
\end{equation}
which implies that $\lambda_{\min}(W_{a,n}) \geq \rho n$. Applying Lemma~\ref{lemma:matrix_chernoff} with $\delta = 1/2$ and $\mu_n = \rho n$, we have
\begin{equation} \label{eq:lm_min_M_failure}
    \mathbb{P}\left(\lambda_{\min}\left(\sum_{j=1}^n \mathbf{Z}_{a,j}\right) \leq \frac{\rho n}{2}\right) \leq (d+m)\left(\sqrt{\frac{2}{e}}\right)^\frac{\rho n}{B^2} \leq (d+m) \exp\left(-\frac{\rho n}{8B^2}\right),
\end{equation}
where the second inequality follows from $\sqrt{2/e} \leq e^{-1/8}$.

To obtain a bound that holds uniformly over all $K$ actions and selection rounds, we choose $T_0$ such that the failure probability in \eqref{eq:lm_min_M_failure} is at most $\delta/(KT)$ whenever $n \geq T_0$. 
That is, with $T_0$ as in \eqref{eq:T0}, for every $n \geq T_0$, we have
\begin{equation}
    (d+m) \exp\left(-\frac{\rho n}{8B^2}\right) \leq (d+m) \exp \left(-\frac{\rho T_0}{8B^2} \right) \leq \frac{\delta}{KT}.
\end{equation}
Therefore, by a union bound over $a \in [K]$ and $n \in \{T_0, \dots, T\}$ such that action $a$ has been selected at least $n$ times, we obtain that with probability at least $1 - \delta$, 
\begin{equation} \label{eq:lm_min_sum_Z}
    \lambda_{\min}\left(\sum_{j=1}^n \mathbf{Z}_{a,j}\right) \geq \frac{\rho n}{2},
\end{equation}
for every action and selection count $n \geq T_0$. 

Finally, on this event, for any $a \in [K]$ and $t \in [T]$ such that $\tau_{a,t} \geq T_0$, set $n = \tau_{a,t}$. Since $\sum_{j=1}^{\tau_{a,t}} \mathbf{Z}_{a,j} = \sum_{s \in \cT_{a,t}} \bar{x}_s \bar{x}_s^\top$, we have that by \eqref{eq:lm_min_sum_Z} and definition of $\mathbf{M}_{a,t}$,
\begin{equation}
    \lambda_{\min}(\mathbf{M}_{a,t}) \geq \lambda + \frac{\rho \tau_{a,t}}{2}.
\end{equation}
Since $a$ and $t$ were arbitrary, the bound holds simultaneously for all such actions and rounds. 
\end{proof}

\section{Detailed Proofs in Section \ref{sec:asym_norm}}
\label{sec:detailed_proofs_asym_norm}
\subsection{Proofs in Section \ref{ssec:consistency} } \label{ssec:proof_ridge_tail}
\begin{proof}[Proof of Proposition \ref{prop:ridge_tail_eps_dissigucb}]
Fix action $a \in [K]$. Under $\vep_N = 0$, \eqref{eq:beta_M} in the proof of Lemma~\ref{prop:x_beta} 
implies that for any $\delta_C \in (0,1)$, the following event occurs with probability at least $1 - \delta_C$:
\begin{equation} \label{eq:beta_error_bound}
    C_t = \left\{\|\hat{\beta}_{a,t} - \beta^*_a\|_{\mathbf{M}_{a,t}} \leq \gamma_t(\delta_{C})\right\}.
\end{equation}
Then we observe that the 1-norm distance between the ridge estimate and the true parameter, with probability at least $1 -\delta_C$, is upper bounded as follows:
\begin{equation} \label{eq:beta_1norm}
    \|\hat{\beta}_{a,t} - \beta^*_a\|_1 \leq \sqrt{d+m}  \;\|\hat{\beta}_{a,t} - \beta^*_a\|_2 \leq \frac{\sqrt{d+m} \; \|\hat{\beta}_{a,t} - \beta^*_a\|_{\mathbf{M}_{a,t}}}{\sqrt{\lambda_{\min}(\mathbf{M}_{a,t})}} 
    \leq \frac{\gamma_t(\delta_C) \sqrt{d+m}}{\sqrt{\lambda_{\min}(\mathbf{M}_{a,t})}},
\end{equation}
where the first inequality follows from Cauchy-Schwarz, the second inequality follows from the minimum-eigenvalue bound, and the last uses event $C_t$ in \eqref{eq:beta_error_bound}. It remains to find a lower bound on $\lambda_{\min}(\mathbf{M}_{a,t})$.

We first establish a high-probability lower bound on the number of times action $a$ has been selected under $\zeta$-\texttt{DisSigUCB}. 
Let $N_{a,t}^\text{exp}$ be the number of exploration-based selections of action $a$ before round $t$, which is the sum of $t-L$ independent Bernoulli random variables with success probabilities $\zeta_s/K, s=L, \dots, t-1$. Since exploration occurs with probability $\zeta_s$ and the explored action is selected uniformly from the $K$ actions, we have
\begin{equation}
    \mu_t := \mathbb{E}[N_{a,t}^\text{exp}] = \frac{1}{K} \sum_{s=L}^{t-1} \zeta_s \geq \frac{(t-L)\zeta_t}{K},
\end{equation}
where the last inequality uses that $\zeta_s$ is non-increasing. Now we apply the multiplicative Chernoff bound to $N_{a,t}^\text{exp}$ with parameter $\delta = 1/2$, which gives us
\begin{equation}
    \mathbb{P}\left(N_{a,t}^\text{exp} \leq \frac{\mu_t}{2}\right) \leq \left(\sqrt\frac{2}{e}\right)^{\mu_t} \leq e^{- \mu_t / 8}. 
\end{equation}
Let
\(
n_t:=\left\lceil \frac{(t-L)\zeta_t}{2K}\right\rceil.
\)
Since $N_{a,t}^{\exp}$ is integer-valued and
\(
\frac{(t-L)\zeta_t}{2K}\le \frac{\mu_t}{2},
\)
we have
\begin{equation}
\{N_{a,t}^{\exp}<n_t\}
=
\left\{N_{a,t}^{\exp}<\frac{(t-L)\zeta_t}{2K}\right\}
\subseteq
\left\{N_{a,t}^{\exp}\le \frac{\mu_t}{2}\right\}.
\end{equation}
Since $\tau_{a,t}\ge N_{a,t}^{\exp}$, it follows that
\begin{equation}
    \mathbb{P}\left(\tau_{a,t} < n_t\right) \leq \exp\left(-\frac{(t-L)\zeta_t}{8K}\right).
\end{equation}
So with high probability, action $a$ has been selected at least $n_t$ times before round $t$.

Next, to get a high-probability bound on the smallest eigenvalue of $\mathbf{M}_{a,t}$, we use the adapted matrix Chernoff inequality \cite{TMC:11} (Lemma~\ref{lemma:matrix_chernoff}). On the event $\{\tau_{a,t} \geq n_t\}$, consider the first $n_t$ selections of $a$. Let $s_{a,j}$ and $\cG_{a,j}$ be the selection times and filtrations defined in Assumption~\ref{assum:unified} and define 
\begin{equation}
    \mathbf{Z}_{a,j} := \bar{x}_{s_{a,j}} \bar{x}_{s_{a,j}}^\top, \quad W_{a, n_t} := \sum_{j=1}^{n_t} \mathbb{E}[\mathbf{Z}_{a,j} | \cG_{a,j-1}].
\end{equation}
By Assumption~\ref{assum:unified}(ii), we have
\begin{equation}
    W_{a,n_t} = \sum_{j=1}^{n_t} \mathbb{E}[\mathbf{Z}_{a,j} | \cG_{a,j-1}] \succeq \rho n_t \mathbf{I},
\end{equation}
which implies that $\lambda_{\min}(W_{a,n_t}) \geq \rho n_t$. Following the same application of the matrix Chernoff inequality (Lemma~\ref{lemma:matrix_chernoff}) as in the proof of Lemma~\ref{lemma:min_eig_V}, we get
\begin{equation}
    \mathbb{P}\left(\tau_{a,t} \geq n_t, \lambda_{\min}\left(\sum_{j=1}^{n_t} \mathbf{Z}_{a,j}\right) \leq \frac{\rho n_t}{2}\right) \leq (d+m) \exp\left(-\frac{\rho n_t}{8B^2}\right) \leq (d+m) \exp\left(-\frac{\rho (t-L)\zeta_t}{16KB^2}\right). 
\end{equation}
It thus follows that on $\{\tau_{a,t} \geq n_t\}$, we have 
\begin{equation}
    \mathbf{M}_{a,t} = \lambda \mathbf{I} + \sum_{j=1}^{\tau_{a,t}} \mathbf{Z}_{a,j} \succeq \lambda \mathbf{I} + \sum_{j=1}^{n_t} \mathbf{Z}_{a,j},
\end{equation}
since $\mathbf{Z}_{a,j} \succeq 0$.
It thus follows that, with high probability, the smallest eigenvalue of $\mathbf{M}_{a,t}$ can be lower bounded as follows:
\begin{equation} \label{eq:lambda_min_M_bound}
    \lambda_{\min}(\mathbf{M}_{a,t}) \geq \lambda + \frac{\rho n_t}{2} \geq \lambda + \frac{\rho (t-L)\zeta_t}{4K}.
\end{equation}

Combining the above result with the confidence event in \eqref{eq:beta_1norm} which holds with probability at least $1 - \delta_C$, and the selection-count event $\{\tau_{a,t} \geq n_t\}$ which holds with probability at least $1 - \exp(-\frac{(t-L)\zeta_t}{8K})$, we conclude that
\begin{equation}
    \mathbb{P}\left(\|\hat{\beta}_{a,t} - \beta^*_a\|_1 \leq \frac{\gamma_t \sqrt{d+m}}{\sqrt{\lambda + \rho(t-L)\zeta_t / (4K)}}\right) \geq 1 - \delta_C - \exp\left(-\frac{(t-L)\zeta_t}{8K}\right) - (d+m) \exp\left(-\frac{\rho(t-L) \zeta_t}{16KB^2}\right).
\end{equation}
This proves the desired tail bound for the ridge estimator under $\zeta$-\texttt{DisSigUCB}.

\end{proof}

\begin{proof}[Proof of Corollary~\ref{cor:eps_dissigucb_consistency}]
For each round $t \geq L+1$, let $\delta_{C,t} := \frac{1}{t}$. Since Proposition~\ref{prop:ridge_tail_eps_dissigucb} holds for any $\delta_{C,t} \in (0,1)$, for any fixed action $a$ at time $t$, we have
\begin{equation}
\mathbb{P} \left(\|\hat\beta_{a,t}-\beta_a^*\|_1 \leq h_t \right) \geq 1 - \delta_{C,t} - \exp\left(-\frac{(t-L)\zeta_t}{8K}\right) - (d+m)\exp\left(-\frac{\rho(t-L)\zeta_t}{16KB^2}
\right),
\end{equation}
where
$h_t = \frac{\gamma_t(\delta_{C,t})\sqrt{d+m}}
{\sqrt{\lambda+\rho(t-L)\zeta_t/(4K)}},$ and  $\gamma_t(\delta_{C,t}) := \sqrt{(d+m) \log \left(t(1 + \frac{tB^2}{\lambda(d+m)})\right)} + S\sqrt{\lambda} = \cO(\sqrt{\log t}).$

Fix $x>0$. We observe that
$
    h_t^2 = \cO\left(\frac{\log t}{t \zeta_t}\right).
$
Since $t\zeta_t / \log t \to \infty$ by assumption, there exists a time $T_x$ such that for all $t \geq T_x$, we have $h_t < x$. It follows then that for all such $t \geq T_x$,
\begin{equation}
    \mathbb{P}(\|\hat{\beta}_{a,t} - \beta^*_{a}\|_1 > x ) \leq \mathbb{P}(\|\hat{\beta}_{a,t} - \beta^*_{a}\|_1 > h_t) \leq \frac{1}{t} + \exp\left(-\frac{(t-L)\zeta_t}{8K}\right) + (d+m)\exp\left(-\frac{\rho(t-L)\zeta_t}{16KB^2}
\right).
\end{equation}
Since the right-hand side converges to 0, we have that
$
    \mathbb{P}(\|\hat{\beta}_{a,t} - \beta^*_{a}\|_1 > x ) \to 0,
$
which proves the claim. 
\end{proof}

\subsection{Proof of Lemma \ref{lemma:diff_decomp}}
\begin{proof}[\unskip\nopunct]
By \eqref{eq: M},
the ridge estimate for action $a$ before we make a decision in round $t$ is 
    \begin{equation} \label{eq:beta_hat_at}
        \hat{\beta}_{a,t} =\mathbf{M}_{a,t}^{-1} \mathbf{X}_{a,t}^\top \mathbf{r}_{a,t} = \mathbf{M}_{a,t}^{-1} \sum_{s=L}^{t-1} \mathbf{1}\{A_s = a\}  \bar{x}_s (\bar{x}_s^\top \beta_a^* + \eta_s).
\end{equation}
Substituting $\sum_{s=L}^{t-1} \mathbf{1}\{A_s = a\}  \bar{x}_s \bar{x}_s^\top = \mathbf{M}_{a,t} - \lambda \mathbf{I}$ into \eqref{eq:beta_hat_at}, we get that the distance between the estimate and the true parameter is
\begin{align}
    \hat{\beta}_{a,t} - \beta_a^* 
    &= \mathbf{M}_{a,t}^{-1} (\mathbf{M}_{a,t} - \lambda \mathbf{I}) \beta^*_a + \mathbf{M}_{a,t}^{-1} \sum_{s=L}^{t-1} \mathbf{1}\{A_s = a\} \bar{x}_s \eta_s - \beta^*_a \notag \\
    &= \mathbf{M}_{a,t}^{-1} \sum_{s=L}^{t-1} \mathbf{1}\{A_s = a\} \bar{x}_s \eta_s - \lambda \mathbf{M}_{a,t}^{-1}\beta_a^* \label{eq:beta_hat_beta_star}
\end{align}
Multiplying \eqref{eq:beta_hat_beta_star} by $\sqrt{t}$ gives the desired decomposition. 
\end{proof}

\subsection{Proof of Lemma \ref{lemma:action_conv}} \label{ssec:proof_action_conv}

\begin{proof}[\unskip\nopunct]
We first show that the UCB bonus vanishes as the rounds progress.

Note that, for any action $a$ and round $t$, the UCB bonus is bounded above by
\begin{equation} \label{eq:gamma_xbar_M}
\gamma \|\bar{x}_t\|_{\mathbf{M}_{a,t}^{-1}} \leq \frac{B\gamma}{\sqrt{\lambda_{\min}(\mathbf{M}_{a,t})}},
\end{equation}
where we use boundedness of $\bar{x}_t$.

Furthermore, as in the proof of Proposition~\ref{prop:ridge_tail_eps_dissigucb}, \eqref{eq:lambda_min_M_bound} holds for each $a \in [K]$,
\begin{equation}
    \mathbb{P}\left(\lambda_{\min}(\mathbf{M}_{a,t}) \geq \lambda + \frac{\rho(t-L) \zeta_t}{4K}\right) \geq 1 - \exp\left(-\frac{(t-L)\zeta_t}{8K} \right) - (d+m) \exp\left(-\frac{\rho (t-L) \zeta_t}{16KB^2}\right).
\end{equation}
We then apply this high-probability bound on $\lambda_{\min}(\mathbf{M}_{a,t})$ to the upper bound in \eqref{eq:gamma_xbar_M} and get
\begin{equation} \label{eq:max_gamma_xbar_M}
    \mathbb{P}\left(\gamma \|\bar{x}_t\|_{\mathbf{M}_{a,t}^{-1}} \leq \frac{B\gamma}{\sqrt{\lambda + \rho(t-L) \zeta_t/4K}}\right) \geq 1 - \exp\left(-\frac{(t-L)\zeta_t}{8K} \right) - (d+m) \exp\left(-\frac{\rho(t-L)\zeta_t}{16KB^2}\right).
\end{equation}
Since $\zeta_t \to \zeta_\infty \in (0,1)$, the right-hand side of \eqref{eq:max_gamma_xbar_M} converges to one and the upper bound on the UCB bonus 
$\gamma \|\bar{x}_t\|_{\mathbf{M}_{a,t}^{-1}}$ in \eqref{eq:max_gamma_xbar_M} converges to 0. 
Therefore, by a union bound over \(a\in[K]\), the UCB bonus for all actions converges to 0 in probability,    
$
U_t := \max_{a \in [K]} \gamma \|\bar{x}_t\|_{\mathbf{M}_{a,t}^{-1}} \stackrel{\sP}{\rightarrow} 0.
$

Next, we show that $\frac{1}{t} \sum_{s=L}^t \mathbf{1}\{a_s^{\texttt{DisSigUCB}} \neq a_s^*\} \stackrel{\sP}{\to} 0$.
We let $E_s$ denote the maximum prediction error in round $s$,
  $  E_s := \max_{b \in [K]} \left|\bar{x}_s^\top (\hat{\beta}_{b,s} - \beta^*_b) \right|.$
By boundedness of $\bar{x}_s$ and consistency of the estimators (Corollary \ref{cor:eps_dissigucb_consistency}), we have 
   $ E_s \leq B \max_{b \in [K]} \|\hat{\beta}_{b,s} - \beta^*_b\|_2 \stackrel{\sP}{\to} 0.$

Now suppose \texttt{DisSigUCB} selects a suboptimal action $j := a_s^{\texttt{DisSigUCB}} \neq a_s^*$. Because action $j$ maximizes the UCB, we have
$
\bar{x}_s^\top \hat\beta_{j,s} + \gamma \|\bar{x}_s\|_{\mathbf{M}_{j,s}^{-1}} \geq \bar{x}_s^\top \hat\beta_{a_s^*,s} + \gamma \|\bar{x}_s\|_{\mathbf{M}_{a_s^*,s}^{-1}},
$
which implies that
\begin{equation} \label{eq:reward_estimate_diff}
    \bar{x}_s^\top (\hat\beta_{a_s^*,s} - \hat\beta_{j,s}) \leq \gamma \|\bar{x}_s \|_{\mathbf{M}_{j,s}^{-1}} \leq U_s.
\end{equation}
Then we have that the gap between the reward of the optimal action and the \texttt{DisSigUCB} action can be bounded above as follows,
\begin{align*}
\bar{x}_s^\top \beta^*_{a_s^*} - \bar{x}_s^\top \beta^*_j &= \bar{x}_s^\top (\beta^*_{a_s^*} - \hat\beta_{a_s^*, s}) + \bar{x}_s^\top(\hat\beta_{a_s^*, s} - \hat\beta_{j, s}) + \bar{x}_s^\top (\hat\beta_{j, s} - \beta^*_j) \\
&\leq |\bar{x}_s^\top (\beta^*_{a_s^*} - \hat\beta_{a_s^*, s})| + U_s + |\bar{x}_s^\top (\hat\beta_{j, s} - \beta^*_j)| \leq 2E_s + U_s,
\end{align*}
where we use \eqref{eq:reward_estimate_diff} in the first inequality and the definition of $E_s$  in the second inequality. 

Since $j \neq a_s^*$ is suboptimal, by definition of $\Delta_s$ in Assumption~\ref{assumption:ties}, $\Delta_s \leq \bar{x}_s^\top \beta^*_{a_s^*} - \bar{x}_s^\top \beta^*_j \leq 2E_s + U_s.$
Therefore, for any $\kappa>0$, 
$
    \{a_s^\texttt{DisSigUCB} \neq a_s^*\} \subseteq \{2E_s + U_s > \kappa\} \cup \{\Delta_s \leq \kappa\},
$
implying that
\begin{equation}
    \frac{1}{t} \sum_{s=L}^t \mathbb{P}(a_s^\texttt{DisSigUCB} \neq a_s^*) \leq \frac{1}{t} \sum_{s=L}^t \mathbb{P}(2E_s + U_s > \kappa) + \frac{1}{t} \sum_{s=L}^t \mathbb{P}(\Delta_s \leq \kappa).
\end{equation}
However, since $2E_s + U_s\stackrel{\sP}{\to}0$, by Cesàro's lemma, $\frac{1}{t} \sum_{s=L}^t \mathbb{P}(2E_s + U_s > \kappa) \to 0$. Then,
\begin{equation} \label{eq:action_diff_conv}
    \limsup_{t \to \infty} \frac{1}{t} \sum_{s=L}^t \mathbb{P}(a_s^\texttt{DisSigUCB} \neq a_s^*) \leq \lim_{\kappa \downarrow 0} \limsup_{t \to \infty} \frac{1}{t} \sum_{s=L}^t \mathbb{P}(\Delta_s \leq \kappa) \to 0,
\end{equation}
where convergence follows from Assumption~\ref{assumption:ties}.

Therefore, for any $\kappa > 0$, 
\begin{equation}
    \mathbb{P}\left(\frac{1}{t} \sum_{s=L}^t \mathbf{1}\{a_s^\texttt{DisSigUCB} \neq a_s^*\} > \kappa \right) \leq \frac{\frac{1}{t} \sum_{s=L}^t \mathbb{P}(a_s^\texttt{DisSigUCB} \neq a_s^*)}{\kappa} \to 0,
\end{equation}
where the inequality follows from Markov's inequality and convergence follows from \eqref{eq:action_diff_conv}. Thus, by the definition of convergence in probability,
$
\frac{1}{t} \sum_{s=L}^t \mathbf{1}\{a_s^\texttt{DisSigUCB} \neq a_s^*\} \stackrel{\sP}{\to} 0.
$
Finally, 
we observe that
\begin{align*}
     \frac{1}{t} \sum_{s=L}^t |\pi_{a,s} - \pi^*_{a,s}| &\leq \frac{1}{t} \sum_{s=L}^t \left|(1 - \zeta_s) \left(\mathbf{1}\{a_s^\texttt{DisSigUCB} = a\} - \mathbf{1}\{a_s^* = a\}\right)\right| \leq \frac{1}{t} \sum_{s=L}^t \mathbf{1}\{a_s^\texttt{DisSigUCB} \neq a_s^*\} \stackrel{\sP}{\to} 0.
\end{align*}
\end{proof}

\subsection{Proof of Lemma \ref{lemma:M_at/t_conv}} \label{ssec:proof_M_at/t_conv}

\begin{proof}[\unskip\nopunct]
We decompose $\mathbf{M}_{a,t}/t$ as follows
\begin{equation}
    \frac{\mathbf{M}_{a,t}}{t} = \frac{\lambda}{t} \mathbf{I} + \frac{1}{t} \sum_{s=L}^{t-1} \mathbf{1}\{A_s = a\} \bar{x}_s \bar{x}_s^\top = \frac{\lambda}{t} \mathbf{I} + R_{1,t} + R_{2,t} + R_{3,t} + R_{4,t} + Q_a,
\end{equation}
where
\begin{align}
    R_{1,t} &= \frac{1}{t} \sum_{s=L}^{t-1} (\mathbf{1}\{A_s = a\} - \pi_{a,s}) \bar{x}_s \bar{x}_s^\top, \,
    R_{2,t} = \frac{1}{t}\sum_{s=L}^{t-1}(\pi_{a,s} - \pi^*_{a,s}) \bar{x}_s \bar{x}_s^\top, \\
    R_{3,t} &= \frac{1}{t} \sum_{s=L}^{t-1} (\pi^*_{a,s} \bar{x}_s \bar{x}_s^\top - \mathbb{E}[\pi^*_{a,s} \bar{x}_s \bar{x}_s^\top | \tilde \cF_{s-1}]), \, 
    R_{4,t} = \frac{1}{t} \sum_{s=L}^{t-1} \mathbb{E}[\pi^*_{a,s} \bar{x}_s \bar{x}_s^\top | \tilde \cF_{s-1}] - Q_a.
\end{align}

For the first term $R_{1,t}$, we observe that by the definition of $\pi_{a,s}$, the summands of $R_{1,t}$ form a martingale-difference sequence with respect to $\tilde\cF_{s}$, which implies that its increments are orthogonal in expectation. Therefore, $R_{1,t}$ can be bounded as follows
\begin{equation}
    \mathbb{E}\|R_{1,t}\|_F^2 = \frac{1}{t^2} \sum_{s=L}^{t-1} \mathbb{E}[\|(\mathbf{1}\{A_s = a\} - \pi_{a,s}) \bar{x}_s \bar{x}_s^\top\|_F^2] \leq \frac{B^4}{t}.
\end{equation}
Then by Markov's inequality, for any $\kappa > 0,$
\begin{equation}
    \mathbb{P}(\|R_{1,t}\|_\text{op} > \kappa) \leq \mathbb{P}(\|R_{1,t}\|_F > \kappa) = \mathbb{P}(\|R_{1,t}\|_F^2 > \kappa^2) \leq \frac{\mathbb{E}[\|R_{1,t}\|_F^2]}{\kappa^2} \leq \frac{B^4}{t\kappa^2} \to 0,
\end{equation}
which implies that $R_{1,t}\stackrel{\sP}{\to} 0.$ For the second term $R_{2,t}$, by Lemma \ref{lemma:action_conv},
\begin{equation}
    \|R_{2,t}\|_\text{op} \leq \frac{1}{t} \sum_{s=L}^{t-1} |\pi_{a,s} - \pi^*_{a,s}| \|\bar{x}_s \bar{x}_s^\top\|_\text{op} \leq \frac{B^2}{t} \sum_{s=L}^{t-1} |\pi_{a,s} - \pi^*_{a,s}| \stackrel{\sP}{\to} 0.
\end{equation}
We observe that the third term $R_{3,t}$ is also a sum of martingale differences with respect to $(\tilde\cF_{s})$, which also implies that its increments are orthogonal in expectation. Thus, $R_{3,t}$ can be bounded as follows
\begin{equation}
    \mathbb{E}\|R_{3,t}\|^2_F = \frac{1}{t^2} \sum_{s=L}^{t-1} \mathbb{E}[\|\pi^*_{a,s} \bar{x}_s \bar{x}_s^\top - \mathbb{E}[\pi^*_{a,s} \bar{x}_s \bar{x}_s^\top | \tilde \cF_{s-1}]\|_F^2] \leq \frac{4B^4}{t}.
\end{equation}
Then, again by Markov's inequality, we have for any $\kappa > 0$,
\begin{equation}
    \mathbb{P}(\|R_{3,t}\|_\text{op} > \kappa) \leq \mathbb{P}(\|R_{3,t}\|_F > \kappa) = \mathbb{P}(\|R_{3,t}\|_F^2 > \kappa^2) \leq \frac{\mathbb{E}[\|R_{3,t}\|_F^2]}{\kappa^2} \leq \frac{4B^4}{t\kappa^2} \to 0,
\end{equation}
which implies that $R_{3,t} \stackrel{\sP}{\to} 0.$ Lastly, by Assumption~\ref{assumption:oracle_conv}, we have $R_{4,t} \stackrel{\sP}{\to} 0.$ Since all four terms converge to 0 in probability and $\frac{\lambda}{t} \to 0$, we have that
    $\frac{\mathbf{M}_{a,t}}{t} = \frac{\lambda}{t} \mathbf{I} + R_{1,t} + R_{2,t} + R_{3,t} + R_{4,t} + Q_a \stackrel{\sP}{\to} Q_a.$
\end{proof}

\subsection{Proof of Lemma \ref{lemma:score_conv}}
\label{ssec:proof_score_conv}
The proof of Lemma \ref{lemma:score_conv} relies on a martingale central limit theorem, which characterizes the limiting distribution of a martingale difference sequence:
\begin{lemma}[Martingale Central Limit Theorem, Corollary 3.1 of \cite{HH:1980}] \label{lemma:martingale_clt}
Let $\{\xi_s\}_{s\geq1}$ be an $\mathbb R^{d+m}$-valued,
square-integrable martingale difference sequence adapted to $\{\tilde\cF_s\}_{s\geq0}$, so that $\mathbb E[\xi_s | \tilde\cF_{s-1}]=0.$
Suppose that, for some deterministic positive-semidefinite matrix
$\Sigma\in\mathbb R^{(d+m)\times (d+m)}$,
\(
    \frac{1}{t}\sum_{s=1}^{t-1}
    \mathbb E\!\left[
        \xi_s\xi_s^\top
        \mid \tilde\cF_{s-1}
    \right]
    \xrightarrow{\sP}\Sigma,
\)
and that, for every $c>0$,
\(
    \frac{1}{t}\sum_{s=1}^{t-1}
    \mathbb E\!\left[
        \|\xi_s\|^2
        \mathbf 1\{\|\xi_s\|>c\sqrt t\}
        \mid\tilde\cF_{s-1}
    \right]
    \xrightarrow{\sP}0.
\)
Then,
\(
    \frac{1}{\sqrt t}\sum_{s=1}^{t-1}\xi_s
    \xrightarrow{d}\mathcal N(0,\Sigma).
\)
\end{lemma}

\begin{proof}[Proof of Lemma \ref{lemma:score_conv}] 
For any action $a$ and round $s$, let $\xi_{a,s} := \mathbf{1}\{A_s = a\} \bar{x}_s \eta_s$. We observe that by the conditionally mean-zero noise assumption,
\begin{equation}
    \mathbb{E}[\xi_{a,s} | \tilde\cF_{s-1}] = \mathbb{E}[\mathbf{1}\{A_s = a\}\bar{x}_s \eta_s | \tilde\cF_{s-1}] = \mathbb{E}[\mathbf{1}\{A_s = a\}\bar{x}_s \mathbb{E}[\eta_s | \cF_s, A_s] | \tilde\cF_{s-1}] = 0.
\end{equation}
Fix $u \in \mathbb{R}^{d+m}$. Then $(u^\top\xi_{a,s})_{s \geq 1}$ is a square-integrable martingale difference sequence adapted to $\{\tilde\cF_s\}_{s \geq 0}$.
To establish the desired multivariate convergence, we next apply the martingale central limit theorem \cite{HH:1980} in Lemma~\ref{lemma:martingale_clt}  to $u^\top \xi_{a,s}$ and verify its two conditions below.

First, we need to show that $\frac{1}{t}\sum_{s=L}^{t-1} \mathbb E[(u^\top \xi_{a,s})^2 \mid \tilde\cF_{s-1}] \xrightarrow{\sP} \sigma_a^2 u^\top Q_a u.$ We observe that the per-round conditional covariance satisfies
\begin{align*}
    \mathbb{E}[(u^\top\xi_{a,s})^2 | \tilde\cF_{s-1}] &= u^\top \mathbb{E}[\mathbf{1}\{A_s =a\} \bar{x}_s \bar{x}_s^\top \eta_s^2 | \tilde\cF_{s-1}] u \\
    &= u^\top \mathbb{E}[\mathbf{1}\{A_s =a\} \bar{x}_s \bar{x}_s^\top \mathbb{E}[\eta_s^2 | \cF_s, A_s] | \tilde\cF_{s-1}] u = \sigma_a^2 u^\top \mathbb{E}[\mathbf{1}\{A_s =a\} \bar{x}_s \bar{x}_s^\top | \tilde\cF_{s-1}] u. \label{eq:xi_xi},
\end{align*}
where we use $\mathbb{E}[\eta_s^2 | \cF_s, A_s=a] = \sigma_a^2$ in the third equality. Therefore, we now want to show that $\frac{1}{t} \sum_{s=L}^{t-1} \mathbb{E}[\mathbf{1}\{A_s =a\} \bar{x}_s \bar{x}_s^\top | \tilde\cF_{s-1}]$ converges to $Q_a$. We note that since $\sum_{s=L}^{t-1} \mathbf{1}\{A_s = a\} \bar{x}_s \bar{x}_s^\top = \mathbf{M}_{a,t} - \lambda\mathbf{I}$ and $\mathbf{M}_{a,t}/t \stackrel{\sP}{\to} Q_a$ by Lemma \ref{lemma:M_at/t_conv}, we have
\begin{equation} \label{eq:X_conv}
    \frac{1}{t} \sum_{s=L}^{t-1} \mathbf{1}\{A_s = a\} \bar{x}_s \bar{x}_s^\top \stackrel{\sP}{\to} Q_a.
\end{equation}
Now define 
$    H_s := \mathbf{1}\{A_s=a\} \bar{x}_s \bar{x}_s^\top - \mathbb{E}[\mathbf{1}\{A_s=a\} \bar{x}_s \bar{x}_s^\top | \tilde\cF_{s-1}],$ which, by construction, forms matrix martingale differences and thus implies orthogonal increments. This implies by boundedness of $\bar{x}_t$ that
\begin{equation}
    \mathbb{E}\left[\left\|\frac{1}{t}\sum_{s=L}^{t-1} H_s\right\|_F^2\right] = \frac{1}{t^2} \sum_{s=L}^{t-1} \mathbb{E}[\|H_s\|_F^2] \leq \frac{4B^4}{t} \to 0.
\end{equation}
Hence, $\frac{1}{t} \sum_{s=L}^{t-1} H_s \stackrel{\sP}{\to} 0$ and therefore by \eqref{eq:X_conv},
\begin{equation}
    \frac{1}{t} \sum_{s=L}^{t-1} \mathbb{E}[\mathbf{1}\{A_s = a\}\bar{x}_s \bar{x}_s^\top | \tilde\cF_{s-1}] = \frac{1}{t} \sum_{s=L}^{t-1} \mathbf{1}\{A_s = a\} \bar{x}_s \bar{x}_s^\top - \frac{1}{t} \sum_{s=L}^{t-1} H_s \stackrel{\sP}{\to} Q_a.
\end{equation}
Therefore,
\begin{equation}
    \frac{1}{t} \sum_{s=L}^{t-1} \mathbb{E}[u^\top \xi_{a,s} \xi_{a,s}^\top u | \tilde\cF_{s-1}] = \frac{\sigma_a^2 u^\top}{t} \sum_{s=L}^{t-1} \mathbb{E}[\mathbf{1}\{A_s =a\} \bar{x}_s \bar{x}_s^\top | \tilde\cF_{s-1}]  u\stackrel{\sP}{\to} \sigma_a^2 u^\top Q_a u.
\end{equation}

Second, we need to verify that for every $c>0$,
$\frac{1}{t}\sum_{s=L}^{t-1} \mathbb E [|u^\top\xi_{a,s}|^2 \mathbf 1\{|u^\top\xi_{a,s}|>c\sqrt t\} \mid\tilde\cF_{s-1}] \xrightarrow{\sP}0.$ We observe that since $|\bar{x}_s| \leq B$, we have $|u^\top \xi_{a,s}| \leq |u|B|\eta_s|$.
The conditional sub-Gaussianity of $\eta_s$ therefore implies that for some finite constant $C_u$,
$
\mathbb{E}[|u^\top\xi_{a,s}|^4 | \tilde{\cF}_{s-1}] \leq C_u, \, \forall s \in \sN.
$
Hence, for every $c>0$, 
\begin{equation}
    \frac{1}{t}\sum_{s=L}^{t-1} \mathbb{E}[|u^\top\xi_{a,s}|^2 \mathbf{1}\{|u^\top\xi_{a,s}| > c\sqrt t\} | \tilde{\cF}_{s-1}] 
    \leq \frac{1}{t}\sum_{s=L}^{t-1} \frac{\mathbb{E}[|u^\top\xi_{a,s}|^4 |\tilde{\cF}_{s-1}]}{c^2t}
    \leq \frac{C_u}{c^2t} \longrightarrow 0 \text{ a.s.},
\end{equation}
where we use $|u^\top\xi_{a,s}|^2/(c^2t) > 1$ on the event $\{|u^\top\xi_{a,s}| > c\sqrt t\}$ in the inequality.

Now that we have verified the two conditions, by the martingale central limit theorem \cite{HH:1980} (cf. Lemma~\ref{lemma:martingale_clt}), for every fixed $u\in\mathbb{R}^{d+m}$,
   $ \frac{1}{\sqrt t} {u^\top \sum_{s=L}^{t-1} \xi_{a,s}}\stackrel{d}{\to} \cN(0, \sigma_a^2 u^\top Q_a u).$
The proof is finished by the Cramér--Wold theorem.
\end{proof}

\subsection{Proof of Corollary~\ref{cor:cov_consistency}} \label{ssec:proof_conf_int}
The proof of Corollary~\ref{cor:cov_consistency} relies on the martingale strong law (see Lemma \ref{lemma:mtg_stronglaw} below), which establishes almost sure convergence of the sample average of martingale differences to zero under suitable moment conditions.

\begin{lemma}[Martingale Strong Law, Theorem 2.18 of \cite{HH:1980}] \label{lemma:mtg_stronglaw}
Let $\{D_s\}_{s \geq L}$ be a martingale-difference sequence adapted to $\{\tilde\cF_s\}_{s \geq L-1}$, \textit{i.e.,} $\mathbb{E}[D_s | \tilde\cF_{s-1}] = 0$. If $\sum_{s=L}^\infty \mathbb{E}[D_s^2 | \tilde\cF_{s-1}]/s^2 < \infty $ \textit{a.s.}, then $\frac{1}{t} \sum_{s=L}^{t-1} D_s \to 0$ \textit{a.s.}
\end{lemma}

\begin{proof}[Proof of Corollary~\ref{cor:cov_consistency}] 
Fix \(a\in[K]\) and let $\delta_{a,t}:=\hat\beta_{a,t}-\beta_a^*$. 
Since \(\varepsilon_N=0\), whenever $A_s = a$,
\(
r_s-\bar x_s^\top\hat\beta_{a,t} = \eta_s-\bar x_s^\top\delta_{a,t}.
\)
Thus, we can decompose $\hat\sigma_{a,t}^2-\sigma_a^2$ as follows,
\begin{align*}
    \hat\sigma_{a,t}^2-\sigma_a^2 &= 
    \frac1{\tau_{a,t}} \sum_{s=L}^{t-1} \mathbf1\{A_s=a\}(\eta_s^2-\sigma_a^2) 
    - 2\delta_{a,t}^\top \frac1{\tau_{a,t}} \sum_{s=L}^{t-1} \mathbf1\{A_s=a\}\bar x_s\eta_s \\
    & \quad \quad + \delta_{a,t}^\top \left(\frac1{\tau_{a,t}} \sum_{s=L}^{t-1} \mathbf1\{A_s=a\}\bar x_s\bar x_s^\top \right) \delta_{a,t}.
\end{align*}

For the first term, define
\(
D_{a,s}:=\mathbf1\{A_s=a\}(\eta_s^2-\sigma_a^2).
\)
{Since the noise has conditional variance $\mathbb{E}[\eta_s^2\mid\mathcal{F}_s,A_s=a]=\sigma_a^2$}, by the tower property, we have
\begin{equation}
  \mathbb E[D_{a,s}|\tilde{\mathcal F}_{s-1}] = \mathbb E\left[\mathbf 1\{A_s=a\} \mathbb E[\eta_s^2-\sigma_a^2 | \mathcal F_s,A_s] | \tilde{\mathcal F}_{s-1} \right]=0.  
\end{equation}
Thus, \(\{D_{a,s}\}\) is a martingale-difference sequence with respect to \(\{\tilde{\mathcal F}_s\}\). 
Moreover, since $\sigma_a^2 \leq 1$ and conditional sub-Gaussianity implies uniformly bounded conditional fourth moments of \(\eta_s\), it follows that
\(
\sum_{s=L}^{\infty} \mathbb E[D_{a,s}^2 | \tilde{\mathcal F}_{s-1}] / s^2 < \infty \textit{ a.s.}
\)
Then, by the martingale strong law (cf. Lemma~\ref{lemma:mtg_stronglaw}), we have
\begin{equation} \label{eq:D_conv}
    \frac1t \sum_{s=L}^{t-1}D_{a,s} = \frac{1}{t} \sum_{s=L}^{t-1} \mathbf{1}\{A_s = a\}(\eta_s^2 - \sigma_a^2) \to 0  \textit{ a.s.}
\end{equation}
To replace the normalization $t$ by $\tau_{a,t}$, recall from the proof of Proposition~\ref{prop:ridge_tail_eps_dissigucb} that
\begin{equation}
    \mathbb{P}\left(\tau_{a,t} < \frac{(t-L) \zeta_t}{2K}\right) \leq \exp\left(-\frac{(t-L) \zeta_t}{8K} \right).
\end{equation}
Since $\zeta_t \downarrow \zeta_\infty > 0$, there exists $c > 0$ such that $\mathbb{P}(\tau_{a,t} \geq ct) \to 1$.
Consequently, this together with \eqref{eq:D_conv} implies
\begin{equation}
\frac1{\tau_{a,t}} \sum_{s=L}^{t-1} \mathbf1\{A_s=a\}(\eta_s^2-\sigma_a^2) \xrightarrow{\mathbb{P}} 0.
\end{equation}

For the second term, Lemma~\ref{lemma:score_conv} gives
\(
\frac1{\sqrt t} \sum_{s=L}^{t-1} \mathbf1\{A_s=a\}\bar x_s\eta_s \xrightarrow{d}
\cN(0,\sigma_a^2 Q_a).
\)
Since $1/\sqrt{t} \to 0$, Slutsky's theorem implies
\(
\frac1t \sum_{s=L}^{t-1} \mathbf1\{A_s=a\}\bar x_s\eta_s \xrightarrow{\mathbb{P}}0.
\)
Using \(\delta_{a,t}\xrightarrow{\mathbb{P}}0\) from Corollary~\ref{cor:eps_dissigucb_consistency} and \(\mathbb P(\tau_{a,t}\ge ct)\to1\) again, the second term converges to 0 in probability:
\begin{equation}
2 \delta_{a,t}^\top \frac1{\tau_{a,t}} \sum_{s=L}^{t-1} \mathbf1\{A_s=a\}\bar x_s\eta_s \xrightarrow{\mathbb{P}}0.
\end{equation}

For the third term, boundedness of the features gives
\(
\left\| \frac1{\tau_{a,t}} \sum_{s=L}^{t-1} \mathbf1\{A_s=a\}\bar x_s\bar x_s^\top \right\|_{\mathrm{op}} \le B^2.
\)
Thus, by Corollary~\ref{cor:eps_dissigucb_consistency},
\begin{equation}
\left| \delta_{a,t}^\top \left(\frac1{\tau_{a,t}} \sum_{s=L}^{t-1} \mathbf1\{A_s=a\}\bar x_s\bar x_s^\top \right) \delta_{a,t} \right|
\le B^2\|\delta_{a,t}\|^2 \xrightarrow{\mathbb{P}}0.
\end{equation}

Since all three terms converge to 0 in probability, we have
\(
\hat\sigma_{a,t}^2\xrightarrow{\mathbb{P}}\sigma_a^2.
\)
Furthermore, by Lemma~\ref{lemma:M_at/t_conv},
\(
\frac{\mathbf{M}_{a,t}}{t}\xrightarrow{\mathbb{P}}Q_a,
\)
where \(Q_a\) is positive definite. Therefore, by continuity of matrix inversion,
\begin{equation}
\hat\Sigma_{a,t} = \hat\sigma_{a,t}^2 \left(\frac{\mathbf{M}_{a,t}}{t}\right)^{-1} \xrightarrow{\mathbb{P}} \sigma_a^2 Q_a^{-1}.
\end{equation}
Combining this with Theorem~\ref{thm:asymptotic_normality_eps_dissigucb} and applying Slutsky's theorem gives, for every coordinate \(j\),
\begin{equation}
   \frac{\hat\beta_{a,t,j}-\beta_{a,j}^*}{\sqrt{[\hat\Sigma_{a,t}]_{j,j}/t}} \xrightarrow{d} \cN(0,1), 
\end{equation}
which yields the stated asymptotically valid confidence interval.
\end{proof}



\end{document}